\documentclass[twoside]{article}

\usepackage[accepted]{aistats2025}
\usepackage[round]{natbib}
\renewcommand{\bibname}{References}
\renewcommand{\bibsection}{\subsubsection*{\bibname}}

\usepackage{amsmath,amsfonts,bm}

\def\eqref#1{equation~\ref{#1}}

\def\1{\bm{1}}

\DeclareMathAlphabet{\mathsfit}{\encodingdefault}{\sfdefault}{m}{sl}
\SetMathAlphabet{\mathsfit}{bold}{\encodingdefault}{\sfdefault}{bx}{n}

\usepackage[table]{xcolor}  % Enables \rowcolor
\usepackage{colortbl}  
\usepackage{fancyhdr}
\usepackage{mathtools} % amsmath with fixes and additions
\usepackage{amsfonts}
\usepackage{amssymb}
\usepackage{amsmath}
\usepackage{amsthm}
\usepackage{mathrsfs}
\usepackage{amssymb}
\usepackage{amsmath} % for \boldsymbol macro
\usepackage{subfigure}
\usepackage{graphicx}
\usepackage{tcolorbox}
\usepackage{booktabs}
\newtheorem{theorem}{Theorem}
\newtheorem{definition}{Definition}
\newtheorem{lemma}{Lemma}
\newtheorem{assumption}{Assumption}
\newtheorem{proposition}{Proposition}

\newtheorem{theoremalt}{Theorem}[theorem]
\newenvironment{theoremref}[1]{
	\renewcommand\thetheoremalt{#1}
	\theoremalt
}{\endtheoremalt}

\usepackage{hyperref}
\usepackage{url}
\usepackage{graphicx} 
\usepackage{amsmath, bm} 
\usepackage{amsfonts}
\usepackage{multirow}
\usepackage{rotating}
\usepackage{booktabs} % for professional tables

\runningauthor{Yu, Li, Saha, Lou, Chen}
\begin{document}
\twocolumn[

\aistatstitle{Evidential Uncertainty Probes for Graph Neural Networks}

\aistatsauthor{ Linlin Yu \\ linlin.yu@utdallas.edu \And Kangshuo Li \\ kangshuo.li@utdallas.edu \And  Pritom Kumar Saha \\ pritom.saha@utdallas.edu }

\aistatsaddress{The University of Texas at Dallas \And  The University of Texas at Dallas \And The University of Texas at Dallas } 

\aistatsauthor{Yifei Lou \\ yflou@unc.edu \And Feng Chen \\ feng.chen@utdallas.edu}
\aistatsaddress{The University of North Carolina at Chapel Hill \And The University of Texas at Dallas}]
\newcommand{\fix}{\marginpar{FIX}}
\newcommand{\new}{\marginpar{NEW}}

\begin{abstract}
Accurate quantification of both aleatoric and epistemic uncertainties is essential when deploying Graph Neural Networks (GNNs) in high-stakes applications such as drug discovery and financial fraud detection, where reliable predictions are critical. Although Evidential Deep Learning (EDL) efficiently quantifies uncertainty using a Dirichlet distribution over predictive probabilities, existing EDL-based GNN (EGNN) models require modifications to the network architecture and retraining, failing to take advantage of pre-trained models.
We propose a plug-and-play framework for uncertainty quantification in GNNs that works with pre-trained models without the need for retraining. 
Our Evidential Probing Network (EPN) uses a lightweight Multi-Layer-Perceptron (MLP) head to extract evidence from learned representations, allowing efficient integration with various GNN architectures.  We further introduce evidence-based regularization techniques, referred to as EPN-reg, to enhance the estimation of epistemic uncertainty with theoretical justifications. 
Extensive experiments demonstrate that the proposed EPN-reg achieves state-of-the-art performance in accurate and efficient uncertainty quantification, making it suitable for real-world deployment.

\end{abstract}

\section{Introduction}
Graph Neural Networks (GNNs) have advanced machine learning on structured data, such as social networks, molecular biology, and real-time systems~\citep{kipf2016semi, gilmer2017neural, li2018adaptive}. By leveraging node features and topological connections, GNNs excel in predictive tasks such as node classification and link prediction. However, in fields where erroneous confidence can lead to life-threatening risks or significant economic losses---such as autonomous systems and drug discovery---accurate predictions alone are not enough. Quantifying uncertainty helps assess model confidence and uncover hidden risks in predictions~\citep{gaudelet2021utilizing, zhang2024trajectory}. The complex relationships between nodes in the graph data add further challenges in quantifying uncertainty, which is crucial to ensure that these GNN models can be safely and reliably used in high-stakes settings.

In the probabilistic view, uncertainty can be broadly categorized into aleatoric and epistemic uncertainties. Aleatoric uncertainty arises from intrinsic data noise and randomness, such as variations in measurements or sampling conditions. In a social graph network analysis, aleatoric uncertainty might appear due to variability in user interactions. In contrast, epistemic uncertainty results from the model's limited knowledge and occurs when the model encounters unfamiliar patterns or out-of-distribution (OOD) inputs, as seen in cases of fraud detection and drug discovery with unseen data points. 
Several methods have been explored to capture uncertainties in GNNs~\citep{zhao2020uncertainty, stadler2021graph, wu2023energy, hart2023improvements}.  Among these, Evidential Deep Learning (EDL) by~\cite{sensoy2018evidential} has shown promise due to its efficient computation of uncertainty in a single forward pass and its interpretable results based on evidential theory. However, EDL-based models often require architectural configurations different from conventional deterministic predictive models. Consequently, they must be trained from scratch,  limiting their scalability and compatibility with pre-trained models.

To integrate existing pre-trained classification GNNs with EDL, we propose a plug-and-play framework that does not alter the pre-trained model. Instead, we attach a lightweight uncertainty quantification head to extract evidence from the pre-trained model's representations, referred to as the Evidential Probing Network (EPN). This design preserves EDL's computational efficiency, reduces the number of trainable parameters, and enables flexible integration with various pre-trained GNN models. It is particularly effective in scenarios with limited training data, constrained computational resources, and strict trustworthiness requirements.

The contributions of this work are fourfold. First, we propose a simple probe network attached to a pre-trained classification network and train it using evidential deep learning theory to capture multidimensional uncertainties. Second, we develop evidence-based regularizations to improve the quality of predicted uncertainties, especially in the task of OOD detection. The combination of EPN and regularizations is referred to as EPN-reg. Third, we provide a series of theoretical analyses on the limitations of the standard uncertainty-cross-entropy (UCE) loss when training the EPN and demonstrate how our regularizations address these issues. Lastly, we evaluate our approach on multiple datasets, where the proposed EPN-reg consistently outperforms baselines while maintaining real-time efficiency. Specifically, EPN ranks among the top two in 60 out of 150 OOD detection cases and achieves the best average performance in both OOD and misclassification detection. In summary, this work enhances GNNs' practicality, explainability, and reliability for critical real-world applications.
The code can be found in GitHub repository \footnote{\url{https://github.com/linlin-yu/Evidential-Probing-GNN.git}}.

\section{Related Work}

\textbf{Uncertainty quantification in machine learning} has been explored through various approaches, including Bayesian methods, ensembles, deterministic techniques, and post-hoc recalibration methods. Bayesian methods estimate distributions over network parameters using techniques such as Monte Carlo dropout~\citep{gal2016dropout} and leverage information theory to quantify uncertainty. Ensemble methods train multiple independent models and use the differences in their predictions to measure uncertainty~\citep{lakshminarayanan2017simple}. Deterministic approaches predict prior-based distributions~\citep{ulmer2021prior, sensoy2018evidential} to capture multiple dimensions of uncertainty. Lastly, post-hoc recalibration methods assess uncertainty after training, e.g., by computing energy scores~\citep{liu2020energy}, to obtain uncertainty estimates without modifying the underlying model architecture.

\textbf{Evidential Deep Learning.} \cite{sensoy2018evidential} first introduced Evidential Deep Learning (EDL) for classification tasks. Instead of predicting a single-point estimate of class probabilities, an EDL classification model outputs the parameters of a Dirichlet distribution over the classes. From the perspective of subjective logic, the concentration parameters of this Dirichlet distribution represent the model’s confidence, or evidence, for each class outcome.
\cite{amini2020deep} extended EDL to regression by having the network predict the parameters of a Normal-Inverse-Gamma distribution, which serves as the conjugate prior for a Gaussian likelihood. \cite{charpentier2020posterior} further generalized EDL with the Natural Posterior Network, allowing it to handle any target distribution within the exponential family. This framework has its application in classification, regression, and count prediction tasks.
Ongoing work has focused on refining and analyzing EDL itself. For example, relaxing-EDL~\citep{chen2024r} explores the relaxation of subjective logic assumptions to improve robustness. Additionally, \cite{yu2023uncertainty, bengs2022pitfalls} highlight limitations in the optimization loss used for epistemic uncertainty estimation, raising concerns about its ability to reliably reflect model uncertainty. 

\textbf{Uncertainty quantification in graph.} Compared to uncertainty quantification in input-independent scenarios, methods specifically designed for input-dependent settings, such as node classification in graphs, remain less explored.
\cite{zhao2020uncertainty} first extended the EDL to graph data using knowledge distillation via graph kernel density estimation (GKDE) and Bayesian learning. \cite{stadler2021graph} adapted posterior networks for graphs, introducing class-wise evidence propagation through GNN layers. 
In this paper, we also explore a direct application of EDL to GNNs, which we refer to as EGNN.
Additionally, \citet{wu2023energy} introduced energy-based models to graphs, using node energy-based label propagation to enhance OOD detection performance.

\section{Preliminary}
\subsection{Problem Formulation}
We consider a graph $\mathcal{G} = (\mathbb{V}, \mathbf{A}, \mathbf{X}, \mathbf{y}_{\mathbb{L}})$, where $\mathbb{V}=\{1, \cdots, N\}$ represents the set of nodes, $\mathbf{A}$ is the adjacency matrix, and $\mathbf{X}\in \mathbb{R}^{N\times F}$ is the node feature matrix. The set of labeled nodes is denoted by  $\mathbf{y}_{\mathbb{L}} = \{ \mathbf{y}^i | i\in \mathbb{L}\}$ , where the index set $\mathbb{L} \subset \mathbb{V}$ and each element ${\bf y}^i$ is a one-hot vector corresponding to one of $C$ class labels.

An EGNN outputs both class prediction and associated uncertainty by predicting a Dirichlet distribution over class probabilities for each node. The Dirichlet distribution is parameterized by concentration parameters  
$\boldsymbol{\alpha}^i = [\alpha_1^i, \cdots, \alpha_C^i] $ for each node $i$, with the constraint  $\alpha^i_c > 1\  \forall c\in [C],$ to ensure a non-degenerate distribution.

An EGNN follows the same network architecture as a GNN for classification but differs in its final activation function. Instead of using softmax to generate probability distributions, EGNN employs exponential or SoftPlus~\citep{pandey2023learn}
activations to ensure non-negative and unbounded outputs.
The predicted Dirichlet distribution serves as a conjugate prior to a categorical distribution,  which parameterizes the posterior probability of the class probabilities, i.e.,
\begin{align}
    {\bf y}^i \sim \mathtt{Cat}({\bf p}^i),\ \  {\bf p}^i \sim \mathtt{Dir}({\bf p}^i | {\bm \alpha}^i),
    \label{eq: prediction}
\end{align} 
and the expected class probabilities are given by:
\begin{equation}\label{eq:p-alpha}
    \bar{\mathbf{p}}^i = \frac{\boldsymbol{\alpha}^i}{\alpha_0^i}, \ \ \alpha_{0}^i = \sum_c \alpha_c^i.
\end{equation}

\subsection{Uncertainty Quantification in EGNN}
From the perspective of subjective logic, the Dirichlet parameters can be interpreted as a measure of support (evidence) for each class, derived from the labeled training data. The evidence for class $c$ at node $i$ is defined by:
$e_c^i = \alpha_c^i - 1 > 0$, as $\alpha_c^i>1.$ 
Specifically, given the adjacency matrix $\mathbf A$ and the feature matrix $\mathbf X,$ the EGNN model can be expressed by %
\begin{align}
   & \text{Evidence:}\ \  [\mathbf{e}^i]_{i \in \mathbb{V}} = f_{\text{EGNN}}({\mathbf{A}}, {\mathbf X}; {\bm{\theta}_{\text{EGNN}}}) \\
   & \text{Dirichlet parameters:}  \ \ \boldsymbol{\alpha}^i = \mathbf{e}^i + \mathbf{1} ,
\end{align}
where $f_{\text{EGNN}}(\cdot)$ denotes the neural network function parameterized by ${\bm{\theta}_{\text{EGNN}}}$. 
The \textit{total evidence} for node $i$, which reflects the overall evidence accumulated across all classes, is defined by,
\begin{equation}
e_{\text{total}}^i = \sum_c e_c^i = \alpha_0^i - C,
\end{equation}
where $\alpha_0^i$ is defined in (\ref{eq:p-alpha}).

We define two types of uncertainty:

$\bullet$ \textit{Aleatoric uncertainty}, which represents the uncertainty in the class label, is calculated based on the expected class probabilities. We measure the aleatoric uncertainty by taking the negative of the highest expected class probability $\bar{\bf p}^i$ given in (\ref{eq:p-alpha}),
\begin{align}\label{eq:alea_u}
    u_{\text{alea}}^i &= -\text{max} \{ \bar{p}_{1}^i, \cdots, \bar{p}_{C}^i\}.
\end{align}
A higher value of $u_{\text{alea}}^i$ indicates a greater likelihood of an incorrect class prediction. 

$\bullet$ \textit{Epistemic uncertainty}, which quantifies uncertainty in the class probabilities, is defined by
\begin{align}\label{eq:epis_u}
     u_{\text{epi}}^i = \frac{C}{e_{\text{total}}^i+C}.
 \end{align}
 Higher epistemic uncertainty suggests that the Dirichlet distribution is more dispersed, indicating that the model lacks confidence in its predictions. From the subjective logic perspective, epistemic uncertainty corresponds to \textit{vacuity}, meaning a lack of supporting evidence for the prediction.

\subsection{Optimization of EGNN}
An EGNN is typically trained by minimizing the uncertainty cross-entropy (UCE) loss function, defined by,
\begin{equation} \label{eq: enn_uce}
    \begin{aligned}
        & \ \ell^i_{\text{UCE}}(\bm{e}^i, {\bf y}^i; \bm{\theta}_{\text{EGNN}}) \\
  =&\ \mathbb{E}_{\mathbf{p}^i \sim \text{Dir}(\mathbf{p}^i | \boldsymbol{\alpha}^i)} \left[ - \log \mathbb{P}(\mathbf{y}^i | \mathbf{p}^i) \right] \\
  =& \ \psi\left(e_{\text{total}}^i + C\right) - \sum_c y^i_c\psi\left(e_{c}^i + 1\right),
    \end{aligned}
\end{equation}
where $\psi(\cdot)$ is the digamma function and $\mathbb{P}(\mathbf{y}^i | \mathbf{p}^i )$ is the sampled probabilities of node $i$ belonging to ground truth class $\mathbf{y}^i$. 
The UCE loss can be interpreted as the expected cross-entropy loss over all possible class probability distributions $\mathbf{p}^i$, assuming they follow a Dirichlet distribution. 

In this study, we focus on the UCE loss and defer the exploration of alternative loss functions, such as the expected mean squared error loss and the expected log loss in \cite{sensoy2018evidential}, to future work.

\section{Methodology}

\subsection{Evidential Probe Network}

EDL has demonstrated strong capabilities in capturing uncertainties by modeling second-order distributions~\citep{sensoy2018evidential,ulmer2021prior}. However, their implementation typically requires custom network architectures and training from scratch, making them incompatible with widely available, powerful pre-trained models. Additionally, training specialized networks may be infeasible due to limited access to training data or computational resources.

To address these challenges, we propose an Evidential Probe Network (EPN) that quantifies uncertainty by attaching a lightweight network to a pre-trained GNN. Only this small additional module needs to be trained, without modifying or retraining the original classification model. This approach allows efficient uncertainty estimation while preserving the benefits of existing pre-trained GNNs.

Suppose a pre-trained 
GNN is available and denoted by: 
\begin{equation}\label{eq:pGNN}
    [\tilde{\mathbf{p}}^i]_{i \in \mathbb{V}} = f_{\text{GNN}}(\mathbf{A}, \mathbf{X}; \boldsymbol{\theta}_{\text{GNN}}),
\end{equation}
where $\tilde{\mathbf{p}}^i$ represents the predicted class probability vector for node $i$, based on a GNN function $f_{\text{GNN}}$ parameterized by $\boldsymbol{\theta}_{\text{GNN}}$. 

We decompose GNN into a \textit{representation-learning network} (RLN) and a \textit{classification head}. The RLN, typically consisting of graph convolutional and Multi-Layer-Perceptron (MLP) layers, generates node-level hidden representations $[\mathbf{z}^i]_{i \in \mathbb{V}}$ that encode both node features and graph structural information. 
The classification head, composed of the final MLP layers, processes these embeddings to produce node-level probability vectors.

To estimate uncertainty, we introduce a lightweight two-layer MLP as a probe network, which takes the hidden representations $[\mathbf{z}^i]_{i \in \mathbb{V}}$ as input and predicts the overall evidence $e_{\text{total}}^i$ for each node $i$ by
\begin{equation}\label{eq:eEPN}
    e_{\text{total}}^i = f_{\text{EPN}}(\mathbf{z}^i; \bm{\theta}_{\text{EPN}}),
\end{equation}
where $f_{\text{EPN}}(\cdot)$ represents the probe network with a non-negative activation function before the final output and  $\bm{\theta}_{\text{EPN}}$ denotes the EPN's learnable parameters. 

To derive the concentration parameters in the evidential framework, we use the predicted class probabilities $\tilde{\mathbf{p}}^i$  from the GNN to approximate the expected class probabilities $\bar{\mathbf{p}}^i$. The Dirichlet parameters are then computed by,
\begin{equation} 
    \boldsymbol{\alpha}^i = (C + e_{\text{total}}^i) \cdot \tilde{\mathbf{p}}^i. \label{eqn:tao-prob-2-ealpha}
\end{equation}
The Dirichlet parameters $\bm{\alpha}^i$ are used to quantify both aleatoric and epistemic uncertainties, defined in (\ref{eq:alea_u}) and (\ref{eq:epis_u}), respectively.

Inspired by evidence propagation in~\cite{stadler2021graph} and energy propagation in~\cite{wu2023energy}, we propose two propagation schemes, namely \textit{vacuity-prop} and \textit{evidence-prop}, to enhance uncertainty estimation by leveraging graph structure. The key idea is that neighboring nodes within a graph should exhibit similar uncertainty levels.
Specifically, the \textit{vacuity-prop} scheme propagates epistemic uncertainty (vacuity) across nodes, whereas the \textit{evidence-prop} scheme propagates Dirichlet parameters.  Additionally, we introduce a hybrid approach that combines both class-wise evidence and vacuity propagation to further improve uncertainty estimation. Detailed formulations of these methods can be found in Appendix ~\ref{sec: baselines}.

These propagation techniques are applied to both EGNN and EPN models, effectively enhancing uncertainty quantification while preserving the underlying graph structural information.

\subsection{Optimization of EPN}

Following EGNN, the probe network is trained by minimizing the UCE loss. Specifically, we define the UCE loss used in EPN, denoted by EPN-UCE, 
\begin{equation}\label{eq:uce_epn}
\begin{aligned} 
    &\ \ell^i_{\text{EPN, UCE}}(e_{\text{total}}^i, \mathbf{y}^i, \tilde{\mathbf{p}}^i; \boldsymbol{\theta}_{\text{EPN}}) \\
   =&\ \mathbb{E}_{\mathbf{p}^i \sim \text{Dir}(\mathbf{p}^i | \boldsymbol{\alpha}^i)} \left[ - \log \mathbb{P}(\mathbf{y}^i | \mathbf{p}^i) \right] \\
  =&\ \psi\left(e_{\text{total}}^i + C\right) - \sum\nolimits_c y^i_c\psi\left(\left(e_{\text{total}}^i +C \right) \cdot \tilde{p}_c^i\right), 
\end{aligned}
\end{equation}
where \(\tilde{\mathbf{p}}^i\) is the predictedclass probability vector from the pre-trained classification model $f_{\text{GNN}}$ in (\ref{eq:pGNN}), \( e_{\text{total}}^i \) is the total evidence predicted by the probe network \(f_{\text{EPN}}(\mathbf{z}^i; \bm{\theta}_{\text{EPN}})\) in (\ref{eq:eEPN}), and \(\mathbf{y}^i\) represents the ground-truth one-hot class label.
Minimizing the UCE loss (\ref{eq:uce_epn}) optimizes the alignment between the sampled class probabilities from the predicted Dirichlet distribution and the true class labels in the training set.

Although EPN is computationally efficient, the EPN-UCE loss alone is ineffective for uncertainty quantification,  as the model tends to assign high evidence to both ID and OOD nodes.
An ideal model should express greater uncertainty for OOD nodes by assigning lower evidence. However, since the training data consists solely of ID nodes, the UCE loss does not explicitly enforce this behavior, often resulting in overconfident predictions on unfamiliar inputs.
To address this drawback, we introduce two regularizations that guide the model toward more accurate uncertainty estimation, enhancing its ability to differentiate between ID and OOD samples.

\textbf{Intra-Class Evidence-Based (ICE)} regularization encourages the model to preserve class-label information in the latent space by clustering same-class samples and pushing apart different-class samples. It leverages class labels explicitly to capture intra-class structure and enhance uncertainty estimation.
Specifically, we design the final hidden layer of EPN to produce a hidden representation  $\mathbf{q}^i \in \mathbb{R}^C$, which serves as a proxy for class-level evidence,
facilitating effective knowledge distillation. 
The ICE regularization is defined by,
 \begin{equation}\label{eq:ICE}
     \begin{aligned}
    \ell^i_{\text{ICE}}(e_{\text{total}}^i, \mathbf{y}^i;& \boldsymbol{\theta}_{\text{EPN}}) \\
    = &\ \left\lVert (C + e_{\text{total}}^i) \cdot \tilde{\mathbf{p}}^i - \mathbf{q}^i \right\rVert_2^2.
\end{aligned}
 \end{equation}
Note that, although the ICE term appears to link predicted evidence to the hidden representation, it does not create a circular dependency. This is because the predicted total evidence ($e_{\text{total}}^i$) is derived independently from the probe network's final layer, while the class-wise probability vector ($\tilde{\mathbf{p}}^i$) is fixed and obtained from the pre-trained GNN model. Thus, the ICE regularization, as a distillation mechanism, guides the hidden layer to reflect meaningful class-level information without causing a chicken-and-egg dilemma.

\textbf{Positive-Confidence Learning (PCL)} regularization, inspired by~\cite{ishida2018binary}, provides a weak supervision in scenarios where only a single class label (ID) is present and explicit labels for the alternative class (OOD) are unavailable in the training set.
Specifically, we interpret the confidence scores obtained from the pre-trained model as the probability of a node being an ID sample. Using these confidence scores, we introduce a regularization term based on the hinge loss:
\begin{equation}
    \begin{aligned}
            & \ell^i_{\text{PCL}}(e_{\text{total}}^i, \mathbf{y}^i, \tilde{\mathbf{p}}^i; \boldsymbol{\theta}_{\text{EPN}}) \\
        = & \left( \max(0, e_{\text{total}}^{\text{id}} - e_{\text{total}}^i) \right)^2  \\
         & \qquad + \frac{1 - r^i}{r^i} \left( \max(0, e_{\text{total}}^i - e_{\text{total}}^{\text{ood}}) \right)^2,
\end{aligned} 
\end{equation}
where $r^i = \max\{\tilde{p}_1, \dots, \tilde{p}_C\}$ is the confidence score from the pre-trained model, $e_{\text{total}}^{\text{in}} > e_{\text{total}}^{\text{ood}}$ are two predefined margin parameters. Note that $e_{\text{total}}^i$ represents the predicted total evidence for node $i$, while $e_{\text{total}}^{\text{id}}$ is a fixed evidence level.

The PCL term pushes the total evidence values within $[e_{\text{total}}^{\text{ood}}, e_{\text{total}}^{\text{in}}]$ to be higher for confident ID samples and lower for less confident ID samples, which is what we expect. Specifically, nodes with higher confidence scores are encouraged to produce evidence closer to the upper bound, while those with lower confidence scores are expected to yield evidence closer to the lower bound. PCL does not penalize the model when the predicted evidence exceeds the upper threshold ($e_{\text{total}}^i > e_{\text{total}}^{\text{id}}$) for confident ID nodes. Otherwise, the model incurs a penalty proportional to the deviation from the lower threshold. 
In practice, we set $e_{\text{total}}^{\text{ood}} = 0$ and $e_{\text{total}}^{\text{id}}=100$. 

\textbf{Regularized Learning Objective.} To sum up, we propose the following objective function for training a  network,
\begin{equation}
\begin{aligned}
    \mathcal{L}(\bm{\theta}_{\text{EPN}}) &= \frac{1}{\|\mathbb{L}\|}\sum_{i \in \mathbb{L}} \ell^i_{\text{EPN,UCE}}(e_{\text{total}}^i, \mathbf{y}^i, \tilde{\mathbf{p}}^i; \boldsymbol{\theta}_{\text{EPN}}) 
    \\
    &\qquad + \frac{1}{\|\mathbb{V}\|}\sum_{i \in \mathbb{V}} \Big( \lambda_1\ell^i_{\text{ICE}}(e_{\text{total}}^i, \mathbf{y}^i, \tilde{\mathbf{p}}^i; \boldsymbol{\theta}_{\text{EPN}})   
    \\
    &\qquad +\lambda_2\ell^i_{\text{PCL}}(e_{\text{total}}^i, \mathbf{y}^i, \tilde{\mathbf{p}}^i; \boldsymbol{\theta}_{\text{EPN}}) \Big),
\end{aligned}
\end{equation}
where $\lambda_1$ and $\lambda_2$ are positive hyperparamters.  The objective function combines the UCE loss as the primary optimization target, the ICE regularization to preserve intra-class distinctions in the latent space, and the PCL regularization to provide weak supervision that guides evidence prediction based on model confidence. 

\section{Theoretical Analysis}
To establish the theoretical properties of our proposed EPN, we follow a simplified setting considered in \cite{yu2023uncertainty}. Specifically, we analyze a binary classification task designed for distinguishing between ID and OOD instances to investigate the behavior of EPNs in terms of epistemic uncertainty quantification.

\textbf{Problem Setup.} In this section, we focus on a binary node-level classification task ($C=2$), noting that generalization to multiple classes $(C>2)$ can be analyzed similarly, as shown by~\cite{collins2023provable}. For simplicity, we omit the node index $i$ when the context is clear. Let ${\bf z}$ denote a hidden representation vector produced by the RLN for a node in the graph. We use labels $-1$ and $+1$ to represent the two ID classes and label $0$ to represent the OOD class. 
We assume that the conditional distributions for the two ID classes are Gaussian with symmetric means $\pm \bm\mu$ and identical covariance \( \bm{\Sigma} \), i.e.,
${\bf z} \mid Y = -1 \sim N(-\bm{\mu}, \bm{\Sigma})$, ${\bf z} \mid Y = +1 \sim N(\bm{\mu}, \bm{\Sigma})$. The conditional distribution for the OOD class is also Gaussian but centered at the origin: ${\bf z} \mid Y = 0 \sim N(\bm{0}, \bm{\Sigma})$. 
We further assume a uniform class prior: $Y\sim \mathtt{Cat}(\frac{1}{3}, \frac{1}{3}, \frac{1}{3})$. 

Under these distribution assumptions on the hidden variable ${\bf z}$, we investigate an asymptotic behavior of EGNN and EPN networks as the magnitude of the mean vector \( \|{\bm \mu}\|_2 \rightarrow \infty \), corresponding to sufficiently separable class distributions. Specifically, we consider that both networks receive the same hidden representation vector \({\bf z}\) as input. The EGNN outputs a class-level evidence vector \({\bf e}\), while the EPN outputs the overall evidence value \(e_{\text{total}}\). For the pre-trained classification model used in our EPN, we assume an MLP classification head that also takes \({\bf z}\) as input and outputs the class probability vector \(\tilde{{\bf p}} \).

\textbf{EGNN Optimized with the Upper Bound of UCE Loss.} 
We demonstrate that an optimally trained EGNN can effectively capture epistemic uncertainty under reasonable assumptions. Specifically, %
we consider an EGNN architecture defined as a single-layer MLP:
\begin{align} \label{eq:evidence_one_neuron}
f_{\text{EGNN}}(\mathbf{z}; \bm{\theta}_{\text{EGNN}}) 
= \begin{bmatrix} \exp(-\bar{\mathbf{w}}^\top \cdot \mathbf{z} - \bar{b}) \\ \exp(\bar{\mathbf{w}}^\top \cdot \mathbf{z} + \bar{b}) \end{bmatrix},
\end{align}
where $\bm{\theta}_{\text{EGNN}} = \{\bar{\mathbf{w}}, \bar{b}\}$ refers to the EGNN parameters.

Under the Gaussian data assumption and the EGNN architecture (\ref{eq:evidence_one_neuron}), we utilize an upper bound of the UCE loss given by~\cite{yu2023uncertainty}:
\begin{equation}\label{eq:upper}
    \overline{\ell_{\text{EGNN,UCE}}^i}
    \;=\;
    \frac{2}{\,\sum_c y_c^i e^i_c\,},
\end{equation}
where $e^i_c$ is the predicted evidence of the node $i$ belonging to class $c$ from the EPN network $f_{\text{EPN}}$.
Note that this upper bound (\ref{eq:upper}) leads to an analytical solution of model parameters, which facilitate the proof of Theorem \ref{thm:1}.

\begin{theorem}\label{thm:1}
For any \(\epsilon > 0\), there exists a positive constant \(\digamma > 0\) such that, for any data distribution satisfying Gaussian data assumption with \(\|\bm{\mu}\|_2 > \digamma\), the probability that the epistemic uncertainty obtained by an optimal single-layer EGNN based on an upper bound of UCE loss, correctly distinguishes ID and OOD samples is greater than \(1 - \epsilon\). 
\end{theorem}

The intuition behind Theorem~\ref{thm:1} is that, given a sufficient separation between the class means (\( \|\bm{\mu}\|_2 \to \infty \)), the optimal EGNN model %
achieves near-perfect detection of OOD samples, thus providing a reliable measure of epistemic uncertainty. 

\textbf{EPN Optimized with the EPN-UCE Loss.} 
The proposed EPN uses a MLP to predict a total evidence value for uncertainty quantification. In our theoretical analysis, we consider a two-layer MLP as the architecture of the EPN, defined as follows:
\begin{equation}
\begin{aligned}
  & e_{\text{total}} = f_{\text{EPN}}(\mathbf{z}; \boldsymbol{\theta}_{\text{EPN}}) \\ &=
\text{ReLU} \Bigl(\mathbf{w}^{[2]\top}\,\exp
\bigl(\mathbf{W}^{[1]} \,\mathbf{z}
+ \mathbf{b}^{[1]}\bigr)
+ b^{[2]}\Bigr),
\end{aligned}
\end{equation}
where $\mathbf{W}^{[1]} \in \mathbb{R}^{C\times d}$, $\mathbf{w}^{[2]} \in \mathbb{R}^C$, $\mathbf{b}^{[1]} \in \mathbb{R}^C$, $b^{[2]}\in \mathbb{R},$ and $\boldsymbol{\theta}_{\text{EPN}} = \{\mathbf{W}^{[1]},\mathbf{w}^{[2]}, \mathbf{b}^{[1]}, b^{[2]} \}$. 
We reveal a scenario where optimizing EPN solely with the EPN-UCE loss does not necessarily yield reliable epistemic uncertainty estimates, which may impair OOD detection. In particular, we consider the following parameter configuration $ \tilde{\boldsymbol{\theta}}_n=\{\tilde{\bf W}^{[1]}, \tilde{\bf w}^{[2]}, \tilde{\bf b}^{[1]}, \tilde{b}^{[2]}\}:$
\begin{equation}\label{eq: theta_tilde}
\begin{aligned}
    &\tilde{\bf W}^{[1]} = {\bf 0},  \tilde{\bf w}^{[2]} = {\bf 1}, 
    \tilde{\bf b}^{[1]} = n \cdot \begin{bmatrix}
1 \\
-1 
\end{bmatrix} , \ \tilde{b}^{[2]} = 0, 
\end{aligned}
\end{equation}
where $n$ is a scalar. Note that the only varying component in $\tilde{\boldsymbol{\theta}}_n$ is $\tilde{\bf b}^{[1]}$, which scales with $n$. 
\begin{theorem} \label{thm:EPN-UCE}
Given 
a two-layer EPN network with parameters $\tilde{\boldsymbol{\theta}}_n$ defined in (\ref{eq: theta_tilde}), the corresponding EPN-UCE loss, $\ell_{\text{EPN,UCE}}(\mathbf{z}, \mathbf{y}; \tilde{\boldsymbol{\theta}}_n)$, attains its infimum asymptotically at infinity, i.e., as $n\rightarrow \infty$. Furthermore, this parameterized EPN has the property:
\begin{align}
  & \mathbb{P}\Bigl(u_{\text{epi}}\bigl(f_{\text{EPN}}(\mathbf{z}_{y=0}; \tilde{\boldsymbol{\theta}}_n)\bigr) > u_{\text{epi}}(f_{\text{EPN}}\bigl(\mathbf{z}_{y \in \{-1, +1\}}; \tilde{\boldsymbol{\theta}}_n\bigr)\Bigr) \notag \\
  & = 0.
\end{align}
\end{theorem}
Theorem~\ref{thm:EPN-UCE} indicates that there exists a set of parameters that allows the two-layer EPN to attain the infimum of the expected EPN-UCE loss, but such solutions may fail to produce the correct epistemic uncertainty ordering. 
 
\textbf{Regularized EPN with ICE.} The proposed EPN is designed to predict total evidence without explicitly accounting for ID classification. To incorporate class-specific structure, we introduce the ICE regularization, which 
aligns a latent space representation of EPN ($\mathbf{q}^i$) with the class probability vector ($\mathbf{p}^i$)  from a pre-trained classification model up to a scaling factor of $C + e_{\text{total}}^i$, thereby enforcing intra-class consistency.

When the Gaussian distributions of the ID and OOD classes become sufficiently separable, a simple two-layer MLP as the classification network is sufficient for classification. %
Under this scenario, we define a pre-trained classification head that outputs the optimal class probability vector $\tilde{\mathbf{p}}$, which has an analytical formula provided in Appendix~\ref{sec: them_3}.

To enable a theoretical analysis similar to that of the EGNN (Theorem~\ref{thm:1}), we construct a specific EPN architecture with its weight and bias matrices given by:
\begin{equation}\label{eq: epn_param}
\begin{aligned}
& \mathbf{W}^{[1]} = [\mathbf{w}_{\text{P}}, -\mathbf{w}_{\text{P}}]^\top, 
\mathbf{b}^{[1]} = [b_{\text{P}}, -b_{\text{P}}]^\top, \\
& \mathbf{w}^{[2]} = {\bf 1},  b^{[2]} = 0,
\end{aligned}
\end{equation}
where \( \mathbf{w}_{\text{P}} \in \mathbb{R}^d, b_{\text{P}}\in \mathbb{R} \).  
Using this  EPN and the optimal class probabilities $[\tilde{\mathbf{p}}^i]_{i \in \mathbb{V}}$, we prove in Theorem \ref{thm:EPN-ICE} that an EPN trained solely on ICE can correctly distinguish OOD from ID samples.
\begin{theorem}\label{thm:EPN-ICE} 
Given a well-trained classification model producing the class probabilities 
 $[\tilde{\mathbf{p}}^i]_{i \in \mathbb{V}}$ and an EPN constructed by (\ref{eq: epn_param}), we have that for any \(\epsilon > 0\), there exists a positive constant \(\digamma > 0\) such that, for any data distribution with \(\|\bm{\mu}\|_2 > \digamma \), the probability that the epistemic uncertainty obtained by an optimal two-layer EPN solely based on the  ICE loss, correctly distinguishes ID and OOD samples is greater than \(1 - \epsilon\).
\end{theorem}

Notably, the optimal EPN parameters under ICE coincide with those of EGNN under our assumptions; see Assumption~\ref{assum:data} in Appendix \ref{sect:proofs}. Theorem~\ref{thm:EPN-ICE} thus confirms that EPN trained with ICE can help the learning of epistemic uncertainty under assumptions.
Please refer to Appendix \ref{sect:proofs} for all the proofs in this section.

\begin{table*}[htbp]
    \centering
    \small
    \setlength{\tabcolsep}{1.0pt} 
    \resizebox{0.85\linewidth}{!}{%
    \begin{tabular}{l|cccccccc|c}
    \toprule
    \multirow{2}{*}{\textbf{Model}} & \multirow{2}{*}{\textbf{CoraML}} & \multirow{2}{*}{\textbf{CiteSeer}} & \multirow{2}{*}{\textbf{PubMed}} & \multicolumn{1}{c}{\textbf{OGBN}} & \multicolumn{1}{c}{\textbf{Amazon}} & \multicolumn{1}{c}{\textbf{Amazon}} & \multicolumn{1}{c}{\textbf{Coauthor}} & \multicolumn{1}{c|}{\textbf{Coauthor}} & \multirow{2}{*}{\textbf{Average}} \\
     &  &  &  & \multicolumn{1}{c}{\textbf{Arxiv}} & \textbf{Photos} & \textbf{Computers} & \textbf{CS} & \textbf{Physics} &  \\
    \midrule
    \multicolumn{10}{c}{\textbf{logit based}} \\ 
    \midrule
    VGNN-entropy   & 8.4 & 7.4 & 10.4 & 9.2 & 6.2 & 8.4 & 5.6 & 11.4 & 8.4 \\
    VGNN-max-score & 10.8 & 10.2 & 11.0 & 11.2 & 6.2 & 8.0 & 7.4 & 12.6 & 9.7 \\
    VGNN-energy~\citep{liu2020energy}    & 8.2 & 7.8 & 11.0 & 7.8 & 10.0 & 9.6 & 8.8 & 9.4 & 9.1 \\
    VGNN-gnnsafe~\citep{wu2023energy}   & 4.6 & {\colorbox[rgb]{0.812,0.886,0.953}{4.4}} & 9.4 & 9.0 & 5.0 & 7.4 & 3.4 & 4.4 & 6.0 \\
    VGNN-dropout~\citep{gal2016dropout}   & 10.6 & 8.6 & 11.8 & 8.0 & 9.4 & 8.6 & 8.4 & 9.6 & 9.4 \\
    VGNN-ensemble~\citep{lakshminarayanan2017simple}  & 7.6 & 9.4 & 9.4 & 6.6 & 8.0 & 7.2 & 7.6 & 11.4 & 8.4 \\
    \midrule \multicolumn{10}{c}{\textbf{evidential based}} \\ \midrule
    GPN~\citep{stadler2021graph}            & 9.4 & 11.0 & 5.0 & 4.8 & {\colorbox[rgb]{0.812,0.886,0.953}{2.4}} & {\colorbox[rgb]{0.957,0.8,0.8}{1.4}} & 7.8 & 5.8 & 6.0 \\
    SGNN-GKDE~\citep{zhao2020uncertainty}      & {\colorbox[rgb]{0.812,0.886,0.953}{3.0}} & {\colorbox[rgb]{0.957,0.8,0.8}{1.4}} & 3.8 & n.a.  & 7.2 & 6.2 & 14.0 & {\colorbox[rgb]{0.812,0.886,0.953}{3.8}} & 5.8 \\
    EGNN~\citep{sensoy2018evidential}           & 12.2 & 6.0 & 7.2 & 3.6 & 13.0 & 12.0 & 10.4 & 11.8 & 9.5 \\
    \quad + vacuity-prop & {\colorbox[rgb]{0.957,0.8,0.8}{2.0}} & 6.2 & 4.2 &{\colorbox[rgb]{0.812,0.886,0.953}{2.0}} & 8.8 & 8.6 & {\colorbox[rgb]{0.812,0.886,0.953}{2.8}} & 4.2 & {\colorbox[rgb]{0.812,0.886,0.953}{4.9}} \\
    \quad + evidence-prop & 7.8 & 8.2 & {\colorbox[rgb]{0.812,0.886,0.953}{2.2}} & 9.6 & 11.6 & 8.4 & 9.0 & 4.6 & 7.8 \\
    \quad + vacuity-prop + evidence-prop      & 7.6 & 9.6 & {\colorbox[rgb]{0.957,0.8,0.8}{2.0}} & 6.0 & 9.6 & 8.8 & 6.2 & {\colorbox[rgb]{0.957,0.8,0.8}{2.0}} & 6.7 \\
    \midrule \multicolumn{10}{c}{\textbf{ours}} \\ \midrule
    EPN            & 8.4 & 9.8 & 8.8 & 12.0 & 6.0 & 7.6 & 12.2 & 8.8 & 9.2 \\
   EPN-reg & 4.4 & 5.0 & 8.8 & {\colorbox[rgb]{0.957,0.8,0.8}{1.2}} & {\colorbox[rgb]{0.957,0.8,0.8}{1.6}} & {\colorbox[rgb]{0.812,0.886,0.953}{2.8}} & {\colorbox[rgb]{0.957,0.8,0.8}{1.4}} & 5.2 & {\colorbox[rgb]{0.957,0.8,0.8}{3.8}} \\
    \bottomrule
    \end{tabular}
    }
    \caption{OOD Detection: Average performance ranking on OOD-AUROC  ($\downarrow$) for each model across various LOC settings, using GCN as the backbone. \colorbox[rgb]{0.957,0.8,0.8}{Best} and \colorbox[rgb]{0.812,0.886,0.953}{Runner-up} results are highlighted in red and blue.}
    \label{tab:ood_gcn_rank}
    \vspace{-4mm}
\end{table*}

\section{Experiment}
We evaluate uncertainty quantification for node-level classification tasks on graphs by comparing 14 models across 8 datasets to explore the following key research questions (RQ) on two tasks: OOD detection and misclassification detection.

\textbf{RQ1}: How do the proposed approaches (EPN and EPN-reg) perform compared to baseline models regarding uncertainty quantification and running time? %

\textbf{RQ2}: How do the proposed regularization terms (ICE and PCL) improve EPN's performance?

\textbf{RQ3}: How robust is EPN-reg regarding the GNN backbone, features, and activation function?

\subsection{Setup}

\textbf{Datasets.} We use eight graph datasets, including citation networks (CoraML, CiteSeer, PubMed, CoauthorCS, CoauthorPhysics), two co-purchase datasets (AmazonPhotos, AmazonComputers), and the large-scale OGBN-Arxiv to evaluate scalability.

\textbf{Evaluation.} We evaluate the model from three perspectives: classification, aleatoric uncertainty, and epistemic uncertainty. (1) \textbf{Classification}: We report classification accuracy (ACC), calibration via Brier Score (BS), and Expected Calibration Error (ECE). (2) \textbf{Aleatoric uncertainty}: Misclassification detection is assessed via  Area Under the ROC Curve (AUCROC) and the Area Under the Precision-Recall Curve (AUCPR), with lower uncertainty linked to correct predictions and higher uncertainty indicating misclassification. (3) \textbf{Epistemic uncertainty}: For OOD detection, we report AUCROC and AUCPR, using ``Left-out-classes'' (LOC) for distributional shifts, where we remove certain classes from the training set and include them for testing. 
Since different studies define OOD categories differently, we consider five LOC settings with the public class indexes: selecting the last classes as OOD, following the works of~\cite{zhao2020uncertainty, stadler2021graph}, labeled by OS-1, selecting the first classes as OOD, following~\citet{wu2023energy}, labeled by OS-2, and randomly selecting OOD classes in three additional settings, labeled by OS-3 to OS-5. The number of OOD categories aligns with~\cite{stadler2021graph}. Further details are provided in Appendix~\ref{sec:eva}. 

\textbf{Baselines.} We evaluate the proposed EPN and EPN-reg methods against 12 baseline models from a range of uncertainty-aware techniques for semi-supervised node classification. This includes probability-based models like VGNN-entropy, VGNN-max\_score, VGNN-dropout, VGNN-ensemble, and energy-based methods, including VGNN-energy and VGNN-gnnsafe. For evidential-based methods, we consider SGNN-GKDE, GPN, and three variants of EGNN: EGNN+vacuity-prop, EGNN+evidence-prop, and EGNN+vacuity-prop+evidence-prop. 
The default EPN is optimized using the EPN-UCE loss, whereas EPN-reg incorporates the two proposed regularization terms, ICE and PCL.

We use the recommended hyperparameters from the respective literature for the baselines. For our proposed EPN, we tune the two parameters of $\lambda_1$ and $\lambda_2$ using one OOD detection setting and apply the same hyperparameters across other ones. By default, we use a two-layer GCN as the backbone for all models except GPN. For EPN, GCN is also used as the backbone of a pre-trained model. Additionally, when calculating aleatoric uncertainty, we project the class probabilities and add a small value (e.g., 1) to the Dirichlet strength (\ref{eqn:tao-prob-2-ealpha}) to ensure a stable, non-degenerate Dirichlet distribution during training.

Please refer to Appendix~\ref{sec: baselines} for a more detailed description of baseline methods, including model architectures, loss functions, and hyperparameters.

\subsection{OOD and Misclassification Detection}
Due to space limitations, we present average performance rankings in the main paper, with complete results available in Appendix~\ref{sec: add_exp}.

\textbf{OOD detection.} Table~\ref{tab:ood_gcn_rank} presents the average performance rankings of various methods across different OOD settings, with the last column summarizing overall performance (complete results in Tables~\ref{tab:ood_gcn_1}--\ref{tab:ood_gcn_4}). Lower ranks indicate better performance, which is denoted by ($\downarrow$). Our proposed model, EPN-reg, achieves the best overall performance. Specifically, EPN-reg ranks first on 3 out of 8 datasets, second on 1 dataset, and third on 2 datasets. For the remaining two datasets, our method also ranks among the top-performing models. When comparing EPN with EPN-reg, the additional regularization terms consistently enhance performance across all datasets. It is important to note that baseline model performance varies significantly across datasets. For example,  SGNN-GKDE ranks among the top three performers on four datasets, but it performs the worst on CoauthorCS. Additionally, propagation techniques that incorporate vacuity or class-wise evidence boost EGNN's performance, making it competitive with GPN and SGNN-GKDE, which apply knowledge distillation.

\begin{table}[htbp]
    \centering
    \small
    \setlength{\tabcolsep}{1pt}
    \resizebox{0.7\linewidth}{!}{%
    \begin{tabular}{l|c}
    \toprule
    \textbf{Model} & \textbf{Average Rank} \\
    \midrule
        VGNN-entropy & 9.9  \\ 
        VGNN-max-score & 5.4  \\ 
        VGNN-energy & 13.0  \\ 
        VGNN-gnnsafe & 11.8  \\ 
        VGNN-dropout & 9.3  \\ 
        VGNN-ensemble & 9.0  \\ \midrule
        GPN & 6.4  \\ 
        SGNN-GKDE & 12.6  \\ 
        EGNN & 4.3  \\ 
        \quad + vacuity-prop & 5.0  \\ 
        \quad + evidence-prop & 5.3  \\ 
        \quad +vacuity-prop + evidence-prop & 6.8  \\ \midrule
        EPN & \colorbox[rgb]{0.957,0.8,0.8}{3.0}  \\ 
        EPN-reg & \colorbox[rgb]{0.812,0.886,0.953}{3.5}  \\ 
    \bottomrule
    \end{tabular}}
    \caption{Misclassification detection: Average performance ranking based on MIS-AUROC ($\downarrow$) for each model across various datasets, using GCN as the backbone. \colorbox[rgb]{0.957,0.8,0.8}{Best} and \colorbox[rgb]{0.812,0.886,0.953}{Runner-up} results are highlighted in red and blue.}
    \label{tab:mis_rank}
\end{table}

\textbf{Misclassification detection.} Table~\ref{tab:mis_rank} presents the average misclassification detection performance in terms of AUROC across datasets (complete results in Tables~\ref{tab:mis_1}--\ref{tab:mis_2}). Our model without regularizations achieves the best overall performance. %
A detailed analysis reveals that while all models perform similarly, our models (EPN-reg) exhibit greater stability across different datasets.

\textbf{Running time.} 
Table~\ref{tab:running_citeseer_ogbn} reports the training time, showing that EPN and EPN-reg are the fastest on AmazonComputers and OGBN-Arxiv, with EPN being five times faster than the next best model on OGBN-Arxiv. VGNN-ensemble is significantly slower because it requires training multiple models. 
Note that run times vary by epoch count and early stopping. For example, GPN requires many epochs to converge on OGBN-Arxiv. VGNN-dropout has a slightly shorter training time than VGNN-entropy; its higher inference time is offset by randomness in the training epochs.

\begin{table}[ht]
    \centering
    \small
    \setlength{\tabcolsep}{1pt}
    \resizebox{0.85\linewidth}{!}{%
    \begin{tabular}{l|cc}
    \toprule
    \textbf{Model} & \textbf{AmazonComputers}& \textbf{OGBN\_Arxiv} \\
    \midrule
    VGNN-entropy & 161.8& 1117.9  \\ 
    VGNN-dropout & 153.8& 1134.0  \\ 
    VGNN-ensemble & 1086.5& 9843.7  \\ \midrule
    SGNN-GKDE & 258.1& n.a  \\ 
    GPN & 77.1& 5512.5  \\ 
    EGNN & 117.7& 437.3  \\ 
    \quad + vacuity-prop & 125.0& 385.6  \\ 
    \quad + evidence-prop & 170.1& 576.6  \\ 
    \quad + vacuity-prop + evidence-prop & 144.7& 605.3  \\ \midrule
    EPN & \colorbox[rgb]{0.812,0.886,0.953}{76.1}& \colorbox[rgb]{0.812,0.886,0.953}{80.2}\\ 
    EPN-reg & \colorbox[rgb]{0.957,0.8,0.8}{45.8}& \colorbox[rgb]{0.957,0.8,0.8}{72.7}\\ 
    \bottomrule
    \end{tabular}}
    \caption{Training Time Comparison. The \colorbox[rgb]{0.957,0.8,0.8}{best} and \colorbox[rgb]{0.812,0.886,0.953}{runner-up} results are highlighted in red and blue.} %
    \label{tab:running_citeseer_ogbn}
    \vspace{-1mm}
\end{table}

\subsection{Ablation Study}
\textbf{Regularization term.} 
The results of the ablation study are presented in 
Table~\ref{tab:abl}. We evaluate the performance based on AUROC with the higher the better, denoted by ($\uparrow$). Compared to the basic EPN framework, incorporating ICE improves model performance on nearly all datasets, with the exception of CoauthorPhysics, where the improvement is less than one percent. 
Adding the PCL regularization term to EPN-ICE further enhances performance across all datasets, again except for CoauthorPhysics.
Notably, on OGBN-Arxiv, EPN-ICE increases AUROC by 9.61\% compared to EPN, while EPN-ICE-PCL boosts AUROC by an additional 7.27\% over EPN-ICE. Similarly, for AmazonComputers and CoauthorCS, the regularized EPN models outperform the baseline EPN by more than 10\%.

\begin{table*}[hbtp!]
    \centering
    \normalsize
    \setlength{\tabcolsep}{4pt}
    \resizebox{0.9\linewidth}{!}{%
    \begin{tabular}{l|cccccccc}
    \toprule
    \multirow{2}{*}{\textbf{Model}} & \multirow{2}{*}{\textbf{CoraML}} & \multirow{2}{*}{\textbf{CiteSeer}} & \multirow{2}{*}{\textbf{PubMed}} & \multicolumn{1}{c}{\textbf{OGBN}} & \multicolumn{1}{c}{\textbf{Amazon}} & \multicolumn{1}{c}{\textbf{Amazon}} & \multicolumn{1}{c}{\textbf{Coauthor}} & \multicolumn{1}{c}{\textbf{Coauthor}} \\
     &  &  &  & \multicolumn{1}{c}{\textbf{Arxiv}} & \textbf{Photos} & \textbf{Computers} & \textbf{CS} & \textbf{Physics} \\
    \midrule
    EPN & 88.06 ± 2.62 & 84.02 ± 4.22 & 66.38 ± 0.77 & 67.12 ± 1.30 & 84.68 ± 4.18 & 73.50 ± 3.59 & 82.30 ± 7.91 & \colorbox[rgb]{0.957,0.8,0.8}{94.10 ± 2.68} \\
    \quad + ICE & 89.54 ± 2.06 & 88.23 ± 2.79 & 67.38 ± 3.85 & 76.73 ± 0.68 & 84.48 ± 6.64 & 80.45 ± 4.57 & 92.85 ± 1.86 & 93.59 ± 2.41 \\
    \quad + ICE + PCL (EPN-reg) & \colorbox[rgb]{0.957,0.8,0.8}{89.97 ± 2.48} & \colorbox[rgb]{0.957,0.8,0.8}{89.74 ± 3.86} & \colorbox[rgb]{0.957,0.8,0.8}{68.39 ± 4.41} & \colorbox[rgb]{0.957,0.8,0.8}{83.99 ± 0.33} & \colorbox[rgb]{0.957,0.8,0.8}{86.69 ± 3.94} & \colorbox[rgb]{0.957,0.8,0.8}{83.26 ± 6.06} &\colorbox[rgb]{0.957,0.8,0.8}{95.09 ± 1.37} & 93.31 ± 3.13 \\
    \bottomrule
    \end{tabular}
    }
    \caption{Ablation study on OOD detection: OOD-AUROC ($\uparrow$)  for  EPN on the OS-1  setting. }
    \label{tab:abl}
\end{table*}

\begin{table*}[hbt!]
    \centering
    \normalsize
    \setlength{\tabcolsep}{4pt}
    \resizebox{0.9\linewidth}{!}{%
    \begin{tabular}{l|l|ccccccc}
    \toprule
    \multirow{2}{*}{\textbf{Variant}} &\multirow{2}{*}{\textbf{Model}} & \multirow{2}{*}{\textbf{CoraML}} & \multirow{2}{*}{\textbf{CiteSeer}} & \multirow{2}{*}{\textbf{PubMed}}  & \multicolumn{1}{c}{\textbf{Amazon}} & \multicolumn{1}{c}{\textbf{Amazon}} & \multicolumn{1}{c}{\textbf{Coauthor}} & \multicolumn{1}{c}{\textbf{Coauthor}} \\
     &  & & & & \textbf{Photos} & \textbf{Computers} & \textbf{CS} & \textbf{Physics}   \\
    \midrule
\multirow{2}{*}{\textbf{Backbone}} & GCN & 89.97 ± 2.48 & 88.23 ± 2.79 & 67.38 ± 3.85 & 86.49 ± 5.40 & 83.26 ± 6.06 & \colorbox[rgb]{0.957,0.8,0.8}{95.09 ± 1.37} & 93.59 ± 2.41 \\
 & GAT &\colorbox[rgb]{0.957,0.8,0.8}{91.21 ± 0.76} &\colorbox[rgb]{0.957,0.8,0.8}{90.96 ± 1.99} & \colorbox[rgb]{0.957,0.8,0.8}{68.67 ± 2.56} & \colorbox[rgb]{0.957,0.8,0.8}{93.47 ± 1.37} & \colorbox[rgb]{0.957,0.8,0.8}{90.31 ± 2.64} & 94.21 ± 1.75 & \colorbox[rgb]{0.957,0.8,0.8}{95.57 ± 0.93} \\ 
\hline
\multirow{2}{*}{\textbf{Feature}} & Second-to-Last & 87.27 ± 5.34 & 85.06 ± 2.79 & 65.83 ± 4.62 & \colorbox[rgb]{0.957,0.8,0.8}{84.54 ± 5.55} & \colorbox[rgb]{0.957,0.8,0.8}{84.62 ± 2.60} & 90.34 ± 5.41 & \colorbox[rgb]{0.957,0.8,0.8}{94.97 ± 0.92} \\
 & Last & \colorbox[rgb]{0.957,0.8,0.8}{89.54 ± 2.06} & \colorbox[rgb]{0.957,0.8,0.8}{88.23 ± 2.79} & \colorbox[rgb]{0.957,0.8,0.8}{67.38 ± 3.85} & 84.48 ± 6.64 & 80.45 ± 4.57 & \colorbox[rgb]{0.957,0.8,0.8}{92.85 ± 1.86} & 93.59 ± 2.41 \\ 
\hline
\multirow{2}{*}{\textbf{Activation}} & SoftPlus &\colorbox[rgb]{0.957,0.8,0.8}{90.57 ± 1.38} & \colorbox[rgb]{0.957,0.8,0.8}{88.24 ± 2.24} &\colorbox[rgb]{0.957,0.8,0.8}{67.77 ± 3.91} & \colorbox[rgb]{0.957,0.8,0.8}{88.57 ± 4.94} & 77.21 ± 7.10 & 90.14 ± 3.31 &\colorbox[rgb]{0.957,0.8,0.8}{94.98 ± 2.28} \\
 & Exponential& 89.54 ± 2.06 & 88.23 ± 2.79 & 67.38 ± 3.85 & 84.48 ± 6.64 & \colorbox[rgb]{0.957,0.8,0.8}{80.45 ± 4.57} &\colorbox[rgb]{0.957,0.8,0.8}{92.85 ± 1.86} & 93.59 ± 2.41 \\ 
\bottomrule
\end{tabular}
}
\caption{Discussions on the EPN architecture: OOD-AUROC scores ($\uparrow$) for GCN/GAT as the pre-trained model, using either the last or second-to-last hidden layer from the pre-trained model as the latent representation, and applying exponential or SoftPlus as the activation functions for the EPN's final layer.}
\label{tab:discussion}
\end{table*}

\textbf{Backbone architecture.} To investigate the effect of the GNN backbone, we conduct experiments on Graph attention networks (GATs) by~\cite{velivckovic2017graph}, including the baseline models and our proposed EPN/EPN-reg. Table~\ref{tab:ood_gat_rank} shows the average performance rank, 
while Tables~\ref{tab:ood_gat_1}--\ref{tab:ood_gat_4} present the complete results. The observations are similar when using GCN as the backbone. Specifically, EPN-reg achieves the best overall performance, followed by \quad + vacuity-prop and VGNN-gnnsafe. Regularization terms consistently enhance EPN's performance.

Table~\ref{tab:discussion} compares GCN and GAT as the pre-trained models for EPN, where the first two rows provide different latent representations as input features for EPN. We observe that GAT as the pre-trained model outperforms GCN on 6 out of 7 datasets, with an improvement up to 7\%. This highlights the significant influence of representation quality on EPN's performance.

\textbf{Feature layer.} We extract latent representations from both the last and second-to-last layers of the pre-trained model, as shown in the middle two rows of Table~\ref{tab:discussion} (additional results in Table~\ref{tab:layer}). The performance difference between these two approaches is within 4\%, with features from the last layer outperforming those from the second-to-last on 4 out of 7 datasets. This suggests that the most effective layer for feature extraction varies across datasets.

\textbf{Activation function.} 
We compare SoftPlus and exponential functions as the activation function for the output layer, with the results shown in the last two rows of Table~\ref{tab:discussion} (full results in Table~\ref{tab:act}). The performance difference between the two designs is within 4\%, indicating comparable effectiveness.
However, the exponential function can cause instability, leading to exploding evidence predictions in some runs, such as in the PubMed dataset.

\section{Conclusion}
We propose the Evidential Probe Network (EPN) for node-level uncertainty quantification in Graph Neural Networks (GNNs). EPN offers a flexible, computationally efficient alternative to existing evidential methods by leveraging pre-trained classification models without requiring retraining. Grounded in subjective logic opinion, it enhances interpretability while maintaining accuracy. Additionally, we introduce evidence-based regularization techniques to improve epistemic uncertainty estimation.
We theoretically show that training EPN solely with UCE loss fails to ensure proper epistemic uncertainty ordering, and our proposed ICE regularization aligns the latent representation of EPN with the optimal class probability obtained from a pre-trained model to enforce intra-class consistency, thereby preserving the correct epistemic ordering.
Empirical results demonstrate that EPN remains competitive with state-of-the-art methods while significantly reducing computational costs. Moving forward, we plan to extend EPN to broader deep learning architectures, including image and text classification, while further refining its theoretical foundations.

\section*{Acknowledgments}
This work is partially supported by the
National Science Foundation (NSF) under Grant No.~2414705, 2220574, 2107449, and 1954376. 

\newpage
\bibliography{refer}
\bibliographystyle{apalike}
\section*{Checklist}

 \begin{enumerate}

 \item For all models and algorithms presented, check if you include:
 \begin{enumerate}
   \item A clear description of the mathematical setting, assumptions, algorithm, and/or model. [\textbf{Yes}]
   \item An analysis of the properties and complexity (time, space, sample size) of any algorithm. [\textbf{Yes}]
   \item (Optional) Anonymized source code, with specification of all dependencies, including external libraries. [\textbf{Yes}]
 \end{enumerate}

 \item For any theoretical claim, check if you include:
 \begin{enumerate}
   \item Statements of the full set of assumptions of all theoretical results. [\textbf{Yes}]
   \item Complete proofs of all theoretical results. [\textbf{Yes}]
   \item Clear explanations of any assumptions. [\textbf{Yes}]     
 \end{enumerate}

 \item For all figures and tables that present empirical results, check if you include:
 \begin{enumerate}
   \item The code, data, and instructions needed to reproduce the main experimental results (either in the supplemental material or as a URL). [\textbf{Yes}]
   \item All the training details (e.g., data splits, hyperparameters, how they were chosen). [\textbf{Yes}]
         \item A clear definition of the specific measure or statistics and error bars (e.g., with respect to the random seed after running experiments multiple times). [\textbf{Yes}]
         \item A description of the computing infrastructure used. (e.g., type of GPUs, internal cluster, or cloud provider). [\textbf{Yes}]
 \end{enumerate}

 \item If you are using existing assets (e.g., code, data, models) or curating/releasing new assets, check if you include:
 \begin{enumerate}
   \item Citations of the creator If your work uses existing assets. [\textbf{Yes}]
   \item The license information of the assets, if applicable. [\textbf{Not Applicable}]
   \item New assets either in the supplemental material or as a URL, if applicable. [\textbf{Not Applicable}]
   \item Information about consent from data providers/curators. [\textbf{Not Applicable}]
   \item Discussion of sensible content if applicable, e.g., personally identifiable information or offensive content. [\textbf{Not Applicable}]
 \end{enumerate}

 \item If you used crowdsourcing or conducted research with human subjects, check if you include:
 \begin{enumerate}
   \item The full text of instructions given to participants and screenshots. [\textbf{Not Applicable}]
   \item Descriptions of potential participant risks, with links to Institutional Review Board (IRB) approvals if applicable. [\textbf{Not Applicable}]
   \item The estimated hourly wage paid to participants and the total amount spent on participant compensation. [\textbf{Not Applicable}]
 \end{enumerate}

 \end{enumerate}
 
\clearpage
\onecolumn

\aistatstitle{Evidential Uncertainty Probes for Graph Neural Networks \\
Supplementary Materials}

\appendix

\section{List of Symbols}\label{sec:symbols}
The following table contains a list of symbols that are frequently used in the main paper as well as in the subsequent supplementary materials. 
\small
\begin{table}[bhtp]
\centering
		\resizebox{0.65\textwidth}{!}{
			\begin{tabular}{l|p{0.6\textwidth}}
				\toprule
				\multicolumn{2}{c}{\textbf{Data Distribution}} \\
				\midrule
				$C$& Number of classes. \\
				$\mathcal{G}$& Abstract representation of an attributed graph. \\			
				$\mathbb{V}, |\mathbb{V}|=N$& Set of nodes in the graph. \\
				$\mathbb{L}, \mathbb{L}\in \mathbb{V}$& Set of labeled nodes in the graph. \\			
				$\mathbf{A} \in \mathbb{R}^{N \times N}$& Adjacency matrix of the graph. \\
				$\mathbf{X} \in \mathbb{R}^{N \times F} $& Node feature matrix of the graph. \\			
			    $\mathbf{y}\in \mathbb{R}^C$& One-hot encoded class label vector. The unbold $y$ represents a scalar class label. \\
				$\mathbf{y}_{\mathbb{L}} = \{\mathbf{y}^i | i \in \mathbb{L}\}$ & Ground truth labels for the labeled set. \\
                $Y \in \{-1, +1, 0\}$& Class label space. \\			 			
				$\mathbf{z} \in \mathbb{R}^d$& Latent node feature vector. For OOD nodes, it is denoted as $\mathbf{z}_0$, while for ID nodes, it is $\mathbf{z}_{\text{ID}}$. Positive ID class nodes are represented as $\mathbf{z}_{-1}$ and negative ID class nodes as $\mathbf{z}_{+1}$. \\
                $\boldsymbol{\mu} \in \mathbb{R}^d$& Mean vector of a Gaussian distribution. \\
				$\boldsymbol{\Sigma} \in \mathbb{R}^{d \times d}$& Covariance matrix of a Gaussian distribution. \\
				\midrule
				\multicolumn{2}{c}{\textbf{Network Parameters}} \\
				\midrule			
				$\boldsymbol{\theta}_{\text{EGNN}} = \{\bar{\mathbf{w}}, \bar{b}\}$& Parameters of the ENN model. The optimal parameters derived from the upper bound of the UCE loss are denoted as $\{\bar{\mathbf{w}}^*, \bar{b}^*\}$. \\
				$\boldsymbol{\theta}_{\text{EPN}} = \{\mathbf{W}^{[1]},\mathbf{w}^{[2]}, \mathbf{b}^{[1]}, b^{[2]} \}$& Parameters of the EPN model. \\
				$\tilde{\boldsymbol{\theta}}_n$ & Series of EPN parameters indexed by $n$ used in Theorem ~\ref{thm:EPN-UCE}. \\
                    $\vec{\boldsymbol{\theta}}_{\text{EPN}}$ &  Constrained EPN parameters used in Theorem ~\ref{thm:EPN-ICE}. \\
				$\boldsymbol{\theta}_{\text{GNN}} = \{\mathbf{w}_{C}, b_{C} \}$& Parameters of the EPN model. \\
				\midrule
				\multicolumn{2}{c}{\textbf{Distributions}} \\
				\midrule
				$\mathcal{N}(\boldsymbol{\mu}, \boldsymbol{\Sigma})$ & Gaussian distribution. \\
				$\texttt{Cat}(\mathbf{p})$ & Categorical distribution with parameter $\mathbf{p} \in \Delta_C$. \\
				$\texttt{Dir}(\boldsymbol{\alpha})$ & Dirichlet distribution with parameter $\boldsymbol{\alpha} \in \mathbb R_+^C$. \\
				\midrule
				\multicolumn{2}{c}{\textbf{Network Predictions}} \\
				\midrule
				$e_{\text{total}}^i \in \mathbb{R}_+$ & Total predicted evidence for node $i$. \\	
                $\mathbf{e}^i \in \mathbb{R}^C_+$ & Predicted evidence vector for node $i$. \\
                $\boldsymbol{\alpha}^i \in \mathbb{R}^C_+$ & Predicted concentration parameter vector for node $i$. \\
                $\boldsymbol{\alpha}^i_{\text{agg}}$ & Predicted concentration parameter vector for node $i$, specifically represent the one after uncertainty propagation.  \\
                $\mathbf{p}^i \in \Delta_C$ & Realized probability vector for node $i$, drawn from the predicted Dirichlet distribution. \\
				$\tilde{\mathbf{p}}^i$ & Predicted probability vector from a softmax classification model. \\
			    $\mathbf{q}^i$ & Representation predicted by the second-to-last layer of the EPN model. \\
				\hline
				\multicolumn{2}{c}{\textbf{Uncertainty Metrics}} \\
				\hline		
				$r \in (0,1)$ & Confidence score based on the predicted probability. \\
				$u_{\text{epi}}(\cdot)$ & Epistemic uncertainty function. \\
				$u_{\text{alea}}^i$ & Aleatoric uncertainty for node $i$. \\				
				$u_{\text{epis}}^i$ & Epistemic uncertainty for node $i$. \\	
                    $U^*_{\text{E}}(\cdot)$ & Epistemic uncertainty function for the optimal ENN model used in Theorem ~\ref{thm:1}. \\
				\midrule
				\multicolumn{2}{c}{\textbf{Hyperparameters}} \\
				\midrule				
				$e_{\text{total}}^{\text{id}}$ & Predefined total evidence value for ID samples. \\
				$e_{\text{total}}^{\text{ood}}$ & Predefined total evidence value for OOD samples. \\
                $\gamma^1, \gamma^2$ & Parameters for uncertainty propagation. \\
                $\lambda_1, \lambda_2$ & Loss function parameters for optimizing the EPN model. \\
				$\epsilon, \digamma, \sigma, t$ & Temporary parameters. \\
                \bottomrule
		\end{tabular}}
\end{table}

\section{Theoretical Analysis} \label{sect:proofs}

\textbf{Binary Classification with OOD Detection.}
We consider a balanced binary classification task together with OOD detection. A sample can belong to either the out-of-distribution (OOD) class (\(Y=0\)) or one of two in-distribution (ID) classes (\(Y=-1\) or \(+1\)). First, we determine whether a sample is OOD (\(Y=0\)) or ID (\(Y=-1\) or \(+1\)). If it is ID, we further classify it as either \(-1\) or \(+1\). We assume the label \(Y\) follows a categorical distribution with a uniform class prior
\[
Y \;\sim\; \mathtt{Cat}\Bigl(\tfrac{1}{3},\, \tfrac{1}{3},\, \tfrac{1}{3}\Bigr).
\]

Our theoretical analysis is built on a latent representation of feature vectors. Specifically, we can decompose the entire network into a representation-learning network (RLN) and a task head. For the node classification task, the RLN typically consists of graph convolutional layers and MLP layers, leading to node-level hidden representations $[\mathbf{z}^i]_{i \in \mathbb{V}}$ that embed the node features.   The task head, comprising the final MLP layers, is then applied to these node embeddings to produce their associated node-level probability vectors for classification task and node-level Dirichlet distribution for evidential-based models. For our analysis, we focus on the task head where we examine an asymptotic behavior of ENN and EPN networks. 

In practice, the training set consists only of ID samples with \(\mathbb{P}(Y=-1) = \mathbb{P}(Y=+1) = \tfrac{1}{2}\) without OOD samples, i.e., \(\mathbb{P}(Y=0) = 0\). Thus, the joint distribution of the training data is
\[
\mathbb{P}\bigl(\mathbf{z} \,\big|\,
Y \in \{-1, +1\}\bigr)
= \tfrac{1}{2}\,\mathbb{P}(\mathbf{z} \mid Y = -1)
+ \tfrac{1}{2}\,\mathbb{P}(\mathbf{z} \mid Y = +1).
\]
The testing set consists of both ID and OOD samples.

For the ease of deriving an analytical solution, we assume a generative model in which the conditional distributions of the two ID classes are Gaussian with means $\pm \bm\mu$ and identical covariance $\bm{\Sigma}$, while the conditional distribution of the OOD class has the zero mean and covariance $\bm{\Sigma}$, as detailed in Assumption \ref{assum:data}.

\begin{assumption}[Data Distribution Assumption]
\label{assum:data}
Let \(Y \in \{-1, 0, +1\}\) be the class label,  we assume each latent data vector \(\mathbf{z}\) follows a multivariate Gaussian distribution,
\[
\left\{
\begin{aligned}
&\mathbf{z} \,\big|\, Y = -1 \;\sim\; \mathcal{N}(-\bm{\mu},\, \bm{\Sigma}),\\
&\mathbf{z} \,\big|\, Y = +1 \;\sim\; \mathcal{N}(\bm{\mu},\, \bm{\Sigma}),\\
&\mathbf{z} \,\big|\, Y = 0 \;\sim\; \mathcal{N}(\mathbf{0},\, \bm{\Sigma}),
\end{aligned}
\right.
\]
where \(\bm{\mu} \in \mathbb{R}^d\) and \(\bm{\Sigma} \in \mathbb{R}^{d \times d}\) are the mean vector and the covariance matrix of the Gaussian distribution. We assume \(\bm{\Sigma}\) is fixed and finite, while varying \(\bm{\mu}\) to analyze the asymptotic behavior.

For simple notation, we use \(\mathbf{z}_0\sim\mathcal{N}(\mathbf{0},\bm{\Sigma})\) for the OOD class \((Y=0)\),  \(\mathbf{z}_{+1}\sim\mathcal{N}(\bm{\mu},\bm{\Sigma})\) for the (positive) ID class \((Y=+1)\) and \(\mathbf{z}_{-1}\sim\mathcal{N}(-\bm{\mu},\bm{\Sigma})\) for the (negative) ID class \((Y=-1)\).
\end{assumption}

\subsection{Theorem 1}
\begin{definition}[Single-layer ENN] \label{assum:enn_net}

We consider a single-layer perception network as the architecture of ENN, defined as follows,
\begin{align} \label{eq:evidence_one_neuron2}
f_{\text{EGNN}}(\mathbf{z}; \bm{\theta}_{\text{EGNN}}) 
= \begin{bmatrix} 
e_{\text{E}, -y}(\mathbf{z}; \bm{\theta}_{\text{EGNN}})\\
e_{\text{E}, y}(\mathbf{z}; \bm{\theta}_{\text{EGNN}})
\end{bmatrix}
= \begin{bmatrix} 
\exp(-\bar{\mathbf{w}}^\top \mathbf{z} - \bar{b}) \\ 
\exp(\bar{\mathbf{w}}^\top \mathbf{z} + \bar{b}) 
\end{bmatrix},
\end{align}
where $\bm{\theta}_{\text{EGNN}} = \{\bar{\mathbf{w}}, \bar{b}\}$ denotes the model parameters, $e_{\text{E}, y}$ represents the predicted evidence for the ground truth class $y$, and $e_{\text{E}, -y}$ represents the predicted evidence for the opposite class $-y$. 
\end{definition}

The concentration parameters are then calculated as 
\begin{align} \label{eq:alpha_one_neuron}
\boldsymbol{\alpha} 
= \begin{bmatrix} 
\alpha_{\text{E}, -y}\\
\alpha_{\text{E}, y}
\end{bmatrix}
= \begin{bmatrix} 
\exp(-\bar{\mathbf{w}}^\top \mathbf{z} - \bar{b}) + 1 \\ 
\exp(\bar{\mathbf{w}}^\top \mathbf{z} + \bar{b}) + 1
\end{bmatrix}.
\end{align}

\begin{definition}[Upper Bound of the UCE Loss]
\label{def:upperbound}
Under the data assumption specified in Assumption~\ref{assum:data}, an ENN with the structural assumption outlined in Definition ~\ref{assum:enn_net}  has the following upper bound for the uncertainty cross-entropy (UCE) loss:
\begin{equation}\label{UCE_upper}
    \overline{\ell_{\text{EGNN,UCE}}}
    \;=\;
    \frac{2}{\,e_{\text{E}, y}(\mathbf{z}; \bm{\theta}_{\text{EGNN}})\,},
\end{equation}
where \(\mathbf{z} \in \mathbb{R}^d\) is the latent feature vector and  \(e_{\text{E}, y}(\mathbf{z}; \bm{\theta}_{\text{EGNN}}) \in \mathbb{R}^+\) is the predicted evidence for the ground truth class \(y\), as defined in~(\ref{eq:evidence_one_neuron2}).
\end{definition}
Please refer to \citet[Proposition 1]{yu2023uncertainty} for the derivation of such upper bound in (\ref{UCE_upper}). We provide a closed-form formula for the predicted epistemic uncertainty under Assumptions \ref{assum:data} and Definition \ref{assum:enn_net} when minimizing this upper bound $\overline{\ell_{\text{EGNN,UCE}}}$ in Proposition \ref{prop:opt_enn}.

\begin{proposition}
[Optimal Epistemic Uncertainty]
\label{prop:opt_enn}
Suppose an optimal ENN satisfying Definition \ref{assum:enn_net} is trained on the data with its distribution specified in Assumption~\ref{assum:data} while minimizing $\overline{\ell_{\text{EGNN,UCE}}}$ defined in Definition~\ref{def:upperbound}.
Its predicted epistemic uncertainty for a given latent representation \(\mathbf{z} \in \mathbb{R}^d\) can be expressed as follows,
\begin{equation}
    U_{\text{E}}^*(\mathbf{z})
    \;=\;
    \frac{1}{1 +  \,\cosh\!\Bigl(\bm{\mu}^\top \bm{\Sigma}^{-1}\mathbf{z}\Bigr)}.
\end{equation}
\end{proposition}

\begin{proof}
Under the same assumptions, \citet[Theorem 1]{yu2023uncertainty} demonstrated that minimizing $\overline{\ell_{\text{EGNN,UCE}}}$  yields the optimal parameters:
\begin{equation} \label{eq: optim_enn}
    \bar{\mathbf{w}}^* = \bm{\Sigma}^{-1} \bm{\mu}, \quad \bar{b}^* = 0.
\end{equation}
Plugging these parameters into (\ref{eq:alpha_one_neuron}) leads to the formula for the predicted concentration parameters of the Dirichlet distribution
\begin{equation}\label{eq:alphas}
    \begin{aligned}
        \alpha_{\text{E}, -y}(\mathbf{z}; \bm{\theta}_{\text{E}}) &= \exp(-\bm{\mu}^\top \bm{\Sigma}^{-1} \mathbf{z}) + 1, \\
        \alpha_{\text{E}, y}(\mathbf{z}; \bm{\theta}_{\text{E}}) &= \exp(\bm{\mu}^\top \bm{\Sigma}^{-1} \mathbf{z}) + 1.
    \end{aligned}
\end{equation}

Following \citet{sensoy2018evidential}, the total epistemic uncertainty for the Dirichlet distribution can be computed by
\begin{equation}
\begin{aligned}
    U_{\text{E}}^*(\mathbf{z}) 
    =& \frac{2}{\alpha_{\text{E}, -y}(\mathbf{z}; \bm{\theta}_{\text{E}}) + \alpha_{\text{E}, y}(\mathbf{z}; \bm{\theta}_{\text{E}})} \\
= &\frac{2}{\exp(-\bm{\mu}^\top \bm{\Sigma}^{-1} \mathbf{z}) + 1 + \exp(\bm{\mu}^\top \bm{\Sigma}^{-1} \mathbf{z}) + 1}\\
=& \frac{1}{1 + \cosh(\bm{\mu}^\top \bm{\Sigma}^{-1} \mathbf{z})},
\end{aligned}
\end{equation}
where we use the formula (\ref{eq:alphas}) and the definition that \(2\cosh(x)=\exp(x) + \exp(-x)\). This completes the proof.

\end{proof}

We are ready to establish the first main result: given a sufficiently large separation between the means of the two ID classes, with the OOD samples lying in between, the optimal ENN model can reliably distinguish between ID and OOD data. When the separation is sufficiently large, the probability of the ENN model detecting OOD samples can be arbitrarily close to 1, which can be stated by the precision definition of the limit, as demonstrated in Theorem \ref{thm:1}.

\setcounter{theorem}{0}
\begin{theoremref}{\ref{thm:1}}
For any \(\epsilon > 0\), there exists a positive constant \(\digamma > 0\) such that, for any data distribution satisfying Assumption~\ref{assum:data} with \(\|\bm{\mu}\|_2 > \digamma\), the probability that the epistemic uncertainty obtained by an optimal single-layer ENN based on an upper bound of UCE loss in Definition~\ref{def:upperbound}, correctly distinguishes ID and OOD samples is greater than \(1 - \epsilon\). 
\end{theoremref}

\begin{proof}
For every $\epsilon,$ we choose \(N = \sqrt{\dfrac{8 \lambda_{\text{max}}}{\epsilon}}\), where $\lambda_{\max}$ denotes the maximum singular value of $\Sigma.$ For any $\bm\mu$ with $\|\bm \mu\|_2>N,$ we start by showing that
\begin{equation}\label{ineq:prob}
    \mathbb{P}\bigr(U_{\text{E}}^*(\mathbf{z}_0)> U_{\text{E}}^*(\mathbf{z_1})\bigr)>1-\epsilon, \, \forall \quad  \mathbf{z}_0\sim \mathcal{N}(\mathbf{0}, \boldsymbol{\Sigma}), \mathbf{z}_{1}\sim \mathcal{N}(\boldsymbol{\mu}, \boldsymbol{\Sigma}),
\end{equation}
which implies that a latent vector belonging to OOD $(\mathbf z_0)$ is more likely to classify as OOD than the one belonging to ID $(\mathbf z_1)$. 

Recall in Proposition~\ref{prop:opt_enn} that 
\begin{equation}
    U_{\text{E}}^*(\mathbf{z})
\;=\;
\frac{1}{\,\cosh\!\bigl(\bm{\mu}^\top \bm{\Sigma}^{-1}\,\mathbf{z}\bigr)+1\,}, \quad \forall \mathbf z\in\mathbb R^d.
\end{equation}
Since the cosh function's shape resembles a parabola, the following two inequalities are equivalent,
\begin{eqnarray}\label{ineq:equiv}
    U_{\text{E}}^*(\mathbf{z}_0)> U_{\text{E}}^*(\mathbf{z_1} )	\Longleftrightarrow  |\boldsymbol{\mu}^\top \boldsymbol{\Sigma}^{-1} \mathbf{z}_0 |<|\boldsymbol{\mu}^\top \boldsymbol{\Sigma}^{-1} \mathbf{z}_1| .
\end{eqnarray}

For any $\mathbf{z}_0\sim \mathcal{N}(\mathbf{0}, \boldsymbol{\Sigma})$ and $ \mathbf{z}_{1}\sim \mathcal{N}(\boldsymbol{\mu}, \boldsymbol{\Sigma}),$ we have  $\bm{\mu}^\top \bm{\Sigma}^{-1} \mathbf{z}_0 \sim \mathcal{N}(0, \sigma^2)$ and $\bm{\mu}^\top \bm{\Sigma}^{-1} \mathbf{z}_1 \sim \mathcal{N}(\sigma^2, \sigma^2)$ with \(\sigma^2 = \bm{\mu}^\top \bm{\Sigma}^{-1} \bm{\mu}\). It is straightforward that with $\lambda_{\max}$ as the maximum eigen value of the matrix $\boldsymbol{\Sigma}$,
\begin{equation}\label{ineq:sigma-eps}
    \sigma^2 = \bm{\mu}^\top \bm{\Sigma}^{-1} \bm{\mu}\geq \frac 1 {\lambda_{\max}} \|\bm\mu\|_2^2\geq \frac 8 {\epsilon}.
\end{equation}

It follows from the Chebyshev's inequality that	for any $t>0,$ one has the probability upper bounds:
\begin{equation}
\left\{
\begin{aligned}
	&\mathbb{P}\left(|\boldsymbol{\mu}^\top \boldsymbol{\Sigma}^{-1} \mathbf{z}_0 | \geq t\sigma\right) \leq \frac{1}{t^2} \\
	&\mathbb{P}\left(|\boldsymbol{\mu}^\top \boldsymbol{\Sigma}^{-1} \mathbf{z}_1 - \sigma^2| \geq t\sigma\right) \leq \frac{1}{t^2},
\end{aligned}
\right.
\end{equation}
which can be equivalently expressed by
	\begin{equation}\label{ineq:cheb}
\left\{
\begin{aligned}
	&\mathbb{P}\left(-t\sigma \leq \boldsymbol{\mu}^\top \boldsymbol{\Sigma}^{-1} \mathbf{z}_0  \leq t\sigma\right) \geq 1 - \frac{1}{t^2} \\
	&\mathbb{P}\left(-t\sigma + \sigma^2 \leq \boldsymbol{\mu}^\top \boldsymbol{\Sigma}^{-1} \mathbf{z}_1  \leq t\sigma + \sigma^2\right) \geq 1 - \frac{1}{t^2}.
\end{aligned}
\right.
	\end{equation}

For simple notation, we define two events:
\begin{equation}
\left\{
\begin{aligned}
	&\text{Event A}: -t\sigma \leq \boldsymbol{\mu}^\top \boldsymbol{\Sigma}^{-1} \mathbf{z}_0  \leq t\sigma \\
	&\text{Event B}: -t\sigma + \sigma^2 \leq \boldsymbol{\mu}^\top \boldsymbol{\Sigma}^{-1} \mathbf{z}_1 \leq t\sigma + \sigma^2.
\end{aligned}
\right.
\end{equation}
Then the inequalities in (\ref{ineq:cheb}) become
\begin{equation}
\mathbb{P}(\text{A}) \geq 1 - \frac{1}{t^2}, \quad \mathbb{P}(\text{B}) \geq 1 - \frac{1}{t^2}.
\end{equation}
Using the above lower bounds together with the inclusion-exclusion principle, we can estimate the probability of both events occurring, that is,
\begin{align}
    \mathbb{P}(\text{A} \cap \text{B}) =& \mathbb{P}(\text{A}) + \mathbb{P}(\text{B}) - \mathbb{P}(\text{A} \cup \text{B}) \\
    \geq &\mathbb{P}(\text{A}) + \mathbb{P}(\text{B}) - 1\\
    \geq &\left(1 - \frac{1}{t^2}\right) + \left(1 - \frac{1}{t^2}\right) - 1 = 1 - \frac{2}{t^2}.
\end{align}

We choose \(t = \frac{\sigma}{2}\), or equivalently \(t\sigma= -t\sigma + \sigma^2 ,\) and hence we have
\begin{equation}
-t\sigma \leq \boldsymbol{\mu}^\top \boldsymbol{\Sigma}^{-1} \mathbf{z}_0 \leq t\sigma = -t\sigma + \sigma^2 \leq \boldsymbol{\mu}^\top \boldsymbol{\Sigma}^{-1} \mathbf{z}_1 .
\end{equation}
Using the inequality in (\ref{ineq:sigma-eps}), we get 
\begin{equation}
\mathbb{P}\left(|\boldsymbol{\mu}^\top \boldsymbol{\Sigma}^{-1} \mathbf{z}_0 |<|\boldsymbol{\mu}^\top \boldsymbol{\Sigma}^{-1} \mathbf{z}_1 |\right) \geq \mathbb{P}(\text{A} \cap \text{B}) \geq 1 - \frac{2}{t^2} = 1-\frac{8}{\sigma^2} \geq 1-\epsilon.
\end{equation}

Due to the equivalent relationship in (\ref{ineq:equiv}), we arrive at the desired inequality (\ref{ineq:prob}).
Thanks to the symmetry of the epistemic uncertainty, we can obtain the following result in a similar manner
\begin{equation}
    \mathbb{P}\bigr(U_{\text{E}}^*(\mathbf{z}_0)> U_{\text{E}}^*(\mathbf{z_{-1}})\bigr)>1-\epsilon, \forall \mathbf{z}_0\sim \mathcal{N}(\textbf{0}, 
\boldsymbol{\Sigma}), \mathbf{z}_{-1}\sim \mathcal{N}(-\bm{\mu}, \boldsymbol{\Sigma}).
\end{equation}

We then analyze the probability of correctly distinguishing between ID and OOD samples. Successful classification of ID and OOD requires that ID samples exhibit lower epistemic uncertainty compared to OOD samples. Given that OOD samples follow \(\mathbf{z}_0 \sim \mathcal{N}(\mathbf{0}, \bm{\Sigma})\) and ID samples follow \(\mathbf{z}_{+1} \sim \mathcal{N}(\bm{\mu}, \bm{\Sigma})\), \(\mathbf{z}_{-1} \sim \mathcal{N}(-\bm{\mu}, \bm{\Sigma})\) with equal probability, we have:

\begin{equation}
  \mathbb{P}\big(U_E^*(\mathbf{z}_{0}) > U_E^*(\mathbf{z}_{\text{ID}})\big).
\end{equation}

If follows from Assumption \ref{assum:data} that \(\mathbf{z}_{+1} \sim \mathcal{N}(\bm{\mu}, \bm{\Sigma})\), \(\mathbf{z}_{-1} \sim \mathcal{N}(-\bm{\mu}, \bm{\Sigma})\), and \(\mathbf{z}_0 \sim \mathcal{N}(\mathbf{0}, \bm{\Sigma})\). Using the \(\xi(\mathbf{z})\)  denotes the PDF  of $\mathbf{z}$ 
and expanding this probability, we have: 
\begin{equation}
\begin{aligned}
    &\mathbb{P}\big(U_E^*(\mathbf{z}_{0}) > U_E^*(\mathbf{z}_{\text{ID}})\big) \\
    &= \int_{U_E^*(\mathbf{z}_{0}) > U_E^*(\mathbf{z}_{\text{ID}})} \xi(\mathbf{z}_0 )\xi(\mathbf{z}_{\text{ID}}) \, d\mathbf{z}_0 d\mathbf{z}_{\text{ID}} \\
    &= \int_{U_E^*(\mathbf{z}_0) > U_E^*(\mathbf{z}_{\text{ID}})} \xi(\mathbf{z}_0) \big(\xi(\mathbf{z}_{\text{ID}} = \mathbf{z}_{-1})\mathbb{P}(y = -1) + \xi(\mathbf{z}_{\text{ID}} = \mathbf{z}_{+1})\mathbb{P}(y = +1)\big) \, d\mathbf{z}_0 d\mathbf{z}_{\text{ID}}.
\end{aligned}
\end{equation}

Next, splitting the integrals based on the ID components:
\begin{equation}
\begin{aligned}
    &\mathbb{P}\big(U_E^*(\mathbf{z}_{0}) > U_E^*(\mathbf{z}_{\text{ID}})\big) \\
    &= \int_{U_E^*(\mathbf{z}_0) > U_E^*(\mathbf{z}_{-1})} \xi(\mathbf{z}_0)\xi(\mathbf{z}_{-1})\mathbb{P}(y = -1) \, d\mathbf{z}_0 d\mathbf{z}_{-1} \\
    &\quad + \int_{U_E^*(\mathbf{z}_0) > U_E^*(\mathbf{z}_{+1})} \xi(\mathbf{z}_0)\xi(\mathbf{z}_{+1})\mathbb{P}(y = +1) \, d\mathbf{z}_0 d\mathbf{z}_{+1}.
\end{aligned}
\end{equation}

Factoring out the prior probabilities: 
\begin{equation}
\begin{aligned}
    &\mathbb{P}\big(U_E^*(\mathbf{z}_{\text{OOD}}) > U_E^*(\mathbf{z}_{\text{ID}})\big) \\
    &= \mathbb{P}(y = -1) \int_{U_E^*(\mathbf{z}_0) > U_E^*(\mathbf{z}_{-1})} \xi(\mathbf{z}_0)\xi(\mathbf{z}_{-1}) \, d\mathbf{z}_0 d\mathbf{z}_{-1} \\
    &\quad + \mathbb{P}(y = +1) \int_{U_E^*(\mathbf{z}_0) > U_E^*(\mathbf{z}_{+1})} \xi(\mathbf{z}_0)\xi(\mathbf{z}_{+1}) \, d\mathbf{z}_0 d\mathbf{z}_{+1}.
\end{aligned}
\end{equation}

Using the definition of probabilities:
\begin{equation}
\begin{aligned}
   & \mathbb{P}\big(U_E^*(\mathbf{z}_{\text{OOD}}) > U_E^*(\mathbf{z}_{\text{ID}})\big)\\
& = \mathbb{P}(y = -1)\mathbb{P}\big(U_E^*(\mathbf{z}_0) > U_E^*(\mathbf{z}_{-1})\big) 
+ \mathbb{P}(y = +1)\mathbb{P}\big(U_E^*(\mathbf{z}_0) > U_E^*(\mathbf{z}_{+1})\big)\\
& > 1 - \epsilon 
\end{aligned}
\end{equation}

\end{proof}

\subsection{Theorem 2}
The proposed Evidential Probe Network (EPN) leverages a lightweight two-layer perceptron to estimate total evidence, which is used to quantify uncertainty. Given the latent representation of an input sample, \(\mathbf{z} \in \mathbb{R}^d\) and its corresponding one-hot encoded class label, $\mathbf{y}\in \{0 , 1\}^C$, the EPN, parameterized by $\boldsymbol{\theta}_{\text{EPN}} = \{\mathbf{W}^{[1]}, \mathbf{w}^{[2]}, {\bf b}^{[1]}, b^{[2]}\}$, produces a total evidence value $e_{\text{total}}$. The EPN network is defined as:
\begin{eqnarray}\label{eq: epn_arch}
 e_{\text{total}} = f_{\text{EPN}}(\mathbf{z}; \boldsymbol{\theta}_{\text{EPN}})
&=& 
\text{ReLU} \Bigl(\mathbf{w}^{[2]\top}\,\exp
\bigl(\mathbf{W}^{[1]} \,\mathbf{z}
+ \mathbf{b}^{[1]}\bigr)
+ b^{[2]}\Bigr),
\end{eqnarray}
where $\mathbf{W}^{[1]} \in \mathbb{R}^{C\times d}$, $\mathbf{w}^{[2]} \in \mathbb{R}^C$, $\mathbf{b}^{[1]} \in \mathbb{R}^C$, and  $b^{[2]}\in \mathbb{R}$. 

A pre-trained GNN classifier provides a function $f_{\text{GNN}}(\mathbf{z}; \boldsymbol{\theta}_{\text{GNN}})$, which outputs a class probability vector: \(\tilde{\mathbf{p}}({\bf z}) \in \Delta_C\), where $\Delta_C$ denotes the probability simplex of dimension $C$. 
This probability vector is used as an approximation of the expected class probability, i.e., $\mathbb{E}[y | {\bf z}]$. The Dirichlet distribution over class probabilities is parameterized as follows:
\begin{equation}
\begin{aligned}\label{eqn: p-alpha-e}
& \boldsymbol{\alpha}(\mathbf{z}; \boldsymbol{\theta}_{\text{EPN}}) = \bigl(f_{\text{EPN}}(\mathbf{z}; \boldsymbol{\theta}_{\text{EPN}}) + C\bigr)\,\tilde{\mathbf{p}}({\bf z}), \\
& \mathbf{e}(\mathbf{z}; \boldsymbol{\theta}_{\text{EPN}}) = \boldsymbol{\alpha}(\mathbf{z}; \boldsymbol{\theta}_{\text{EPN}}) - \mathbf{1}.
\end{aligned}
\end{equation}
This formulation ensures that the predicted class probabilities follow a Dirichlet distribution: ${\bf p} \sim \text{Dir}({\bf p} | \boldsymbol{\alpha}(\mathbf{z}; \boldsymbol{\theta}_{\text{EPN}})) 
$. The EPN-UCE loss is defined in (\ref{eq:uce_epn}). Specifically, the analytic solution is given by
\begin{equation}\label{eq: epn_uce_app}
\begin{aligned} 
     \ell_{\text{EPN, UCE}}(e_{\text{total}}, y ; \boldsymbol{\theta}_{\text{EPN}}) 
  = \psi\left(e_{\text{total}} + C\right) - \psi\Bigl(\left(e_{\text{total}} +C \right) \cdot \tilde{p}_y\Bigr), 
\end{aligned}
\end{equation}
where $\psi(\cdot)$ is the digamma function and $y$ is the scalar of ground truth label.

We focus on a binary classification task, i.e., $C=2$, and consider the following configuration of $\tilde{\boldsymbol{\theta}}_n=\{\tilde{\bf W}^{[1]}, \tilde{\bf w}^{[2]}, \tilde{\bf b}^{[1]}, \tilde{b}^{[2]}\}$: 
\begin{eqnarray}\label{eq:2-layer-network}
    \tilde{\bf W}^{[1]} = {\bf 0}\in\mathbb R^{2\times d},  \tilde{\bf w}^{[2]} = {\bf 1}\in\mathbb R^{2}, \tilde{\bf b}^{[1]} = n \cdot \begin{bmatrix}
1 \\
-1 
\end{bmatrix} \in\mathbb R^{2}, \ \tilde{b}^{[2]} = 0, 
\end{eqnarray}
with a scalar $n$, and hence the only varying parameter in $\tilde{\boldsymbol{\theta}}_n$ is the component $\tilde{\bf b}^{[1]}$ that depends on $n$.

Under this configuration (\ref{eq:2-layer-network}), we can derive a closed-form formula for the output of the EPN network, 
\begin{equation}
    e_{\text{total}} = f_{\text{EPN}}(\mathbf{z}; \tilde{\boldsymbol{\theta}}_n) = \exp(n) + \exp(-n) = 2\cdot \cosh(n) \quad \forall \, \mathbf{z} \in \mathbb{R}^d. 
\end{equation}
We establish in Theorem~\ref{thm:EPN-UCE} that such two-layer EPN, parameterized by $\tilde{\boldsymbol{\theta}}_n$, attains the infimum of the expected EPN-UCE loss, i.e.,  $ \ell_{\text{EPN,UCE}}(e_{\text{total}}, \mathbf{y}, \tilde{\mathbf{p}}; \tilde{\boldsymbol{\theta}}_n)$ asymptotically at infinity, i.e., as $n\rightarrow \infty$. We further show that this EPN network assigns a constant value of epistemic uncertainty to all samples and is therefore incapable to distinguish between ID and OOD nodes. The proof of Theorem~\ref{thm:EPN-UCE} relies on Lemma \ref{prop:epn_uce_decrease}. We will first prove Theorem~\ref{thm:EPN-UCE}, then Lemma \ref{prop:epn_uce_decrease}.

\begin{theoremref}{\ref{thm:EPN-UCE}}
Given a two-layer EPN network with parameters $\tilde{\boldsymbol{\theta}}_n=\{\tilde{\bf W}^{[1]}, \tilde{\bf w}^{[2]}, \tilde{\bf b}^{[1]}, \tilde{b}^{[2]}\}$ defined in (\ref{eq:2-layer-network}), the corresponding EPN-UCE loss, $\ell_{\text{EPN,UCE}}(\mathbf{z}, \mathbf{y}; \tilde{\boldsymbol{\theta}}_n)$, attains its infimum asymptotically at infinity, i.e., as $n\rightarrow \infty$. Furthermore, this parameterized EPN has the property:
\begin{equation}
  \mathbb{P}\Bigl(u_{\text{epi}}\bigl(f_{\text{EPN}}(\mathbf{z}_{y=0}; \tilde{\boldsymbol{\theta}}_n)\bigr) > u_{\text{epi}}(f_{\text{EPN}}\bigl(\mathbf{z}_{y \in \{-1, +1\}}; \tilde{\boldsymbol{\theta}}_n\bigr)\Bigr) \notag \\
  = 0.
\end{equation}
\end{theoremref}

\begin{proof}%
Setting $C=2$ in (\ref{eq: epn_uce_app}), we obtain the following expression for the EPN-UCE loss, 
\begin{equation}
\ell_{\text{EPN,UCE}}\Bigl(\mathbf{z}, y; \boldsymbol{\theta}_{\text{EPN}} \Bigr)
=\psi\Bigl(f_{\text{EPN}}(\mathbf{z}; \boldsymbol{\theta}_{\text{EPN}} ) + 2\Bigr)
-\psi\Bigl(\bigl(f_{\text{EPN}}(\mathbf{z}; \boldsymbol{\theta}_{\text{EPN}}) + 2\bigr)\tilde{p}_{y}(\mathbf{z})\Bigr).
\end{equation}
Taking $x=f_{\text{EPN}}(\mathbf{z}; \boldsymbol{\theta}_{\text{EPN}}) + 2$ and $a=\tilde{p}_{y}\in(0,1)$ in Lemma~\ref{prop:epn_uce_decrease}, we have 
\begin{equation}
g(x):=\psi(x+2) - \psi\bigl(a(x+2)\bigr) > -\ln(a),
\end{equation}
which implies that $\ell_{\text{EPN,UCE}}(\mathbf{z}, y; \boldsymbol{\theta}_{\text{EPN}})$ has a lower bound of $-\ln(\tilde{p}_{y}(\mathbf{{z}}))$. Lemma~\ref{prop:epn_uce_decrease} also indicates that $g(x)$ is a monotonically decreasing function of $x$ for any $a\in(0,1).$
Concsequently, by the Monotone Convergence Theorem, the infimum of $g(x)$ is attained asymptotically at infinity, which suggests that the infimum of the EPN-UCE loss attains when $f_{\text{EPN}}(\mathbf{z}; \tilde{\boldsymbol{\theta}}_n)\rightarrow \infty$, or equivalently when $n\to \infty$.

When $f_{\text{EPN}}(\mathbf{z}; \tilde{\boldsymbol{\theta}}_n)\rightarrow \infty$, the estimated epistemic uncertainty as (\ref{eq:epis_u}) is %
\begin{equation}\label{eq:epi-EPN}
    u_{\text{epi}} \bigl(f_{\text{EPN}}(\mathbf{z}; \tilde{\boldsymbol{\theta}}_n) \bigr)= \frac{2}{f_{\text{EPN}}(\mathbf{z}; \tilde{\boldsymbol{\theta}}_n) + 2} \rightarrow 0,\quad \forall \, \mathbf{z} \in \mathbb{R}^d.
\end{equation}
As the asymptotic behavior (\ref{eq:epi-EPN}) holds regardless of whether \(\mathbf{z}\) is drawn from ID or OOD distribution, it implies that the probability of correctly distinguishing between ID and OOD samples based on the uncertainty is almost impossible. In other words, under a specific data distribution on \(\mathbf{z}\) and a two-layer network with parameters $\tilde{\boldsymbol{\theta}}_n$, we have
\begin{equation}
 \mathbb{P}\Bigl(u_{\text{epi}}\bigl(f_{\text{EPN}}(\mathbf{z}_{y=0}; \tilde{\boldsymbol{\theta}}_n)\bigr) > u_{\text{epi}}(f_{\text{EPN}}\bigl(\mathbf{z}_{y \in \{-1, +1\}}; \tilde{\boldsymbol{\theta}}_n\bigr)\Bigr) = 0.
\end{equation}
\end{proof}

\begin{lemma}\label{prop:epn_uce_decrease}
For a constant scalar \(a\) with \(0 < a < 1\) and the digamma function \(\psi(\cdot)\), we define
\[
\upsilon(x;a) = \psi(x+2) - \psi\bigl(a(x+2)\bigr).
\]
Then \(\upsilon(x;a)\) is a monotonically decreasing function of \(x>0\) and 
\begin{equation}\label{eq:low-bound}
    \lim_{x\to +\infty} \upsilon (x;a)  = -\ln\bigl(a\bigr).
\end{equation}
\end{lemma}
\begin{proof}

First, we compute the derivative of \(\upsilon(x;a)\) with respect to \(x\):
\begin{equation}
\upsilon'(x;a) = \frac{d}{dx}\Bigl[\psi(x+2) - \psi\bigl(a(x+2)\bigr)\Bigr] = \psi^{(1)}(x+2) - a\, \psi^{(1)}\bigl(a(x+2)\bigr),
\end{equation}
where \(\psi^{(1)}(\cdot)\) denotes the trigamma function (the derivative of the digamma function). 
We use the integral representation of the trigamma function: 
\begin{equation}\label{eq:trigamma}
\psi^{(1)}(x)=\int_0^\infty \frac{t\,e^{-xt}}{1-e^{-t}}\,dt,
\end{equation}
which implies that 
\begin{equation}\label{eq:trigamma-a}
a\psi^{(1)}(ax)=a\int_0^\infty \frac{t\,e^{-ax\,t}}{1-e^{-t}}\,dt =a\int_0^\infty \frac{(u/a)\,e^{-ax\,(u/a)}}{1-e^{-u/a}} \cdot \frac{du}{a} = \int_0^\infty \frac{u\,e^{-xu}}{a\Bigl(1-e^{-u/a}\Bigr)}\,du
\end{equation}
with a change of variable by letting $u=at.$ We compare the denominators of the integrands of (\ref{eq:trigamma}) and (\ref{eq:trigamma-a}) with a unified variable $t$ by defining 
\[
h(t) = (1-e^{-t})-a(1-e^{-t/a}).
\]
Clearly, $h(0)=0$ and 
\[
h'(t) =  e^{-t} -e^{-t/a}>0 \quad \forall t>0, 0<a<1.
\]
Therefore, $h(t)\geq 0, \forall t\geq 0.$ It further follows from $t\geq 0, e^{-xt}\geq 0$ and both denominators of the integrands of (\ref{eq:trigamma}) and (\ref{eq:trigamma-a}) are strictly positive that the integrand of  (\ref{eq:trigamma}) is strictly smaller than the one of (\ref{eq:trigamma-a}), which implies that 
$\psi^{(1)}(x) < a\, \psi^{(1)}\bigl(ax\bigr)$ for $x>0.$  As a result, for any constant $a\in (0,1)$, $\upsilon'(x;a)<0$ and hence $\upsilon(x;a)$ is a monotonically decreasing function of $x>0$.   

To show the limit in (\ref{eq:low-bound}), we use the asymptotic expansion of the digamma function, 
\begin{equation}
\psi(x) = \ln x - \frac{1}{2x} + O\Bigl(\frac{1}{x^2}\Bigr) \quad \text{as } x \to +\infty,
\end{equation}
which leads to
\begin{equation*}
\begin{aligned}
    \upsilon(x;a) = &\Bigl(\ln(x+2) - \frac{1}{2(x+2)}\Bigr) - \Bigl(\ln(a(x+2)) - \frac{1}{2a(x+2)}\Bigr)+O\Bigl(\frac{1}{x^2}\Bigr)\\
    =&\Bigl(\ln(x+2) \Bigr) - \Bigl(\ln(a)+\ln(x+2) \Bigr)+O\Bigl(\frac{1}{x}\Bigr) = -\ln(a) +O\Bigl(\frac{1}{x}\Bigr).
\end{aligned}
\end{equation*}
As $x\rightarrow \infty,$ we have $\upsilon(x;a)\rightarrow -\ln(a).$ Due to its monotonicity, $\upsilon(x;a)$ has a lower bound of $-\ln(a).$

\end{proof}

\subsection{Theorem 3}\label{sec: them_3}

For our analysis of the EPN-ICE loss, we start by the optimal class probability vector $\tilde{\mathbf{p}}$ produced by a pre-trained classification model. We first introduce a simplified single-layer classification model in Definition~\ref{def:cnn}, followed by deriving the optimal network parameters trained with cross-entropy loss under the data assumption stated in Assumption~\ref{assum:data}.

\begin{definition}[Single-layer Classification Neural Network] \label{def:cnn}
Let \(\mathcal{Z}\) be a latent space and let \(\mathcal{Y} = \{-1, +1\}\) be the set of class labels. We define a single-layer classification model 
$f_{\text{GNN}}(\bm{\theta}_{\text{GNN}}): \mathcal{Z} \rightarrow [0,1]^2,$
parameterized by \(\bm{\theta}_{\text{GNN}} = (\mathbf{w}_{\text{C}}, b_{\text{C}})\). For any input \(\mathbf{z} \in \mathcal{Z}\), the model computes a single logit:
$v_{\text{C}} = \mathbf{w}_{\text{C}}^\top \mathbf{z} + b_{\text{C}},$
which is then transformed into a probability via the sigmoid function $\sigma(v_{\text{C}}) = \frac{1}{1+e^{-v_{\text{C}}}}.$
The output probability vector is given by
\begin{equation}
    \tilde{\mathbf{p}} = \Bigl[ 1-\sigma(v_{\text{C}}),\; \sigma(v_{\text{C}}) \Bigr],
\end{equation}

where the first component is interpreted as the probability of the class \(y=-1\) and the second as the probability of the class \(y=+1\).
\end{definition}

\begin{proposition}[Optimal Class Probabilities from the Classification Model]
\label{prop:opt_prob}
Assume that the data distribution is given by Assumption~\ref{assum:data} and that an optimal classification model satisfying Definition~\ref{def:cnn} is obtained by minimizing the cross-entropy (CE) loss. Then, for any \(\mathbf{z}\in\mathcal{Z}\), the predicted class probability vector is given by
\begin{equation}\label{eq: opt_cnn}
\tilde{\mathbf{p}}^*(\mathbf{z}) 
=\Biggl[
1-\frac{1}{1+\exp\Bigl(-2\,\boldsymbol{\mu}^\top \boldsymbol{\Sigma}^{-1}\mathbf{z}\Bigr)},
\frac{1}{1+\exp\Bigl(-2\,\boldsymbol{\mu}^\top \boldsymbol{\Sigma}^{-1}\mathbf{z}\Bigr)}
\Biggr].
\end{equation}
\end{proposition}
\begin{proof}
For a given training pair \((\mathbf{z}, y)\) with \(y \in \{-1, +1\}\), the CE loss is defined as
\begin{equation}
    \ell_{\text{CE}}(\mathbf{z}, y; \bm{\theta}_{\text{GNN}}) 
= -\mathbb{I}\{y = +1\}\log\bigl(\sigma(v_{\text{C}})\bigr)
-\mathbb{I}\{y = -1\}\log\bigl(1-\sigma(v_{\text{C}})\bigr),
\end{equation}
which can equivalently be written as
\begin{equation}
    \ell_{\text{CE}}(\mathbf{z}, y; \bm{\theta}_{\text{GNN}}) 
= \log\Bigl(1+\exp\bigl(-y\,v_{\text{C}}\bigr)\Bigr),
\end{equation}
where \(v_{\text{C}} = \mathbf{w}_{\text{C}}^\top\mathbf{z} + b_{\text{C}}\), which is commonly called \textit{logits}.
Under the optimality condition (i.e., when the CE loss is minimized) and the assumption that \( \mathbb{P}(y = +1) = \mathbb{P}(y = -1) \), the logit must satisfy
\begin{equation}
	\begin{aligned}
		v_{\text{C}}^* &= \sigma^{-1}\big(\mathbb{P}(y = +1 \mid \mathbf{z})\big) = \log\left( \frac{\mathbb{P}(y = +1 \mid \mathbf{z})}{\mathbb{P}(y = -1 \mid \mathbf{z})} \right) \\
		&= \log\left( \frac{\mathbb{P}(\mathbf{z} \mid y = +1) \mathbb{P}(y = +1)}{\mathbb{P}(\mathbf{z} \mid y = -1) \mathbb{P}(y = -1)} \right) \\
		&= \log\left( \frac{\mathbb{P}(\mathbf{z} \mid y = +1)}{\mathbb{P}(\mathbf{z} \mid y = -1)} \right).
	\end{aligned}
\end{equation}
Substituting the Gaussian distributions into the above expression gives
\begin{equation}
	\begin{aligned}
		v_{\text{C}}^* = -\tfrac{1}{2} (\mathbf{z} - \bm{\mu})^\top \bm{\Sigma}^{-1} (\mathbf{z} - \bm{\mu}) + \tfrac{1}{2} (\mathbf{z} + \bm{\mu})^\top \bm{\Sigma}^{-1} (\mathbf{z} + \bm{\mu}) = 2 \bm{\mu}^\top \bm{\Sigma}^{-1} \mathbf{z}.
	\end{aligned}
\end{equation}

It follows that the optimal predicted probability for the positive class is
\begin{equation}
\tilde{p}_{y=+1}^* (\mathbf{z}) 
= \sigma\bigl(v_{\text{C}}^*(\mathbf{z})\bigr)
= \frac{1}{1+\exp\Bigl(-2\,\bm{\mu}^\top \bm{\Sigma}^{-1}\mathbf{z}\Bigr)},
\end{equation}
and, consequently, for the negative class, one has
\begin{equation}
\tilde{p}_{y=-1}^* (\mathbf{z}) 
= 1 - \tilde{p}_{y=+1} (\mathbf{z})
= 1 - \frac{1}{1+\exp\Bigl(-2\,\bm{\mu}^\top \bm{\Sigma}^{-1}\mathbf{z}\Bigr)}.
\end{equation}
This completes the proof.
\end{proof}

To enable a theoretical analysis similar to that of the ENN (Theorem~\ref{thm:1}), we construct a specific EPN architecture with its weight and bias matrices given by:
\begin{equation}\label{eq: cons_epn}
\mathbf{W}^{[1]} = [\mathbf{w}_{\text{P}}, -\mathbf{w}_{\text{P}}]^\top, \quad \mathbf{b}^{[1]} = [b_{\text{P}}, -b_{\text{P}}]^\top, \mathbf{w}^{[2]} = {\bf 1}\in\mathbb R^{2}, b^{[2]} = 0
\end{equation}
where \( \mathbf{w}_{\text{P}} \in \mathbb{R}^d, b_{\text{P}}\in \mathbb{R} \).  %
Based on the two-layer EPN architecture defined in (\ref{eq: epn_arch}), the hidden representation \(\mathbf{q}\) is defined by:
\begin{equation}
\mathbf{q} = \exp \bigl(\mathbf{W}^{[1]} \mathbf{z}
+ \mathbf{b}^{[1]}\bigr) = \Bigl[ \exp\bigl(\mathbf{w}_{\text{P}}^\top \mathbf{z} + b_{\text{P}}\bigr),  \exp \bigl(-\mathbf{w}_{\text{P}}^\top \mathbf{z} - b_{\text{P}}\bigr) \Bigr].
\end{equation}
The EPN output is, 
\begin{equation}
    e_{\text{total}} = f_{\text{EPN}}(\mathbf{z}; \vec{\boldsymbol{\theta}}_{\text{EPN}}) = \exp(\mathbf{w}_{\text{P}}^\top \mathbf{z} + b_{\text{P}}) + \exp(-\mathbf{w}_{\text{P}}^\top \mathbf{z} - b_{\text{P}}) = 2\cdot \cosh(\mathbf{w}_{\text{P}}^\top \mathbf{z} + b_{\text{P}}).
\end{equation}

\begin{theoremref}{\ref{thm:EPN-ICE}}
Given a well-trained classification model producing the class probabilities $[\tilde{\mathbf{p}}^i]_{i \in \mathbb{V}}$ in (\ref{eq: opt_cnn}) and constructed EPN parameters in (\ref{eq: cons_epn}), we have that for any \(\epsilon > 0\), there exists a positive constant \(\digamma > 0\) such that, for any data distribution satisfying Assumption~\ref{assum:data} with \(\|\bm{\mu}\|_2 > \digamma \), the probability that the epistemic uncertainty obtained by an optimal two-layer EPN solely based on the ICE loss, correctly distinguishes ID and OOD samples is greater than \(1 - \epsilon\).
\end{theoremref}

\begin{proof}
Under the binary classification setting ($C=2$), the ICE regularization term is defined by
 \begin{equation}
    \ell^i_{\text{ICE}}(e_{\text{total}}^i, \mathbf{y}^i; \vec{\boldsymbol{\theta}}_{\text{EPN}}) 
    = \left\lVert (C + e_{\text{total}}^i) \cdot \tilde{\mathbf{p}}^i - \mathbf{q}^i \right\rVert_2^2.
 \end{equation}
 For simplicity, we denote $m = \mathbf{w}_{\text{P}}^\top \mathbf{z} + b_{\text{P}} \in \mathbb{R}$, $l = 2\,\boldsymbol{\mu}^\top\boldsymbol{\Sigma}^{-1}\mathbf{z} \in \mathbb{R} $, then we have 
\begin{equation}
\begin{aligned}
 \ell_{\text{EPN, ICE}}(\mathbf{z}; \mathbf{w}_{\text{P}}, b_{\text{P}}) = \Biggl\| \Bigl(2 + 2\cosh(m)\Bigr)
\begin{bmatrix} 
1-\dfrac{1}{1+\exp\bigl(-l\bigr)} \\
\dfrac{1}{1+\exp\bigl(-l\bigr)}
\end{bmatrix} 
\quad - 
\begin{bmatrix} 
1+\exp\bigl(-m\bigr) \\
1+\exp\bigl(m\bigr)
\end{bmatrix} \Biggr\|_2^2 .
\end{aligned}
\end{equation}

Recall our training data distribution, 
\begin{equation}
    p(\mathbf{z}) = \frac{1}{2}\,\mathcal{N}(\boldsymbol{\mu},\boldsymbol{\Sigma})
+ \frac{1}{2}\,\mathcal{N}(-\boldsymbol{\mu},\boldsymbol{\Sigma}).
\end{equation}

The ICE loss achieves its minimum (zero) when the vector inside the norm vanishes, which occurs if
\begin{align}
        \Bigl(2+2\cosh(m)\Bigr)\left(1-\frac{1}{1+\exp(-l)}\right) &= 1+\exp(-m) \label{eq:first}\\
        \Bigl(2+2\cosh(m)\Bigr)\frac{1}{1+\exp(-l)} &= 1+\exp(m) \label{eq:second}.
\end{align}
Notice that
$
    2+2\cosh(m)=\bigl(1+\exp(-m)\bigr)+\bigl(1+\exp(m)\bigr).
$
Taking the ratio of (\ref{eq:first}) and (\ref{eq:second}), which is valid when neither one is zero, yields
\begin{equation}\label{eq:temp}
\frac{1-\frac{1}{1+\exp(-l)}}{\frac{1}{1+\exp(-l)}}=\frac{1+\exp(-m)}{1+\exp(m)}.
\end{equation}
Simple calculations show that the left-hand side of (\ref{eq:temp}) equals $\exp(-l)$ and the right-hand side equals $\exp(-m),$  implying that $m=l.$ Consequently, we should have  that
\begin{equation}
\mathbf{w}_{\text{P}}^\top\,\mathbf{z}+b_{\text{P}} = 2\,\boldsymbol{\mu}^\top\boldsymbol{\Sigma}^{-1}\mathbf{z},\quad \text{for all } \mathbf{z},
\end{equation}
which holds if and only if
\begin{equation}
\mathbf{w}_{\text{P}}^\top = 2\,\boldsymbol{\mu}^\top\boldsymbol{\Sigma}^{-1} \quad \text{and} \quad b_{\text{P}} = 0.
\end{equation}

Therefore, the analytical solution that minimizes the  CE loss is
\begin{equation} \label{eq:optim_ice}
\mathbf{w}_{\text{P}}^* = 2\,\boldsymbol{\Sigma}^{-1}\boldsymbol{\mu},\quad b_{\text{P}}^* = 0.
\end{equation}

Since the linear relationship of the optimal solutions in (\ref{eq:optim_ice}) is analogous to (\ref{eq: optim_enn}) in Proposition~\ref{prop:opt_enn}, the proof of Theorem~\ref{thm:EPN-ICE} is analogous to that of Theorem~\ref{thm:1}, thus omitted.
\end{proof}

\section{Experimental setup}
\subsection{Dataset details}\label{sec:datasets}

\paragraph{Statistics} We utilize three citation networks: CoraML, CiteSeer, and PubMed ~\citep{bojchevski2017deep}, along with two co-purchase Amazon datasets, namely Computers and Photos ~\citep{shchur2018pitfalls}. Additionally, we incorporate two coauthor datasets, Coauthor CS and Coauthor Physics \citep{shchur2018pitfalls}, as well as a large-scale dataset, OGBN-Arxiv ~\citep{hu2020open}. All datasets are taken from PyTorch Geometric \citep{fey2019fast}. 

We follow the dataset split in~\cite{zhao2020uncertainty}. In detail, for all the datasets except the OGBN-Arxiv, we use the default split of 20 training samples per class. We use 20\% nodes for testing, and then the remaining for validation (i.e., close to 80\%). For the OGBN-Arxiv, we use the public split. To avoid the randomness,  we report the results as averages over 5 random model initializations and 5 data splits. Table \ref{tab:dataset} provides a detailed summary of the datasets and the number of categories used for out-of-distribution (OOD) detection.
\begin{table*}[htbp!]
\scriptsize
\centering
\caption{Dataset Statistics}
\vspace{1mm}
\begin{tabular}{p{2.4cm}|ccc|cc|cc|c}
\toprule
        & CoraML & CiteSeer & PubMed & Amazon& Amazon & Coauthor&  Coauthor& OGBN-Arxiv \\ 
        &&& & Computers & Photo & CS & Physics& \\
\midrule
\# Nodes             & 2,995 & 4,230 & 19,717 & 13,752 & 7,650 & 18,333 & 34,493 & 169,343 \\
\# Edges             & 16,316 & 10,674 & 88,648 & 491,722 & 238,162 & 163,788 & 495,924 & 2,315,598 \\
\# Features        & 2,879  & 602   & 500    & 767    & 745   & 6,805  & 8,415  & 128     \\
\# Classes           & 7     & 6     & 3      & 10     & 8     & 15     & 5      & 40      \\
\# Train nodes & 140 & 120 & 60 & 20 & 160 & 300 & 100 & 91,445 \\
\# Left out classes  & 3     & 2     & 1      & 5      & 3     & 4      & 2      & 15      \\ 
\bottomrule
\end{tabular}
\label{tab:dataset}
\end{table*}

\subsection{Evaluation}\label{sec:eva}

\paragraph{Classification} We assess the node-level classification performance on the original graphs, reporting accuracy (\textit{ACC}). Additionally, we evaluate the model's calibration performance using the Brier Score (\textit{BS}) and Expected Calibration Error (\textit{ECE}).

\paragraph{Misclassification Detection} This evaluation is conducted on a clean graph by comparing the model's predictions with the ground truth labels. Misclassification detection is framed as a binary classification task, where misclassified samples are treated as positive instances and correctly classified samples as negative instances. Aleatoric uncertainty is utilized as the scoring metric, and we report AUCROC and the AUCPR.

\textbf{OOD Detection} Similar to misclassification detection, OOD detection is approached as a binary classification task, where out-of-distribution (OOD) data serve as the positive instances and in-distribution (ID) data as the negative instances. Epistemic uncertainty is used as the scoring metric for AUROC and AUPR. 

To construct the OOD data, we adopt the Left-Out-Classes (LOC) setting,  where nodes from predefined OOD classes are excluded from the training set and are reintroduced during testing. We consider five OOD class selection settings, as shown in Table \ref{tab: loc}. The OS-1 setting aligns closely with the setup used in \citet{zhao2020uncertainty, stadler2021graph}, while OS-2 mirrors that of \citet{wu2023energy}. Results for each setting are reported individually, as well as the averaged performance across all five settings.

\begin{table*}[htbp]
\caption{Dataset details summarizing the number of classes,  let-out-classes for 5 different OOD settings}
\label{tab: loc}
\vskip 0.15in
\begin{center}
\begin{scriptsize}
\begin{tabular}{p{0.15\linewidth} |p{0.08\linewidth} |p{0.1\linewidth} | p{0.1\linewidth}| p{0.1\linewidth}| p{0.1\linewidth}| p{0.1\linewidth}}
\toprule
Dataset & Categories classes& OS-1  & OS-2 & OS-3 & OS-4 & OS-5\\
\midrule
CoraML    & 7 & [4, 5, 6]  & [0,1,2]  & [0, 2, 4] & [1, 3, 5] & [3, 4, 5] \\
CiteSeer & 6 & [4,5] & [0,1]  & [1, 2] & [3, 4]  & [0, 5] \\
PubMed   & 3  & [2] & [0] & [1] & - & - \\
AmazonPhoto    & 8 & [5, 6, 7]& [0, 1, 2] & [3, 4, 5] & [1, 4, 6] & [2, 3, 7] \\
AmazonComputers   & 10 & [5, 6, 7, 8, 9] & [0, 1, 2, 3, 4] & [2, 3, 4, 5, 7] & [1,2,3,6,7] & [2, 4, 5, 8 ,9]\\
CoauthorCS   & 15 &  [11, 12, 13, 14] & [0,1,2,3] & [1,2,9,12] & [3, 6, 10, 13] & [2,3, 6, 10]\\
CoauthorPhysics    & 5 &  [3,4]  & [0,1] & [2,3] & [0,4] & [1, 2] \\ 
OGBN-Arxiv & 40 & [25,  26,  27,  28, 29, 30, 31, 32, 33, 34, 35, 36, 37, 38, 39] &
[0, 1, 2, 3, 4, 5, 6, 7, 8, 9, 10, 11, 12, 13, 14] &
[0,  1,  5,  9, 10, 11, 15, 24, 25, 30, 32, 33, 34, 36, 38] &
[4,  7, 12, 15, 16, 17, 18, 21, 22, 25, 26, 28, 29, 31, 32] &
[5,  9, 10, 11, 12, 13, 19, 22, 24, 25, 29, 34, 35, 36, 37] \\
\bottomrule
\end{tabular}
\end{scriptsize}
\end{center}
\vskip -0.1in
\end{table*}

\subsection{Model Details} \label{sec: baselines}
We follow the choice of the architectures in \citet{stadler2021graph}. By default, we use a hidden dimension of $h=64$ two-layer GCN as the basic architecture for all datasets, except that we use three layers of GCN with a hidden dimension of $[256, 128]$, as well as batch norm. We use the early-stopping strategy with the validation loss as the monitor metric and patience of 50 for all datasets except 200 for OGBN-Arxiv, and we select the model with the best validation loss. If not specified explicitly, we use the Adam optimizer with a learning rate of 0.001 and weight decay of 0.0001 for all models. 

\textbf{VGNN-entropy/VGNN-max\_score}: We use the vanilla GCN trained with cross-entropy loss. Following the work of \cite{hendrycks2016baseline}, we use the softmax probability to derive the uncertainty.
\begin{align}
    \text{VGNN-entropy}: u_{\text{alea}} = u_{\text{epi}} = \mathbb{H}(\bm p) \\
    \text{VGNN-max\_score}:  u_{\text{alea}} = u_{\text{epi}} = 1 - \max \bm p .
\end{align}

\textbf{VGNN-energy/VGNN-gnnsafe}: \cite{liu2020energy} proposed to use the energy score to distinguish the ID and OOD, and \cite{wu2023energy} extended the energy method to the graph domain and added label propagation on the graph. We do not consider the pseudo-OOD in the training stage, and hence, we do not consider the regularized learning in these two works. 
\begin{align}
      \text{VGNN-energy}:  u_{\text{alea}} &= u_{\text{epi}}= E = -T \sum_c \exp^{l_c/T} \\
      \text{VGNN-gnnsafe}:  u_{\text{alea}} &= u_{\text{epis}} = E^K \\ 
      E^{k} &= \gamma E^{k-1} + (1 - \gamma)\mathbf{D}^{-1}\mathbf{A}E^{k-1},
\end{align}
where $E$ is the energy score, $l_c$ is the logit corresponding to class $c$, and we use $T=1$ as the temperature, $\gamma=0.2$, the number of propagation iterations is $K=2$. 

\textbf{VGNN-dropout/VGNN-ensemble}: \citep{gal2016dropout} propose to use MC-dropout to capture the uncertainty. We use a dropout probability of 0.5 and evaluate the model 10 times to capture the uncertainty. For the ensemble model \citep{lakshminarayanan2017simple}, we train 10 models with different weight initialization.  We use the 1 minus max score of expected class probabilities as the uncertainty estimation. 

\textbf{GPN}: It is an end-to-end learning with three steps. First, a multi-layer perception is used to encode the raw node features $\bm x$ to latent space $\mathbf{z}$. Then Normalizing Flows are used to fit the class-wise densities on the latent space, followed by multiplying a certainty budget to get non-negative evidence $\boldsymbol{\alpha}$. Lastly, an APPNP network is used to propagate the evidence through the graph and $\boldsymbol{\alpha}^{\text{agg}}$ is used to estimate the uncertainties and class probabilities based on subjective logic theory. 
\begin{align}
    u_{\text{alea}} = -\max \frac{\alpha^{\text{agg}}_c}{\sum_c \alpha^{\text{agg}}_c } \quad \text{and} \quad
    u_{\text{epi}} = \frac{C}{\sum_c \alpha^{\text{agg}}_c }.
\end{align}
We follow the experimental setup described in the GPN paper by \citet{stadler2021graph}. We use warmup for Normalizing Flows for 5 epochs with a learning rate 1e-3. For all datasets except OGBN-Arxiv, we use a two-layer MLP with a hidden dimension of 16. For OGBN-Arxiv, we use a three-layer MLP with hidden dimensions of 256 and 128. In terms of propagation, we employ 10 power-iteration steps.
For the CoraML, CiteSeer, PubMed, CoauthorCS, and CoauthorPhysics datasets, we use a latent dimension of 16, a dropout rate of 0.5, a teleport probability of 0.1 in APPNP, an entropy regularization weight of 1.0e-03, a weight decay of 0.001, and a certainty budget of $\sqrt{4\pi}^{16}$.

For the AmazonPhoto and AmazonComputers datasets, the latent dimension is set to 10, with a dropout rate of 0.5, a teleport probability of 0.2 in APPNP, an entropy regularization weight of 1.0e-05, a weight decay of 0.0005, and a certainty budget of $C\sqrt{4\pi}^{10}$.

For the OGBN-Arxiv dataset, we use a latent dimension of 16, a dropout rate of 0.25, a teleport probability of 0.2 in APPNP, an entropy regularization weight of 1.0e-05, no weight decay, and a certainty budget of $\sqrt{4\pi}^{16}$. Unlike the original paper, we do not apply BatchNorm due to instability issues.

\textbf{SGNN-GKDE}: For the probability teacher, we use the two-layer GCN architecture with a hidden dimension of 64. For the alpha teacher, i.e., graph kernel Dirichlet estimation, we use 10 as the distance cutoff and sigma is 1. For the backbone, we use the two-layer GCN with a hidden dimension of 16. When discussing the loss function, the probability teacher has a weight $\min(1, t/200)$ and an alpha teacher weight of 0.1, as reported in the paper. Notably, the loss function is defined by
\begin{equation}\label{eq:sgcnloss}
 \sum_{c=1}^{C} \left( \left(y_j - \bar{p}_j\right)^2 + \frac{p_j(1-p_j)}{\sum_c\alpha_c + 1} \right) + \lambda_1 \text{KL}\left[\text{Dir}(\boldsymbol{\alpha}) \parallel \text{Dir}(\hat{\boldsymbol{\alpha}})\right] + \lambda_2 \text{KL}\left(\bm p \parallel \hat{\bm p} \right),
\end{equation}
where $\boldsymbol{\alpha}$ is the model prediction, $\hat{\boldsymbol{\alpha}}$ is the GKDE prior, $\bar{\bm{p}}=\bm{\alpha}/\sum_c \alpha_c$, and $\hat{\bm p}$ is the probability from teacher. The first component is the expected mean square loss described in~\cite{sensoy2018evidential}. 

\textbf{EGNN} We directly extend EDL~\citep{sensoy2018evidential} to the graph domain with the expected cross-entropy loss described in~\cite{sensoy2018evidential} associated with the entropy regularization used in~\cite{charpentier2020posterior, stadler2021graph}, i.e.,
\begin{equation}
    \sum_{c=1}^C y_c\left (\psi(\sum_c \alpha_c) - \psi(\alpha_c) \right) + \text{KL}(\boldsymbol{\alpha} \parallel \bm 1),
\end{equation}
where $\psi$ is the digamma function.

\textbf{EPN.} The model architecture of our proposed model is presented in Figure ~\ref{fig:model_archi}. Given an attributed graph $\mathcal{G} = (\mathbb{V}, \mathbf{A}, \mathbf{X}, \mathbf{y}_{\mathbb{L}})$, the pretrained GNN (bottom left) produces node-level class probability vectors $[\tilde{\mathbf{p}}^i]_{i \in \mathbb{V}}$ and node-level embeddings $[\mathbf{z}^i]_{i \in \mathbb{V}}$. These embeddings are passed through an MLP head within the EPN to generate node-level total evidence $[e^i_{\text{total}}]_{i \in \mathbb{V}}$. This total evidence is combined with the class probability vectors $[\tilde{\mathbf{p}}^i]_{i \in \mathbb{V}}$ to compute node-level Dirichlet parameters $[\boldsymbol{\alpha}^i]_{i \in \mathbb{V}}$ using the relationship $\boldsymbol{\alpha}^i = (C + e_{\text{total}}^i) \cdot \tilde{\mathbf{p}}^i$. A label propagation layer smooths these Dirichlet parameters across the graph to obtain aggregated parameters $[\boldsymbol{\alpha}^i_{\text{agg}}]_{i \in \mathbb{V}}$. These outputs enable computation of node-level epistemic uncertainty ($u^i_{\text{epi}}$) for out-of-distribution (OOD) detection and node-level aleatoric uncertainty ($u^i_{\text{alea}}$) for misclassification detection.

To enhance the EPN's predictive performance, two regularization terms are introduced. The Intra-Class Evidence-based (ICE) regularization term minimizes the distance between the learned hidden representations $[\mathbf{q}^i]_{i \in \mathbb{V}}$ from the last hidden layer of EPN and the predicted class-level evidence vectors $[\boldsymbol{\alpha}^i]_{i \in \mathbb{V}}$, promoting alignment between evidence and class probabilities. Additionally, the Positive-Confidence Learning (PCL) term encourages the model to enforce high evidence for in-distribution (ID) nodes and low evidence for OOD nodes, balancing the predictive uncertainty. This design, alongside the Uncertainty-Cross-Entropy (UCE) loss, ensures effective uncertainty calibration for both ID and OOD scenarios.

In our experiment, we use a two-layer MLP as the architecture of EPN, as defined in Equation (\ref{eq: epn_arch}). Furthermore, we set that $\mathbf{w}^{[2]} = \mathbf{1}, b^{[2]}=0$. 
\begin{figure}
    \centering
    \includegraphics[width=0.99\linewidth]{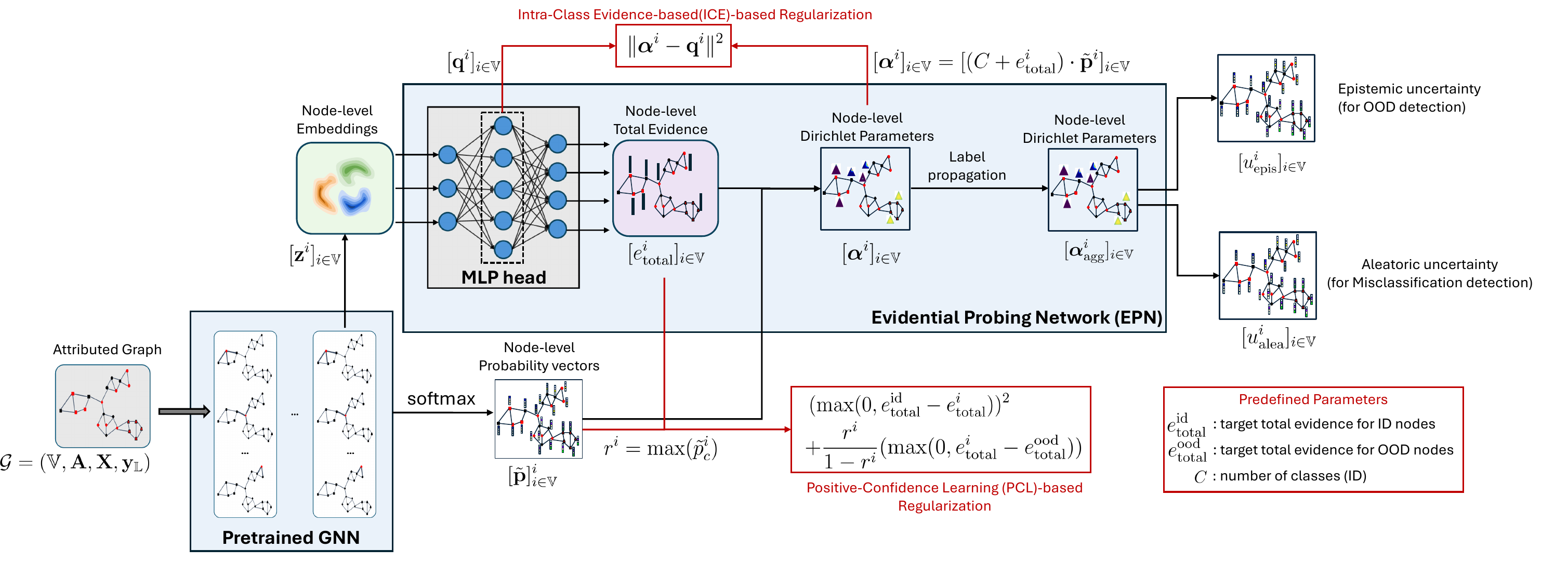}
    \caption{The flow chart of our proposed EPN network and its interaction with the two regularization terms, including the Intra-Class Evidence-based (ICE) and Positive-Confidence Learning (PCL).}
    \label{fig:model_archi}
\end{figure}

\textbf{Uncertainty propagation:} Motivate by the evidence propagation in~\cite{stadler2021graph} and energy propagation in~\cite{wu2023energy}, we propose two variants for the EGCN model. 
\begin{align}
    \text{vacuity-prop}: & \ \alpha_0^k = \gamma^1 \alpha_0^{k-1} + (1 - \gamma^1)\mathbf{D}^{-1}\mathbf{A}\alpha_0^{k-1} \\
    \text{evidence-prop}: & \  \boldsymbol{\alpha}^{k} = (1 - \gamma^2) \mathbf{\hat{D}}^{-1/2}\mathbf{\hat{A}} \mathbf{\hat{D}}^{-1/2} \boldsymbol{\alpha}^{(k-1)} + \gamma^2 \boldsymbol{\alpha}^0 .
\end{align}
We also consider one variant that considers both class-wise evidence propagation and vacuity propagation. Following these techniques proposed in the original paper,  we use 10 iterations for the propagation of evidence with $\gamma=0.1$ and 2 iterations with $\gamma^2=0.5$. 

\subsection{Hyperparameter selection} 
For the baseline models, we adopt the default hyperparameters as specified in their original papers or associated code repositories. In our model, there are two key hyperparameters in our loss function, i.e., $\lambda_1$ for the ICE regularization and $\lambda_2$ for the PCL regularization. 
We determine these values using AUROC for misclassification detection. For OOD detection, we employ the OS-1 setting to optimize the hyperparameters based on AUROC and subsequently apply them across other OS settings. 

\section{Additional Experiments}\label{sec: add_exp}

\subsection{Misclassification detection}
We report evaluation results on the clean graph, covering classification accuracy, calibration performance (in terms of BS and ECE), and misclassification detection. Based on the ranking across datasets as shown in Table  \ref{tab:mis_rank}, our model demonstrates superior misclassification detection performance. As reflected in the detailed metrics in Table \ref{tab:mis_1} and Table \ref{tab:mis_2},  the performance differences among models are minimal.
\begin{table}[ht]
    \centering
    \small
    \setlength{\tabcolsep}{10pt}
    \resizebox{0.9\textwidth}{!}{%
    \begin{tabular}{@{}p{1cm}l|ccc|cc@{}}
    \hline
    \multirow{2}{*}{\textbf{Dataset}} & \multirow{2}{*}{\textbf{Model}} & \multicolumn{5}{c}{\textbf{Clean Graph}} \\ 
    ~ & ~ & clean-acc$\uparrow$ & BS$\downarrow$ & ECE$\downarrow$ & MIS\_ROC$\uparrow$ & MIS\_PR$\uparrow$  \\ \hline
    \multirow{14}{*}{\rotatebox[origin=c]{90}{\textbf{CoraML}}\hspace{0.5cm}} 
    ~ & VGNN-entropy & 79.99 ± 3.97 & 29.31 ± 4.04 & 6.44 ± 1.30 & 80.90 ± 1.54 & 48.76 ± 7.63 \\
    
    ~ & VGNN-max-score & 79.99 ± 3.97 & 29.31 ± 4.04 & 6.44 ± 1.30 & 82.29 ± 1.41 & 51.76 ± 6.57 \\
    
    ~ & VGNN-energy & 79.99 ± 3.97 & 29.31 ± 4.04 & 6.44 ± 1.30 & 78.51 ± 1.99 & 45.69 ± 7.61 \\
    
    ~ & VGNN-gnnsafe & 79.99 ± 3.97 & 29.31 ± 4.04 & 6.44 ± 1.30 & 80.15 ± 1.41 & 45.54 ± 6.52 \\
    
    ~ & VGNN-dropout & 81.62 ± 1.86 & 27.11 ± 2.02 & \textcolor{blue}{6.38 ± 1.27} & 82.32 ± 0.83 & 49.18 ± 4.61 \\
    
    ~ & VGNN-ensemble & 79.68 ± 0.96 & 29.35 ± 0.82 & \textcolor{red}{6.20 ± 1.02} & 81.37 ± 1.32 & 51.53 ± 3.20 \\
    
    ~ & GPN & 77.74 ± 1.77 & 33.60 ± 1.62 & 12.00 ± 2.31 & 82.30 ± 1.54 & 55.72 ± 4.12 \\
    
    ~ & SGNN-GKDE & 40.11 ± 27.02 & 76.49 ± 14.69 & 19.74 ± 17.44 & 74.12 ± 8.04 & \textcolor{red}{75.59 ± 17.39} \\
    
    ~ & EGNN & 82.18 ± 1.00 & 26.89 ± 0.85 & 7.79 ± 1.60 & 83.08 ± 1.29 & 51.17 ± 3.65 \\
    
    ~ & EGNN-vacuity-prop & \textcolor{blue}{82.48 ± 1.16} & \textcolor{blue}{26.36 ± 1.08} & 7.70 ± 1.34 & 83.84 ± 1.44 & 51.31 ± 4.63 \\
    
    ~ & EGNN-evidence-prop & 81.09 ± 1.16 & 30.03 ± 0.86 & 15.22 ± 1.27 & \textcolor{blue}{84.83 ± 1.39} & 54.87 ± 4.55 \\
    
    ~ & EGNN-vacuity-evidence-prop & 81.17 ± 0.81 & 29.83 ± 1.02 & 15.54 ± 1.30 & \textcolor{red}{84.97 ± 0.97} & 55.83 ± 4.33 \\
    
    ~ & EPN & \textcolor{red}{83.22 ± 0.56} & \textcolor{red}{25.21 ± 0.31} & 8.35 ± 0.88 & 84.42 ± 0.51 & \textcolor{blue}{56.23 ± 1.78} \\
    \rowcolor{gray!30}
    ~ & EPN-reg & 81.78 ± 0.74 & 26.99 ± 0.74 & 6.69 ± 1.22 & 83.89 ± 1.66 & 52.03 ± 5.03 \\
    \hline
    \multirow{14}{*}{\rotatebox[origin=c]{90}{\textbf{CiteSeer}}\hspace{0.5cm}} 
    ~ & VGNN-entropy & 83.34 ± 0.90 & 24.77 ± 1.41 & 4.63 ± 0.91 & 83.30 ± 1.32 & 50.92 ± 3.81 \\
    
    ~ & VGNN-max-score & 83.34 ± 0.90 & 24.77 ± 1.41 & 4.63 ± 0.91 & 84.31 ± 1.19 & 53.71 ± 3.58 \\
    
    ~ & VGNN-energy & 83.34 ± 0.90 & 24.77 ± 1.41 & 4.63 ± 0.91 & 81.90 ± 1.36 & 49.60 ± 3.96 \\
    
    ~ & VGNN-gnnsafe & 83.34 ± 0.90 & 24.77 ± 1.41 & 4.63 ± 0.91 & 84.85 ± 1.09 & 54.91 ± 3.83 \\
    
    ~ & VGNN-dropout & 83.90 ± 1.48 & 24.62 ± 1.46 & 4.89 ± 0.96 & 82.81 ± 1.43 & 46.81 ± 4.87 \\
    
    ~ & VGNN-ensemble & 83.59 ± 1.77 & 25.09 ± 1.76 & 4.94 ± 0.96 & 82.02 ± 1.53 & 43.63 ± 6.80 \\
    
    ~ & GPN & 84.82 ± 1.04 & \textcolor{red}{22.33 ± 1.54} & 4.27 ± 1.00 & \textcolor{blue}{86.97 ± 1.75} & 55.12 ± 5.86 \\
    
    ~ & SGNN-GKDE & 82.90 ± 0.91 & 45.38 ± 0.98 & 38.27 ± 1.05 & 83.55 ± 1.82 & 53.18 ± 3.29 \\
    
    ~ & EGNN & 82.93 ± 1.18 & 25.06 ± 1.44 & \textcolor{blue}{4.06 ± 0.63} & 85.94 ± 1.35 & \textcolor{blue}{56.52 ± 4.17} \\
    
    ~ & EGNN-vacuity-prop & 83.55 ± 0.97 & 24.26 ± 1.19 & 4.69 ± 0.97 & 85.78 ± 1.11 & 54.59 ± 4.19 \\
    
    ~ & EGNN-evidence-prop & \textcolor{blue}{85.00 ± 1.25} & 22.80 ± 1.06 & 4.42 ± 1.37 & 84.74 ± 1.92 & 51.27 ± 5.47 \\
    
    ~ & EGNN-vacuity-evidence-prop & \textcolor{red}{85.71 ± 1.23} & \textcolor{blue}{22.29 ± 0.61} & 5.33 ± 1.80 & 83.96 ± 1.72 & 48.49 ± 6.50 \\
    
    ~ & EPN & 83.41 ± 0.42 & 24.08 ± 0.45 & \textcolor{red}{3.99 ± 0.51} & \textcolor{red}{87.68 ± 0.26} & \textcolor{red}{59.88 ± 2.79} \\
    \rowcolor{gray!30}
    ~ & EPN-reg & 84.55 ± 0.96 & 23.18 ± 0.98 & 4.38 ± 0.82 & 85.60 ± 1.67 & 52.25 ± 5.43 \\
    \hline
    \multirow{14}{*}{\rotatebox[origin=c]{90}{\textbf{PubMed}}\hspace{0.5cm}} 
    ~ & VGNN-entropy & 78.25 ± 1.98 & 31.40 ± 2.09 & 3.95 ± 1.29 & 72.46 ± 1.65 & 39.26 ± 2.61 \\
    
    ~ & VGNN-max-score & 78.25 ± 1.98 & 31.40 ± 2.09 & 3.95 ± 1.29 & 74.33 ± 1.35 & 42.06 ± 2.56 \\
    
    ~ & VGNN-energy & 78.25 ± 1.98 & 31.40 ± 2.09 & 3.95 ± 1.29 & 67.75 ± 2.15 & 36.05 ± 2.57 \\
    
    ~ & VGNN-gnnsafe & 78.25 ± 1.98 & 31.40 ± 2.09 & 3.95 ± 1.29 & 68.99 ± 2.04 & 35.17 ± 2.45 \\
    
    ~ & VGNN-dropout & 78.69 ± 1.42 & 30.73 ± 1.72 & 2.86 ± 0.54 & 72.85 ± 1.35 & 39.53 ± 1.28 \\
    
    ~ & VGNN-ensemble & 78.16 ± 1.07 & 31.13 ± 1.10 & 2.88 ± 0.66 & 74.25 ± 2.40 & 42.97 ± 3.63 \\
    
    ~ & GPN & 80.06 ± 1.90 & 29.21 ± 2.82 & 3.47 ± 1.30 & 76.22 ± 2.66 & 42.74 ± 1.72 \\
    
    ~ & SGNN-GKDE & 78.22 ± 2.19 & 36.40 ± 1.75 & 16.93 ± 1.90 & 74.23 ± 2.27 & 42.60 ± 2.70 \\
    
    ~ & EGNN & \textcolor{blue}{80.16 ± 1.44} & \textcolor{blue}{28.82 ± 1.68} & 3.00 ± 0.82 & \textcolor{blue}{76.18 ± 0.70} & 41.95 ± 1.32 \\
    
    ~ & EGNN-vacuity-prop & 78.96 ± 1.99 & 30.27 ± 2.66 & \textcolor{blue}{2.48 ± 0.55} & 75.46 ± 1.72 & 42.99 ± 1.91 \\
    
    ~ & EGNN-evidence-prop & 76.10 ± 2.89 & 33.81 ± 3.29 & 3.92 ± 1.13 & 75.09 ± 2.10 & \textcolor{red}{46.26 ± 2.78} \\
    
    ~ & EGNN-vacuity-evidence-prop & 79.02 ± 1.00 & 31.13 ± 1.72 & 3.45 ± 1.64 & 74.38 ± 2.22 & 42.52 ± 2.06 \\
    
    ~ & EPN & \textcolor{red}{80.48 ± 0.14} & \textcolor{red}{28.24 ± 0.25} & \textcolor{red}{2.44 ± 0.83} & \textcolor{red}{76.32 ± 0.01} & 41.38 ± 0.33 \\
    \rowcolor{gray!30}
    ~ & EPN-reg & 78.20 ± 0.65 & 31.12 ± 0.83 & 2.80 ± 0.74 & 75.36 ± 1.09 & \textcolor{blue}{44.24 ± 1.72} \\
    \hline
    \multirow{14}{*}{\rotatebox[origin=c]{90}{\textbf{AmazonPhotos}}\hspace{0.5cm}} 
    ~ & VGNN-entropy & \textcolor{blue}{90.98 ± 0.63} & \textcolor{red}{15.01 ± 0.81} & 5.38 ± 1.24 & 84.07 ± 0.61 & 38.67 ± 3.32 \\
    
    ~ & VGNN-max-score & \textcolor{blue}{90.98 ± 0.63} & \textcolor{red}{15.01 ± 0.81} & 5.38 ± 1.24 & 85.65 ± 0.47 & 43.52 ± 3.53 \\
    
    ~ & VGNN-energy & \textcolor{blue}{90.98 ± 0.63} & \textcolor{red}{15.01 ± 0.81} & 5.38 ± 1.24 & 76.03 ± 1.62 & 30.35 ± 4.45 \\
    
    ~ & VGNN-gnnsafe & \textcolor{blue}{90.98 ± 0.63} & \textcolor{red}{15.01 ± 0.81} & 5.38 ± 1.24 & 74.16 ± 1.92 & 21.76 ± 1.66 \\
    
    ~ & VGNN-dropout & 88.77 ± 1.68 & 18.14 ± 2.37 & \textcolor{blue}{5.24 ± 1.13} & 84.13 ± 1.67 & 43.61 ± 5.04 \\
    
    ~ & VGNN-ensemble & 90.30 ± 0.95 & 15.68 ± 1.42 & 5.55 ± 0.84 & 84.71 ± 0.96 & 40.40 ± 3.38 \\
    
    ~ & GPN & 88.66 ± 0.99 & 21.99 ± 1.07 & 15.34 ± 1.22 & 83.09 ± 3.07 & 42.32 ± 6.38 \\
    
    ~ & SGNN-GKDE & 12.95 ± 9.58 & 87.48 ± 0.09 & 7.04 ± 6.57 & 55.73 ± 7.51 & \textcolor{red}{88.96 ± 9.40} \\
    
    ~ & EGNN & \textcolor{red}{91.56 ± 0.48} & \textcolor{blue}{15.27 ± 0.55} & 8.85 ± 1.05 & 84.78 ± 1.57 & 37.25 ± 5.00 \\
    
    ~ & EGNN-vacuity-prop & 90.11 ± 1.36 & 16.73 ± 1.89 & 7.69 ± 1.46 & 85.33 ± 1.66 & 43.12 ± 5.55 \\
    
    ~ & EGNN-evidence-prop & 67.51 ± 4.74 & 43.43 ± 4.30 & 16.83 ± 2.81 & \textcolor{red}{90.53 ± 3.16} & \textcolor{blue}{78.59 ± 8.67} \\
    
    ~ & EGNN-vacuity-evidence-prop & 66.98 ± 5.05 & 44.20 ± 3.88 & 17.55 ± 3.63 & \textcolor{blue}{89.96 ± 4.17} & 77.45 ± 11.62 \\
    
    ~ & EPN & 89.65 ± 0.10 & 16.10 ± 0.18 & \textcolor{red}{4.14 ± 0.42} & 86.69 ± 0.48 & 47.58 ± 2.27 \\
    \rowcolor{gray!30}
    ~ & EPN-reg & 89.52 ± 2.34 & 16.84 ± 3.37 & 5.54 ± 0.81 & 86.66 ± 1.22 & 47.93 ± 6.26 \\ \hline
    \end{tabular}
   }
    \caption{Missclassification detection results on CoraML, CiteSeer, PubMed, and AmazonPhotos (\textcolor{red}{best} and \textcolor{blue}{runner-up}).}
    \label{tab:mis_1}
\end{table}

\begin{table}[ht]
    \centering
    \small
    \setlength{\tabcolsep}{10pt}
    \resizebox{0.8\textwidth}{!}{%
    \begin{tabular}{@{}p{1cm}l|ccc|cc@{}}
    \hline
    \multirow{2}{*}{\textbf{Dataset}} & \multirow{2}{*}{\textbf{Model}} & \multicolumn{5}{c}{\textbf{Clean Graph}} \\ 
    ~ & ~ & clean-acc$\uparrow$ & BS$\downarrow$ & ECE$\downarrow$ & MIS\_ROC$\uparrow$ & MIS\_PR$\uparrow$  \\ \hline
    \multirow{14}{*}{\rotatebox[origin=c]{90}{\textbf{AmazonComputers}}\hspace{0.5cm}} 
    ~ & VGNN-entropy & 80.59 ± 1.57 & 30.16 ± 1.47 & \textcolor{red}{6.55 ± 1.81} & 73.21 ± 1.86 & 37.31 ± 2.54 \\
    
    ~ & VGNN-max-score & 80.59 ± 1.57 & 30.16 ± 1.47 & \textcolor{red}{6.55 ± 1.81} & 77.53 ± 1.97 & 43.59 ± 3.02 \\
    
    ~ & VGNN-energy & 80.59 ± 1.57 & 30.16 ± 1.47 & \textcolor{red}{6.55 ± 1.81} & 65.27 ± 1.40 & 32.62 ± 1.83 \\
    
    ~ & VGNN-gnnsafe & 80.59 ± 1.57 & 30.16 ± 1.47 & \textcolor{red}{6.55 ± 1.81} & 73.22 ± 2.20 & 43.16 ± 3.97 \\
    
    ~ & VGNN-dropout & 82.50 ± 1.49 & \textcolor{red}{28.05 ± 1.93} & 8.11 ± 1.73 & 73.63 ± 2.46 & 35.15 ± 3.49 \\
    
    ~ & VGNN-ensemble & \textcolor{blue}{82.79 ± 1.40} & \textcolor{blue}{28.18 ± 1.65} & 8.17 ± 2.08 & 72.31 ± 2.58 & 33.95 ± 3.67 \\
    
    ~ & GPN & 79.97 ± 1.77 & 34.17 ± 2.15 & 17.38 ± 2.51 & 77.16 ± 1.83 & 44.24 ± 3.97 \\
    
    ~ & SGNN-GKDE & 19.83 ± 14.71 & 89.86 ± 0.11 & 14.27 ± 10.07 & 44.89 ± 10.84 & \textcolor{red}{76.72 ± 16.86} \\
    
    ~ & EGNN & 81.93 ± 1.75 & 29.31 ± 2.13 & 9.93 ± 1.84 & 77.46 ± 1.92 & 42.11 ± 4.48 \\
    
    ~ & EGNN-vacuity-prop & \textcolor{red}{83.49 ± 1.45} & 28.48 ± 2.37 & 11.65 ± 2.07 & 75.96 ± 2.86 & 38.64 ± 3.78 \\
    
    ~ & EGNN-evidence-prop & 67.12 ± 3.60 & 54.34 ± 3.28 & 24.55 ± 3.31 & \textcolor{red}{80.27 ± 2.24} & \textcolor{blue}{62.96 ± 5.87} \\
    
    ~ & EGNN-vacuity-evidence-prop & 66.21 ± 4.61 & 55.41 ± 3.52 & 24.74 ± 3.87 & 78.83 ± 4.57 & 62.29 ± 9.29 \\
    
    ~ & EPN & 80.72 ± 1.64 & 31.02 ± 2.01 & 8.38 ± 2.04 & 75.49 ± 1.35 & 41.67 ± 3.05 \\
    \rowcolor{gray!30}
    ~ & EPN-reg & 81.72 ± 1.53 & 29.36 ± 2.39 & \textcolor{blue}{7.94 ± 1.67} & \textcolor{blue}{78.80 ± 1.12} & 43.67 ± 2.63 \\
    \hline
    \multirow{14}{*}{\rotatebox[origin=c]{90}{\textbf{CoauthorCS}}\hspace{0.5cm}} 
    ~ & VGNN-entropy & 91.90 ± 0.32 & \textcolor{blue}{13.01 ± 0.44} & \textcolor{red}{5.27 ± 0.51} & 85.19 ± 0.60 & 32.61 ± 1.66 \\
    
    ~ & VGNN-max-score & 91.90 ± 0.32 & \textcolor{blue}{13.01 ± 0.44} & \textcolor{red}{5.27 ± 0.51} & \textcolor{blue}{88.13 ± 0.61} & 39.89 ± 1.90 \\
    
    ~ & VGNN-energy & 91.90 ± 0.32 & \textcolor{blue}{13.01 ± 0.44} & \textcolor{red}{5.27 ± 0.51} & 75.92 ± 1.05 & 25.93 ± 1.51 \\
    
    ~ & VGNN-gnnsafe & 91.90 ± 0.32 & \textcolor{blue}{13.01 ± 0.44} & \textcolor{red}{5.27 ± 0.51} & 74.51 ± 1.36 & 20.73 ± 1.32 \\
    
    ~ & VGNN-dropout & 91.62 ± 1.04 & 13.69 ± 1.24 & 6.15 ± 0.87 & 83.59 ± 1.07 & 30.28 ± 1.86 \\
    
    ~ & VGNN-ensemble & \textcolor{red}{92.33 ± 0.43} & \textcolor{red}{12.60 ± 0.59} & 5.71 ± 0.46 & 84.87 ± 0.39 & 30.96 ± 1.65 \\
    
    ~ & GPN & 85.30 ± 1.49 & 28.38 ± 1.71 & 21.59 ± 1.74 & 83.76 ± 1.50 & \textcolor{blue}{45.72 ± 3.62} \\
    
    ~ & SGNN-GKDE & 14.30 ± 10.23 & 93.27 ± 0.04 & 9.07 ± 8.88 & 57.04 ± 12.29 & \textcolor{red}{90.02 ± 6.32} \\
    
    ~ & EGNN & \textcolor{blue}{92.07 ± 0.45} & 13.75 ± 0.77 & 8.68 ± 1.19 & 87.37 ± 0.55 & 37.96 ± 2.71 \\
    
    ~ & EGNN-vacuity-prop & 91.49 ± 0.67 & 14.55 ± 0.96 & 8.78 ± 0.69 & 87.23 ± 1.07 & 39.38 ± 1.95 \\
    
    ~ & EGNN-evidence-prop & 85.13 ± 0.71 & 32.37 ± 0.79 & 28.01 ± 1.04 & 81.82 ± 0.84 & 42.61 ± 1.71 \\
    
    ~ & EGNN-vacuity-evidence-prop & 85.09 ± 0.67 & 32.07 ± 0.57 & 27.72 ± 1.30 & 82.23 ± 1.12 & 43.28 ± 2.26 \\
    
    ~ & EPN & 91.74 ± 0.09 & 13.65 ± 0.06 & 7.20 ± 0.34 & \textcolor{red}{88.28 ± 0.21} & 40.42 ± 0.70 \\
    \rowcolor{gray!30}
    ~ & EPN-reg & 91.71 ± 0.39 & 13.23 ± 0.62 & \textcolor{blue}{5.46 ± 0.83} & 87.87 ± 0.92 & 40.61 ± 2.53 \\
    \hline
    \multirow{14}{*}{\rotatebox[origin=c]{90}{\textbf{CoauthorPhysics}}\hspace{0.5cm}} 
    ~ & VGNN-entropy & 92.81 ± 1.13 & 11.45 ± 1.61 & 3.89 ± 0.64 & 87.47 ± 0.94 & 32.10 ± 3.24 \\
    
    ~ & VGNN-max-score & 92.81 ± 1.13 & 11.45 ± 1.61 & 3.89 ± 0.64 & 88.66 ± 1.06 & 35.99 ± 2.50 \\
    
    ~ & VGNN-energy & 92.81 ± 1.13 & 11.45 ± 1.61 & 3.89 ± 0.64 & 83.10 ± 1.38 & 26.93 ± 4.04 \\
    
    ~ & VGNN-gnnsafe & 92.81 ± 1.13 & 11.45 ± 1.61 & 3.89 ± 0.64 & 84.22 ± 0.87 & 26.53 ± 3.66 \\
    
    ~ & VGNN-dropout & 93.30 ± 0.50 & 10.55 ± 0.69 & 3.31 ± 0.67 & 88.49 ± 1.33 & 33.32 ± 2.87 \\
    
    ~ & VGNN-ensemble & \textcolor{red}{93.83 ± 0.35} & \textcolor{red}{9.58 ± 0.51} & \textcolor{blue}{3.08 ± 0.20} & 89.57 ± 1.04 & 35.82 ± 1.50 \\
    
    ~ & GPN & 92.17 ± 0.79 & 14.20 ± 1.15 & 12.08 ± 0.98 & 88.89 ± 1.30 & 39.96 ± 1.97 \\
    
    ~ & SGNN-GKDE & 93.14 ± 0.72 & 31.13 ± 1.95 & 38.26 ± 2.07 & 87.07 ± 1.16 & 35.79 ± 3.37 \\
    
    ~ & EGNN & 92.96 ± 0.72 & 11.09 ± 1.43 & 3.66 ± 1.40 & 89.42 ± 1.10 & 38.38 ± 1.68 \\
    
    ~ & EGNN-vacuity-prop & 93.07 ± 0.68 & 11.09 ± 1.12 & 3.81 ± 1.21 & 88.70 ± 0.96 & 36.67 ± 1.64 \\
    
    ~ & EGNN-evidence-prop & 91.79 ± 0.43 & 15.26 ± 0.61 & 13.50 ± 0.83 & 88.83 ± 0.48 & \textcolor{blue}{40.60 ± 1.89} \\
    
    ~ & EGNN-vacuity-evidence-prop & 91.65 ± 0.35 & 15.50 ± 0.61 & 13.47 ± 0.86 & 88.45 ± 1.11 & \textcolor{red}{40.95 ± 1.96} \\
    
    ~ & EPN & 93.40 ± 0.07 & 9.98 ± 0.06 & \textcolor{red}{2.59 ± 0.40} & \textcolor{blue}{89.82 ± 0.15} & \textcolor{red}{40.81 ± 0.56} \\
    \rowcolor{gray!30}
    ~ & EPN-reg & \textcolor{blue}{93.64 ± 0.67} & \textcolor{blue}{9.73 ± 0.81} & 3.18 ± 0.70 & \textcolor{red}{90.97 ± 0.63} & 40.10 ± 3.67 \\
    \hline
    \multirow{14}{*}{\rotatebox[origin=c]{90}{\textbf{ogbn-arxiv}}\hspace{0.5cm}} 
    ~ & VGNN-entropy & \textcolor{blue}{72.32 ± 0.22} & \textcolor{blue}{39.61 ± 0.22} & \textcolor{red}{2.79 ± 0.18} & 75.69 ± 0.17 & 50.48 ± 0.36 \\
    
    ~ & VGNN-max-score & \textcolor{blue}{72.32 ± 0.22} & \textcolor{blue}{39.61 ± 0.22} & \textcolor{red}{2.79 ± 0.18} & \textcolor{red}{77.58 ± 0.13} & 53.98 ± 0.37 \\
    
    ~ & VGNN-energy & \textcolor{blue}{72.32 ± 0.22} & \textcolor{blue}{39.61 ± 0.22} & \textcolor{red}{2.79 ± 0.18} & 68.86 ± 0.40 & 43.96 ± 0.58 \\
    
    ~ & VGNN-gnnsafe & \textcolor{blue}{72.32 ± 0.22} & \textcolor{blue}{39.61 ± 0.22} & \textcolor{red}{2.79 ± 0.18} & 60.91 ± 0.22 & 38.72 ± 0.28 \\
    
    ~ & VGNN-dropout & 72.16 ± 0.19 & 39.76 ± 0.18 & \textcolor{blue}{2.83 ± 0.16} & 75.70 ± 0.16 & 50.69 ± 0.29 \\
    
    ~ & VGNN-ensemble & \textcolor{red}{72.73 ± 0.06} & \textcolor{red}{39.04 ± 0.06} & 2.83 ± 0.08 & 75.57 ± 0.07 & 49.62 ± 0.17 \\
    
    ~ & GPN & 68.92 ± 0.47 & 44.68 ± 0.35 & 8.53 ± 0.56 & 75.93 ± 0.02 & 56.13 ± 0.32 \\
    
    ~ & SGNN-GKDE & n.a & n.a & n.a & n.a & n.a \\
    
    ~ & EGNN & 69.17 ± 0.78 & 50.63 ± 0.59 & 25.73 ± 1.69 & 76.92 ± 0.43 & 56.68 ± 0.47 \\
    
    ~ & EGNN-vacuity-prop & 69.21 ± 0.71 & 50.39 ± 0.65 & 25.38 ± 1.33 & \textcolor{blue}{77.06 ± 0.43} & 56.85 ± 0.49 \\
    
    ~ & EGNN-evidence-prop & 63.90 ± 1.89 & 59.40 ± 0.94 & 28.83 ± 2.69 & 75.68 ± 0.96 & \textcolor{red}{60.75 ± 1.69} \\
    
    ~ & EGNN-vacuity-evidence-prop & 64.86 ± 2.29 & 59.42 ± 0.75 & 30.22 ± 3.23 & 75.63 ± 0.68 & \textcolor{blue}{60.19 ± 1.69} \\
    
    ~ & EPN & 69.85 ± 0.14 & 42.94 ± 0.19 & 4.80 ± 0.18 & 76.80 ± 0.07 & 55.69 ± 0.17 \\
    \rowcolor{gray!30}
    ~ & EPN-reg & 69.87 ± 0.18 & 42.86 ± 0.21 & 4.67 ± 0.19 & 76.88 ± 0.08 & 55.82 ± 0.18 \\
    \hline
    \end{tabular}
    }
    \caption{Missclassification detection results on AmazonComputers, CoauthorCS, CoauthorPhysics, and OGBN-arxiv (\textcolor{red}{best} and \textcolor{blue}{runner-up}).}
\label{tab:mis_2}
\end{table}

\subsection{OOD detection}

We present the ROC and PR metrics for the OOD detection task on the CoraML, CiteSeer, PubMed, and OGBN-Arxiv datasets in Table~\ref{tab:ood_gcn_1} and Table~\ref{tab:ood_gcn_2}. The corresponding results for the Amazon and Coauthor datasets are provided in Table~\ref{tab:ood_gcn_3} and Table~\ref{tab:ood_gcn_4}. Additionally, we summarize the average performance ranks of the models across all datasets in Table~\ref{tab:ood_gcn_rank} (in the main text). These rankings offer a clearer understanding of each model’s relative performance across diverse datasets and metrics, highlighting the consistency and robustness of the models in OOD detection.

\begin{table}[!ht]
    \centering
    \small
    \setlength{\tabcolsep}{8pt}
    \resizebox{0.98\textwidth}{!}{%
    \begin{tabular}{@{}p{1.2cm}l|cc|cc|cc@{}}
    \toprule
    \multirow{2}{*}{\textbf{Dataset}} & \multirow{2}{*}{\textbf{Model}} & \multicolumn{2}{c|}{\textbf{OS-1 (last)}} & \multicolumn{2}{c|}{\textbf{OS-2 (first)}} & \multicolumn{2}{c}{\textbf{OS-3 (random)}} \\ 
    ~ & ~ & OOD-AUROC$\uparrow$ & OOD-AUPR$\uparrow$ & OOD-AUROC$\uparrow$ & OOD-AUPR$\uparrow$ & OOD-AUROC$\uparrow$ & OOD-AUPR$\uparrow$  \\ \midrule
    \multirow{17}{*}{\rotatebox[origin=c]{90}{\textbf{CoraML}}\hspace{0.5cm}} 
    ~ & \multicolumn{6}{c}{\textbf{logit based}} \\ \cline{2-8}
    ~ & VGCN-entropy & 85.97 ± 1.51 & 81.47 ± 2.04 & 83.93 ± 2.08 & 76.19 ± 4.30 & 86.66 ± 1.21 & 86.09 ± 1.16 \\
    
    ~ & VGCN-max-score & 85.17 ± 1.47 & 80.15 ± 1.85 & 82.78 ± 2.38 & 73.80 ± 5.22 & 86.20 ± 1.30 & 85.74 ± 1.11 \\
    
    ~ & VGCN-energy & 86.14 ± 1.51 & 81.73 ± 2.42 & 84.45 ± 2.39 & 77.27 ± 4.66 & 86.45 ± 1.72 & 85.77 ± 2.33 \\
    
    ~ & VGCN-gnnsafe & 89.04 ± 1.17 & 84.49 ± 2.10 & 87.36 ± 3.60 & 78.91 ± 7.14 & 86.18 ± 2.53 & 85.73 ± 2.63 \\
    
    ~ & VGCN-dropout & 87.38 ± 1.05 & 83.60 ± 2.36 & 81.18 ± 3.52 & 71.42 ± 3.14 & 88.32 ± 1.63 & 88.44 ± 2.17 \\
    
    ~ & VGCN-ensemble & 88.02 ± 1.42 & 84.30 ± 2.29 & 82.37 ± 1.37 & 71.79 ± 3.76 & 89.67 ± 0.81 & 90.01 ± 1.38 \\ 
    \cline{2-8} ~ & \multicolumn{6}{c}{\textbf{evidential based}} \\ \cline{2-8}
    ~ & GPN & 86.83 ± 1.97 & 81.09 ± 2.60 & 79.37 ± 2.47 & 63.43 ± 3.72 & 88.41 ± 2.08 & 88.85 ± 2.29 \\
    
    ~ & SGCN-GKDE & \textcolor{blue}{89.86 ± 1.44} & \textcolor{red}{87.44 ± 2.52} & \textcolor{blue}{87.88 ± 2.36} & \textcolor{red}{81.58 ± 4.02} & \textcolor{blue}{90.18 ± 1.20} & \textcolor{blue}{91.14 ± 1.61} \\
    
    ~ & EGCN & 83.04 ± 2.29 & 78.88 ± 2.61 & 82.18 ± 1.85 & 74.39 ± 2.82 & 84.06 ± 2.13 & 85.97 ± 2.20 \\
    
    ~ & EGCN-vacuity-prop & 89.71 ± 1.17 & \textcolor{blue}{86.03 ± 2.11} & \textcolor{red}{89.61 ± 3.09} & \textcolor{blue}{81.47 ± 6.22} & \textcolor{red}{91.09 ± 1.39} & \textcolor{red}{91.82 ± 1.35} \\
    
    ~ & EGCN-evidence-prop & 87.36 ± 2.13 & 82.43 ± 3.31 & 82.56 ± 3.23 & 69.64 ± 5.07 & 89.19 ± 1.30 & 89.10 ± 1.01 \\

    ~ & EGCN-vacuity-evidence-prop & 87.23 ± 1.24 & 80.07 ± 1.35 & 87.00 ± 1.77 & 72.80 ± 3.85 & 89.63 ± 1.41 & 88.59 ± 1.78 \\
    
    \cline{2-8} ~ & \multicolumn{6}{c}{\textbf{ours}} \\ \cline{2-8}
    ~ & EPN & 88.06 ± 2.62 & 84.38 ± 4.18 & 87.00 ± 3.30 & 78.15 ± 5.64 & 88.11 ± 2.55 & 88.21 ± 2.99 \\
    \rowcolor{gray!30}
    ~ & EPN-reg & \textcolor{red}{89.97 ± 2.48} & 86.01 ± 4.82 & 85.91 ± 6.18 & 75.53 ± 11.08 & 88.96 ± 1.46 & 89.26 ± 1.74 \\
    \midrule
    \multirow{17}{*}{\rotatebox[origin=c]{90}{\textbf{CiteSeer}}\hspace{0.5cm}} 
    ~ & \multicolumn{6}{c}{\textbf{logit based}} \\ \cline{2-8}
    ~ & VGCN-entropy & 86.17 ± 1.47 & 70.48 ± 1.68 & 88.15 ± 1.97 & 75.75 ± 3.24 & 82.05 ± 4.40 & 68.53 ± 6.34 \\
    
    ~ & VGCN-max-score & 85.75 ± 1.57 & 68.70 ± 1.53 & 87.64 ± 2.04 & 74.02 ± 3.49 & 81.85 ± 4.36 & 67.91 ± 6.45 \\
    
    ~ & VGCN-energy & 86.55 ± 1.57 & 69.68 ± 2.84 & 88.91 ± 1.93 & 77.66 ± 3.27 & 82.17 ± 4.71 & 68.04 ± 6.73 \\
    
    ~ & VGCN-gnnsafe & \textcolor{blue}{88.94 ± 1.64} & \textcolor{blue}{71.97 ± 4.01} & \textcolor{red}{91.90 ± 1.55} & \textcolor{blue}{80.78 ± 3.38} & 84.00 ± 4.95 & 68.80 ± 7.67 \\
    
    ~ & VGCN-dropout & 86.14 ± 3.12 & 68.88 ± 5.49 & 88.20 ± 0.96 & 75.23 ± 2.47 & 81.90 ± 4.93 & 66.16 ± 6.89 \\
    
    ~ & VGCN-ensemble & 82.68 ± 1.66 & 63.18 ± 3.13 & 90.16 ± 0.70 & 79.80 ± 1.97 & 81.54 ± 4.41 & 66.96 ± 4.55 \\
    \cline{2-8} ~ & \multicolumn{6}{c}{\textbf{evidential based}} \\ \cline{2-8}
    ~ & GPN & 85.96 ± 2.58 & 64.70 ± 5.11 & 88.95 ± 1.88 & 76.96 ± 3.13 & 69.78 ± 11.16 & 54.10 ± 9.67 \\
    
    ~ & SGCN-GKDE & \textcolor{red}{90.40 ± 2.72} & \textcolor{red}{78.43 ± 6.09} & 89.49 ± 1.80 & 80.27 ± 3.70 & \textcolor{blue}{86.98 ± 4.10} & \textcolor{red}{77.45 ± 7.66} \\
    
    ~ & EGCN & 85.22 ± 2.05 & 70.49 ± 3.89 & 88.34 ± 1.02 & 76.64 ± 3.26 & 85.19 ± 1.50 & \textcolor{blue}{73.18 ± 3.48} \\
    
    ~ & EGCN-vacuity-prop & 88.02 ± 2.46 & 66.55 ± 3.57 & \textcolor{blue}{91.58 ± 1.10} & \textcolor{red}{80.90 ± 2.63} & 85.72 ± 1.85 & 67.89 ± 3.98 \\
    
    ~ & EGCN-evidence-prop & 86.70 ± 3.94 & 69.45 ± 7.23 & 88.56 ± 1.93 & 77.06 ± 2.94 & 81.98 ± 5.24 & 66.12 ± 6.81 \\
    
    ~ & EGCN-vacuity-evidence-prop & 87.08 ± 2.92 & 65.28 ± 6.60 & 89.57 ± 3.40 & 78.12 ± 4.26 & 80.70 ± 7.37 & 62.56 ± 8.83 \\
    \cline{2-8} ~ & \multicolumn{6}{c}{\textbf{ours}} \\ \cline{2-8}
    ~ & EPN & 84.02 ± 4.22 & 65.09 ± 7.29 & 87.91 ± 2.61 & 76.95 ± 4.91 & 83.31 ± 4.11 & 67.48 ± 6.39 \\
    \rowcolor{gray!30}
    ~ & EPN-reg & 88.23 ± 2.79 & 69.69 ± 4.97 & 90.86 ± 1.51 & 79.18 ± 4.27 & \textcolor{red}{87.27 ± 2.52} & 73.24 ± 5.35 \\
    \midrule
    \multirow{17}{*}{\rotatebox[origin=c]{90}{\textbf{PubMed}}\hspace{0.5cm}} 
    ~ & \multicolumn{6}{c}{\textbf{logit based}} \\ \cline{2-8}
    ~ & VGCN-entropy & 66.60 ± 1.66 & 54.77 ± 1.60 & 63.81 ± 4.21 & 29.38 ± 3.88 & 49.96 ± 5.38 & 39.51 ± 3.59 \\
    
    ~ & VGCN-max-score & 66.60 ± 1.66 & 54.77 ± 1.60 & 63.81 ± 4.21 & 29.37 ± 3.88 & 49.96 ± 5.38 & 39.50 ± 3.59 \\
    
    ~ & VGCN-energy & 66.47 ± 1.70 & 54.72 ± 1.80 & 64.02 ± 4.57 & 29.56 ± 4.44 & 49.78 ± 5.73 & 39.44 ± 3.89 \\
    
    ~ & VGCN-gnnsafe & 67.28 ± 2.00 & 54.92 ± 1.87 & 67.94 ± 4.59 & 34.60 ± 6.54 & 48.53 ± 7.14 & 37.92 ± 4.70 \\
    
    ~ & VGCN-dropout & 64.41 ± 1.44 & 52.30 ± 1.46 & 62.71 ± 2.86 & 28.11 ± 2.67 & 51.00 ± 2.73 & 40.04 ± 1.64 \\
    
    ~ & VGCN-ensemble & 67.96 ± 2.12 & 55.43 ± 2.12 & 63.47 ± 6.30 & 29.33 ± 5.10 & 53.68 ± 10.10 & 42.99 ± 7.69 \\
    \cline{2-8} ~ & \multicolumn{6}{c}{\textbf{evidential based}} \\ \cline{2-8}
    ~ & GPN & 67.53 ± 4.34 & 58.93 ± 5.55 & 65.11 ± 3.68 & 35.70 ± 6.00 & 54.75 ± 9.11 & 43.29 ± 8.09 \\
    
    ~ & SGCN-GKDE & 68.69 ± 2.60 & 60.80 ± 4.18 & 60.55 ± 8.84 & 32.41 ± 8.87 & \textcolor{blue}{61.74 ± 5.88} & \textcolor{red}{51.12 ± 6.84} \\
    
    ~ & EGCN & 65.10 ± 1.67 & 58.29 ± 2.43 & 68.96 ± 1.23 & 38.02 ± 2.04 & 49.23 ± 6.54 & 38.57 ± 4.97 \\
    
    ~ & EGCN-vacuity-prop & 71.90 ± 4.39 & 63.83 ± 5.46 & \textcolor{red}{76.76 ± 3.68} & \textcolor{red}{48.42 ± 5.22} & 52.64 ± 8.07 & 39.94 ± 5.65 \\
    
    ~ & EGCN-evidence-prop & \textcolor{blue}{75.44 ± 3.54} & \textcolor{red}{68.46 ± 4.21} & 72.49 ± 13.85 & \textcolor{blue}{48.11 ± 16.03} & 55.69 ± 4.53 & 42.76 ± 4.29 \\
    
    ~ & EGCN-vacuity-evidence-prop & \textcolor{red}{76.36 ± 4.95} & \textcolor{blue}{68.13 ± 6.66} & \textcolor{blue}{76.63 ± 7.30} & 47.46 ± 12.81 & \textcolor{red}{65.19 ± 10.53} & \textcolor{blue}{50.80 ± 9.30} \\
    \cline{2-8} ~ & \multicolumn{6}{c}{\textbf{ours}} \\ \cline{2-8}
    ~ & EPN & 66.38 ± 0.77 & 54.87 ± 1.07 & 51.01 ± 17.29 & 24.32 ± 9.50 & 58.27 ± 7.21 & 46.10 ± 6.31 \\
    \rowcolor{gray!30}
    ~ & EPN-reg & 67.38 ± 3.85 & 53.66 ± 3.72 & 65.25 ± 7.45 & 33.01 ± 6.89 & 53.65 ± 6.11 & 41.47 ± 4.14 \\
    \midrule
    \multirow{17}{*}{\rotatebox[origin=c]{90}{\textbf{OGBN-arxiv}}\hspace{0.5cm}} 
    ~ & \multicolumn{6}{c}{\textbf{logit based}} \\ \cline{2-8}
    ~ & VGCN-entropy & 68.05 ± 0.60 & 46.59 ± 0.73 & 75.82 ± 0.59 & 48.21 ± 1.02 & 77.71 ± 0.46 & 71.28 ± 0.57 \\
    
    ~ & VGCN-max-score & 66.34 ± 0.55 & 44.89 ± 0.64 & 72.89 ± 0.60 & 42.86 ± 0.94 & 76.01 ± 0.45 & 69.29 ± 0.52 \\
    
    ~ & VGCN-energy & 67.42 ± 0.93 & 45.56 ± 0.97 & 76.00 ± 0.91 & 45.74 ± 1.53 & 79.51 ± 0.55 & 73.22 ± 0.65 \\
    
    ~ & VGCN-gnnsafe & 69.23 ± 0.60 & 50.37 ± 0.53 & 80.07 ± 0.21 & 57.18 ± 0.66 & 68.53 ± 1.05 & 64.05 ± 1.01 \\
    
    ~ & VGCN-dropout & 68.17 ± 0.70 & 46.72 ± 0.80 & 75.93 ± 0.53 & 48.55 ± 0.87 & 77.86 ± 0.36 & 71.52 ± 0.55 \\
    
    ~ & VGCN-ensemble & 68.46 ± 0.20 & 46.72 ± 0.21 & 76.70 ± 0.15 & 49.00 ± 0.22 & 78.15 ± 0.18 & 71.78 ± 0.23 \\
    \cline{2-8} ~ & \multicolumn{6}{c}{\textbf{evidential based}} \\ \cline{2-8}
    ~ & GPN & 74.67 ± 0.37 & 57.97 ± 0.36 & 75.84 ± 0.63 & 49.80 ± 0.91 & 81.26 ± 0.28 & 74.67 ± 0.22 \\
    
    ~ & SGCN-GKDE & n.a & n.a & n.a & n.a & n.a & n.a \\
    
    ~ & EGCN & 75.93 ± 1.04 & 58.68 ± 1.72 & 77.42 ± 1.82 & 51.26 ± 3.01 & 82.88 ± 0.92 & \textcolor{blue}{77.53 ± 1.36} \\
    
    ~ & EGCN-vacuity-prop & \textcolor{blue}{81.77 ± 1.33} & \textcolor{blue}{65.22 ± 2.07} & \textcolor{red}{83.76 ± 0.64} & \textcolor{red}{61.94 ± 1.00} & \textcolor{blue}{83.12 ± 1.65} & 77.42 ± 2.34 \\
    
    ~ & EGCN-evidence-prop & 70.19 ± 1.65 & 49.98 ± 2.38 & 75.68 ± 1.57 & 47.20 ± 2.94 & 76.28 ± 1.34 & 68.17 ± 1.37 \\
    
    ~ & EGCN-vacuity-evidence-prop & 78.28 ± 1.08 & 60.98 ± 1.52 & \textcolor{blue}{82.56 ± 0.97} & 59.88 ± 1.64 & 77.17 ± 1.75 & 68.91 ± 1.91 \\
    \cline{2-8} ~ & \multicolumn{6}{c}{\textbf{ours}} \\ \cline{2-8}
    ~ & EPN & 67.12 ± 1.30 & 44.01 ± 1.74 & 79.25 ± 0.78 & 48.42 ± 1.46 & 67.19 ± 1.90 & 55.30 ± 2.04 \\
    \rowcolor{gray!30}
    ~ & EPN-reg & \textcolor{red}{83.99 ± 0.33} & \textcolor{red}{72.87 ± 0.79} & 81.54 ± 0.35 & \textcolor{blue}{60.29 ± 2.50} & \textcolor{red}{85.47 ± 0.25} & \textcolor{red}{82.52 ± 0.30} \\
    \bottomrule
    \end{tabular}
    }
    \caption{OOD detection results (\textcolor{red}{best} and \textcolor{blue}{runner-up}) with GCN as backbone on CoraML, CiteSeer, PubMed and ogbn-arxiv for OS-1, OS-2 and OS-3.}
    \label{tab:ood_gcn_1}
\end{table}

\begin{table}[!ht]
    \centering
    \small
    \setlength{\tabcolsep}{8pt}
    \resizebox{0.98\textwidth}{!}{%
    \begin{tabular}{@{}p{1.2cm}l|cc|cc|cc@{}}
    \toprule
    \multirow{2}{*}{\textbf{Dataset}} & \multirow{2}{*}{\textbf{Model}} & \multicolumn{2}{c|}{\textbf{OS-1 (last)}} & \multicolumn{2}{c|}{\textbf{OS-2 (first)}} & \multicolumn{2}{c}{\textbf{OS-3 (random)}} \\ 
    ~ & ~ & OOD-AUROC$\uparrow$ & OOD-AUPR$\uparrow$ & OOD-AUROC$\uparrow$ & OOD-AUPR$\uparrow$ & OOD-AUROC$\uparrow$ & OOD-AUPR$\uparrow$  \\ \midrule
    \multirow{17}{*}{\rotatebox[origin=c]{90}{\textbf{AmazonPhotos}}\hspace{0.5cm}}
    ~ & \multicolumn{6}{c}{\textbf{logit based}} \\ \cline{2-8}
    ~ & VGCN-entropy & 79.00 ± 3.79 & 67.60 ± 4.47 & 79.74 ± 2.07 & 69.18 ± 4.45 & 81.57 ± 3.04 & 64.95 ± 4.09 \\
    
    ~ & VGCN-max-score & 78.65 ± 4.19 & 66.39 ± 4.81 & 81.10 ± 1.54 & 70.66 ± 2.96 & 82.65 ± 4.90 & 67.67 ± 6.72 \\
    
    ~ & VGCN-energy & 78.21 ± 2.70 & 67.85 ± 3.37 & 74.44 ± 6.24 & 67.40 ± 7.82 & 74.57 ± 2.28 & 56.39 ± 4.12 \\
    
    ~ & VGCN-gnnsafe & 82.14 ± 5.17 & 69.40 ± 5.36 & 85.03 ± 7.53 & 78.33 ± 7.47 & 85.63 ± 2.58 & 68.91 ± 3.16 \\
    
    ~ & VGCN-dropout & 78.08 ± 5.89 & 67.90 ± 7.68 & 80.02 ± 2.13 & 67.89 ± 2.44 & 83.32 ± 5.86 & 66.64 ± 6.39 \\
    
    ~ & VGCN-ensemble & 75.47 ± 5.46 & 63.51 ± 6.09 & 81.24 ± 4.18 & 69.20 ± 4.63 & 83.44 ± 3.51 & 65.24 ± 7.21 \\
    \cline{2-8} ~ & \multicolumn{6}{c}{\textbf{evidential based}} \\ \cline{2-8}
    ~ & GPN & \textcolor{red}{90.67 ± 1.79} & \textcolor{red}{84.70 ± 3.27} & \textcolor{red}{92.44 ± 1.37} & \textcolor{red}{88.88 ± 2.29} & \textcolor{blue}{87.60 ± 1.75} & \textcolor{blue}{76.03 ± 2.66} \\
    
    ~ & SGCN-GKDE & 80.22 ± 3.00 & 72.00 ± 4.98 & 80.20 ± 5.13 & 74.64 ± 7.47 & 78.80 ± 3.35 & 61.62 ± 5.31 \\
    
    ~ & EGCN & 72.49 ± 1.87 & 62.67 ± 2.47 & 64.18 ± 3.09 & 46.56 ± 3.24 & 68.20 ± 2.13 & 47.37 ± 2.65 \\
    
    ~ & EGCN-vacuity-prop & 82.75 ± 2.78 & 73.14 ± 3.84 & 71.98 ± 3.19 & 56.36 ± 4.48 & 77.19 ± 1.80 & 56.68 ± 3.26 \\
    
    ~ & EGCN-evidence-prop & 56.64 ± 2.85 & 53.59 ± 1.74 & 86.72 ± 5.85 & 77.66 ± 7.46 & 75.38 ± 2.21 & 57.43 ± 4.02 \\
    
    ~ & EGCN-vacuity-evidence-prop & 59.53 ± 6.04 & 58.40 ± 4.88 & \textcolor{blue}{91.95 ± 3.28} & 83.13 ± 8.90 & 81.36 ± 4.68 & 64.96 ± 6.89 \\
    \cline{2-8} ~ & \multicolumn{6}{c}{\textbf{ours}} \\ \cline{2-8}
    ~ & EPN & 84.68 ± 4.18 & 77.29 ± 4.33 & 82.66 ± 6.82 & 79.00 ± 7.57 & 70.51 ± 5.96 & 57.79 ± 5.69 \\
    \rowcolor{gray!30}
    ~ & EPN-reg & \textcolor{blue}{86.49 ± 5.40} & \textcolor{blue}{81.37 ± 6.67} & 88.79 ± 5.61 & \textcolor{blue}{85.96 ± 7.93} & \textcolor{red}{91.49 ± 3.74} & \textcolor{red}{84.33 ± 7.45} \\
    \midrule
    \multirow{17}{*}{\rotatebox[origin=c]{90}{\textbf{AmazonComputers}}\hspace{0.5cm}} 
    ~ & \multicolumn{6}{c}{\textbf{logit based}} \\ \cline{2-8}
    ~ & VGCN-entropy & 69.41 ± 2.69 & 44.82 ± 4.54 & 69.79 ± 5.37 & 85.19 ± 3.23 & 61.32 ± 5.91 & 68.55 ± 5.08 \\
    
    ~ & VGCN-max-score & 69.88 ± 3.50 & 46.02 ± 5.63 & 63.92 ± 5.70 & 80.63 ± 3.74 & 60.52 ± 6.95 & 68.72 ± 6.70 \\
    
    ~ & VGCN-energy & 65.73 ± 3.60 & 41.18 ± 5.03 & 75.99 ± 3.49 & 87.88 ± 2.40 & 61.37 ± 4.74 & 68.05 ± 3.90 \\
    
    ~ & VGCN-gnnsafe & 78.82 ± 2.84 & 53.48 ± 4.56 & 76.49 ± 6.64 & 86.08 ± 3.56 & 58.93 ± 7.40 & 66.77 ± 4.88 \\
    
    ~ & VGCN-dropout & 67.72 ± 6.89 & 45.59 ± 9.06 & 69.75 ± 5.16 & 84.86 ± 3.21 & 61.35 ± 8.26 & 69.80 ± 6.60 \\
    
    ~ & VGCN-ensemble & 70.78 ± 2.74 & 48.72 ± 3.79 & 76.01 ± 6.85 & 88.41 ± 4.36 & 60.14 ± 8.85 & 67.87 ± 7.18 \\
    \cline{2-8} ~ & \multicolumn{6}{c}{\textbf{evidential based}} \\ \cline{2-8}
    ~ & GPN & \textcolor{blue}{80.97 ± 3.98} & 57.59 ± 5.87 & \textcolor{red}{91.54 ± 2.33} & \textcolor{red}{95.08 ± 1.28} & \textcolor{red}{85.70 ± 2.31} & \textcolor{red}{87.17 ± 1.70} \\
    
    ~ & SGCN-GKDE & 69.08 ± 6.09 & 51.19 ± 10.92 & 84.07 ± 3.03 & 92.43 ± 1.81 & 62.31 ± 10.18 & 69.55 ± 8.64 \\
    
    ~ & EGCN & 59.46 ± 2.24 & 35.40 ± 2.00 & 69.70 ± 2.74 & 83.47 ± 1.75 & 60.36 ± 2.94 & 67.53 ± 2.38 \\
    
    ~ & EGCN-vacuity-prop & 67.88 ± 5.72 & 43.17 ± 6.64 & 78.33 ± 4.22 & 88.22 ± 2.99 & 64.88 ± 4.53 & 71.86 ± 3.57 \\
    
    ~ & EGCN-evidence-prop & 75.15 ± 3.79 & 50.08 ± 4.30 & 45.44 ± 6.82 & 69.73 ± 3.03 & 66.75 ± 4.99 & 71.25 ± 3.72 \\
    
    ~ & EGCN-vacuity-evidence-prop & 80.67 ± 7.16 & \textcolor{blue}{58.23 ± 8.32} & 40.63 ± 2.49 & 68.78 ± 1.13 & 61.07 ± 11.19 & 66.21 ± 8.85 \\
    \cline{2-8} ~ & \multicolumn{6}{c}{\textbf{ours}} \\ \cline{2-8}
    ~ & EPN & 73.50 ± 3.59 & 49.92 ± 5.21 & \textcolor{blue}{87.00 ± 2.97} & \textcolor{blue}{93.92 ± 1.81} & 55.40 ± 4.66 & 64.70 ± 2.87 \\
    \rowcolor{gray!30}
    ~ & EPN-reg & \textcolor{red}{83.26 ± 6.06} & \textcolor{red}{68.91 ± 8.41} & 81.99 ± 9.82 & 91.49 ± 4.83 & \textcolor{blue}{70.96 ± 9.16} & \textcolor{blue}{79.18 ± 8.00} \\
    \midrule
    \multirow{17}{*}{\rotatebox[origin=c]{90}{\textbf{CoauthorCS}}\hspace{0.5cm}}
    ~ & \multicolumn{6}{c}{\textbf{logit based}} \\ \cline{2-8}
    ~ & VGCN-entropy & 87.91 ± 1.36 & 82.35 ± 3.07 & 87.94 ± 2.38 & 63.41 ± 4.56 & 90.78 ± 1.80 & 66.43 ± 5.29 \\
    
    ~ & VGCN-max-score & 87.52 ± 1.52 & 81.62 ± 3.49 & 87.61 ± 2.29 & 63.20 ± 4.70 & 90.57 ± 1.69 & 66.29 ± 4.94 \\
    
    ~ & VGCN-energy & 86.84 ± 1.68 & 81.74 ± 3.49 & 85.97 ± 2.98 & 59.91 ± 5.61 & 88.61 ± 2.59 & 62.81 ± 6.29 \\
    
    ~ & VGCN-gnnsafe & 90.96 ± 1.18 & 88.49 ± 1.79 & \textcolor{blue}{92.73 ± 1.71} & 72.67 ± 5.27 & \textcolor{blue}{93.54 ± 1.71} & 74.31 ± 5.66 \\
    
    ~ & VGCN-dropout & 85.30 ± 3.44 & 79.66 ± 4.03 & 87.40 ± 3.09 & 63.03 ± 7.99 & 91.36 ± 1.96 & 68.18 ± 5.92 \\
    
    ~ & VGCN-ensemble & 87.05 ± 2.76 & 81.33 ± 3.79 & 87.33 ± 1.65 & 61.69 ± 4.85 & 88.14 ± 4.01 & 61.03 ± 8.73 \\
    \cline{2-8} ~ & \multicolumn{6}{c}{\textbf{evidential based}} \\ \cline{2-8}
    ~ & GPN & 89.67 ± 1.78 & 87.20 ± 2.06 & 84.34 ± 2.20 & 59.97 ± 4.38 & 83.29 ± 2.62 & 53.27 ± 4.74 \\
    
    ~ & SGCN-GKDE & 66.45 ± 2.68 & 52.00 ± 2.57 & 56.51 ± 4.71 & 24.88 ± 3.03 & 59.68 ± 2.52 & 23.81 ± 2.42 \\
    
    ~ & EGCN & 85.34 ± 1.99 & 80.89 ± 3.07 & 83.12 ± 5.92 & 58.61 ± 10.64 & 86.71 ± 2.18 & 60.21 ± 5.57 \\
    
    ~ & EGCN-vacuity-prop & 90.57 ± 1.65 & 87.85 ± 1.93 & \textcolor{red}{93.43 ± 1.85} & \textcolor{blue}{75.35 ± 5.35} & \textcolor{red}{94.31 ± 1.29} & \textcolor{blue}{75.46 ± 4.94} \\
    
    ~ & EGCN-evidence-prop & 89.52 ± 1.27 & 87.42 ± 1.32 & 82.42 ± 1.52 & 55.37 ± 3.16 & 80.51 ± 1.62 & 47.61 ± 2.23 \\
    
    ~ & EGCN-vacuity-evidence-prop & \textcolor{blue}{92.00 ± 1.75} & \textcolor{blue}{89.93 ± 1.49} & 86.92 ± 1.15 & 59.67 ± 2.26 & 86.18 ± 2.47 & 53.92 ± 4.65 \\
    \cline{2-8} ~ & \multicolumn{6}{c}{\textbf{ours}} \\ \cline{2-8}
    ~ & EPN & 82.30 ± 7.91 & 73.94 ± 10.26 & 80.18 ± 10.59 & 51.42 ± 14.39 & 79.00 ± 8.41 & 43.55 ± 10.65 \\
    \rowcolor{gray!30}
    ~ & EPN-reg & \textcolor{red}{95.09 ± 1.37} & \textcolor{red}{94.47 ± 1.29} & 91.03 ± 4.06 & \textcolor{red}{76.13 ± 9.48} & 93.30 ± 3.77 & \textcolor{red}{76.65 ± 12.16} \\
    \midrule
    \multirow{17}{*}{\rotatebox[origin=c]{90}{\textbf{CoauthorPhysics}}\hspace{0.5cm}} 
    ~ & \multicolumn{6}{c}{\textbf{logit based}} \\ \cline{2-8}
    ~ & VGCN-entropy & 92.09 ± 0.88 & 71.22 ± 2.14 & 80.97 ± 5.59 & 66.64 ± 8.38 & 85.38 ± 1.73 & 84.64 ± 1.95 \\
    
    ~ & VGCN-max-score & 91.48 ± 0.94 & 66.62 ± 1.60 & 80.73 ± 5.39 & 65.13 ± 7.97 & 85.77 ± 1.84 & 85.27 ± 1.98 \\
    
    ~ & VGCN-energy & 92.78 ± 0.93 & 76.26 ± 2.22 & 81.75 ± 5.60 & 69.05 ± 8.07 & 84.65 ± 2.16 & 84.84 ± 2.31 \\
    
    ~ & VGCN-gnnsafe & \textcolor{red}{95.60 ± 0.47} & \textcolor{red}{84.81 ± 1.41} & 88.18 ± 4.82 & 77.54 ± 7.66 & 89.50 ± 2.10 & 88.62 ± 2.14 \\
    
    ~ & VGCN-dropout & 92.52 ± 1.69 & 75.95 ± 5.99 & 84.53 ± 2.50 & 73.10 ± 4.57 & 85.30 ± 7.71 & 85.04 ± 7.19 \\
    
    ~ & VGCN-ensemble & 89.00 ± 3.47 & 67.67 ± 8.72 & 82.90 ± 3.65 & 70.26 ± 4.72 & 86.73 ± 1.62 & 85.26 ± 1.56 \\
    \cline{2-8} ~ & \multicolumn{6}{c}{\textbf{evidential based}} \\ \cline{2-8}
    ~ & GPN & 90.60 ± 2.40 & 75.05 ± 5.70 & 84.12 ± 12.22 & 77.11 ± 12.91 & \textcolor{blue}{94.31 ± 1.98} & 94.54 ± 2.17 \\
    
    ~ & SGCN-GKDE & 92.30 ± 2.52 & 76.94 ± 7.12 & 89.09 ± 4.92 & 81.87 ± 8.66 & 93.80 ± 2.55 & \textcolor{red}{95.16 ± 2.58} \\
    
    ~ & EGCN & 88.97 ± 1.73 & 65.52 ± 4.64 & 83.82 ± 4.60 & 70.24 ± 7.55 & 82.66 ± 2.38 & 84.04 ± 2.07 \\
    
    ~ & EGCN-vacuity-prop & 92.79 ± 2.76 & 77.43 ± 5.57 & \textcolor{blue}{94.91 ± 2.08} & \textcolor{blue}{89.23 ± 3.85} & 92.69 ± 2.40 & 91.91 ± 2.93 \\
    
    ~ & EGCN-evidence-prop & 90.80 ± 1.14 & 73.43 ± 2.12 & 91.81 ± 1.90 & 84.78 ± 2.97 & 93.48 ± 1.66 & 93.71 ± 1.85 \\
    
    ~ & EGCN-vacuity-evidence-prop & 93.97 ± 1.46 & 78.89 ± 3.41 & \textcolor{red}{95.59 ± 0.75} & \textcolor{red}{90.61 ± 1.27} & \textcolor{red}{95.62 ± 1.10} & \textcolor{blue}{95.04 ± 1.52} \\
    \cline{2-8} ~ & \multicolumn{6}{c}{\textbf{ours}} \\ \cline{2-8}
    ~ & EPN & \textcolor{blue}{94.10 ± 2.68} & 78.86 ± 6.02 & 88.18 ± 5.77 & 76.47 ± 9.14 & 83.43 ± 10.27 & 83.73 ± 9.45 \\
    \rowcolor{gray!30}
    ~ & EPN-reg & 93.59 ± 2.41 & \textcolor{blue}{79.31 ± 5.13} & 87.14 ± 7.33 & 78.99 ± 9.06 & 89.05 ± 9.34 & 88.37 ± 9.43 \\
    \bottomrule
    \end{tabular}
    }
    \caption{OOD detection results (\textcolor{red}{best} and \textcolor{blue}{runner-up}) with GCN as backbone on AmazonPhotos, AmazonComputers, CoauthorCS, and CoauthorPhysics for OS-1, OS-2 and OS-3.}
    \label{tab:ood_gcn_2}
\end{table}

\begin{table}[!ht]
    \centering
    \small
    \setlength{\tabcolsep}{20pt}
    \resizebox{0.98\textwidth}{!}{%
    \begin{tabular}{@{}p{2cm}l|cc|cc@{}}
    \toprule
    \multirow{2}{*}{\textbf{Dataset}} & \multirow{2}{*}{\textbf{Model}} & \multicolumn{2}{c|}{\textbf{OS-4 (random)}} & \multicolumn{2}{c}{\textbf{OS-5 (random)}} \\ 
    ~ & ~ & OOD-AUROC$\uparrow$ & OOD-AUPR$\uparrow$ & OOD-AUROC$\uparrow$ & OOD-AUPR$\uparrow$  \\ \midrule
    \multirow{17}{*}{\rotatebox[origin=c]{90}{\textbf{CoraML}}\hspace{0.5cm}} 
    ~ & \multicolumn{4}{c}{\textbf{logit based}} \\ \cline{2-6}
    ~ & VGCN-entropy & 79.61 ± 2.32 & 64.81 ± 4.42 & 88.71 ± 0.69 & 87.21 ± 0.95 \\
    
    ~ & VGCN-max-score & 79.05 ± 2.34 & 63.82 ± 4.35 & 87.73 ± 0.52 & 85.78 ± 0.55 \\
    
    ~ & VGCN-energy & 80.47 ± 2.28 & 66.52 ± 4.60 & 88.99 ± 1.00 & 87.64 ± 1.24 \\
    
    ~ & VGCN-gnnsafe & \textcolor{red}{87.50 ± 1.03} & \textcolor{red}{74.95 ± 3.53} & \textcolor{red}{93.48 ± 0.57} & \textcolor{red}{93.04 ± 0.84} \\
    
    ~ & VGCN-dropout & 75.97 ± 3.88 & 59.26 ± 5.66 & 87.12 ± 1.40 & 84.79 ± 1.47 \\
    
    ~ & VGCN-ensemble & 80.22 ± 2.30 & 63.74 ± 2.83 & 88.63 ± 1.51 & 87.16 ± 1.90 \\
    \cline{2-6} ~ & \multicolumn{4}{c}{\textbf{evidential based}} \\ \cline{2-6}
    ~ & GPN & 79.51 ± 2.24 & 60.55 ± 2.88 & 90.30 ± 1.94 & 90.23 ± 1.65 \\
    
    ~ & SGCN-GKDE & 81.29 ± 3.15 & 67.59 ± 5.30 & 89.36 ± 1.37 & 88.58 ± 2.26 \\
    
    ~ & EGCN & 75.17 ± 3.07 & 58.81 ± 3.85 & 83.77 ± 2.07 & 82.20 ± 2.31 \\
    
    ~ & EGCN-vacuity-prop & \textcolor{blue}{85.97 ± 2.43} & 72.33 ± 3.25 & \textcolor{blue}{92.82 ± 0.54} & \textcolor{blue}{92.68 ± 0.63} \\
    
    ~ & EGCN-evidence-prop & 80.79 ± 1.73 & 64.24 ± 2.29 & 90.22 ± 1.25 & 89.95 ± 1.44 \\
    
    ~ & EGCN-vacuity-evidence-prop & 84.84 ± 1.46 & 66.62 ± 2.26 & 91.61 ± 1.09 & 90.92 ± 1.15 \\
    \cline{2-6} ~ & \multicolumn{4}{c}{\textbf{ours}} \\ \cline{2-6}
    ~ & EPN & 82.75 ± 2.33 & 62.82 ± 3.12 & 86.99 ± 3.12 & 85.45 ± 3.75 \\
    \rowcolor{gray!30}
    ~ & EPN-reg & 85.46 ± 4.94 & \textcolor{blue}{73.14 ± 7.73} & 91.27 ± 2.73 & 89.88 ± 3.06 \\
    \midrule
    \multirow{17}{*}{\rotatebox[origin=c]{90}{\textbf{CiteSeer}}\hspace{0.5cm}}
    ~ & \multicolumn{4}{c}{\textbf{logit based}} \\ \cline{2-6}
    ~ & VGCN-entropy & 87.40 ± 2.12 & 75.36 ± 4.73 & 84.22 ± 3.89 & 68.74 ± 6.45 \\
    
    ~ & VGCN-max-score & 86.61 ± 2.13 & 73.71 ± 4.99 & 83.65 ± 3.87 & 67.14 ± 6.28 \\
    
    ~ & VGCN-energy & 87.51 ± 2.49 & 73.52 ± 5.92 & 84.74 ± 3.45 & 69.92 ± 6.03 \\
    
    ~ & VGCN-gnnsafe & 87.90 ± 2.61 & 74.69 ± 6.59 & 88.30 ± 2.72 & 73.16 ± 5.33 \\
    
    ~ & VGCN-dropout & 87.90 ± 1.10 & 75.44 ± 1.79 & 88.29 ± 1.90 & 74.27 ± 4.61 \\
    
    ~ & VGCN-ensemble & 84.70 ± 2.82 & 72.07 ± 5.43 & 87.13 ± 2.94 & 72.27 ± 4.84 \\
    \cline{2-6} ~ & \multicolumn{4}{c}{\textbf{evidential based}} \\ \cline{2-6}
    ~ & GPN & 77.92 ± 6.04 & 58.15 ± 7.38 & 90.10 ± 1.74 & 76.00 ± 3.64 \\
    
    ~ & SGCN-GKDE & \textcolor{red}{91.07 ± 1.76} & \textcolor{red}{82.70 ± 5.06} & 90.54 ± 3.49 & \textcolor{red}{81.98 ± 5.42} \\
    
    ~ & EGCN & 86.56 ± 1.69 & \textcolor{blue}{76.16 ± 2.53} & 86.72 ± 2.17 & 71.90 ± 3.33 \\
    
    ~ & EGCN-vacuity-prop & 86.32 ± 1.73 & 65.99 ± 3.26 & \textcolor{red}{91.48 ± 2.66} & 78.67 ± 5.19 \\
    
    ~ & EGCN-evidence-prop & 80.20 ± 5.61 & 61.65 ± 9.60 & \textcolor{blue}{91.27 ± 1.16} & \textcolor{blue}{80.83 ± 1.72} \\
    
    ~ & EGCN-vacuity-evidence-prop & 78.64 ± 3.63 & 56.89 ± 3.43 & 90.97 ± 1.10 & 78.29 ± 2.71 \\
    \cline{2-6} ~ & \multicolumn{4}{c}{\textbf{ours}} \\ \cline{2-6}
    ~ & EPN & 86.04 ± 2.80 & 67.58 ± 5.96 & 87.85 ± 2.43 & 74.05 ± 3.36 \\
    \rowcolor{gray!30}
    ~ & EPN-reg & \textcolor{blue}{88.78 ± 1.52} & 74.07 ± 4.81 & 90.53 ± 3.06 & 78.03 ± 5.92 \\
    \midrule
    \multirow{17}{*}{\rotatebox[origin=c]{90}{\textbf{PubMed}}\hspace{0.5cm}} 
    ~ & \multicolumn{4}{c}{\textbf{logit based}} \\ \cline{2-6}
    ~ & VGCN-entropy & 52.81 ± 13.09 & 23.43 ± 6.85 & 65.37 ± 3.92 & 53.70 ± 3.60 \\
    
    ~ & VGCN-max-score & 52.81 ± 13.09 & 23.42 ± 6.85 & 65.37 ± 3.92 & 53.70 ± 3.60 \\
    
    ~ & VGCN-energy & 53.01 ± 12.62 & 23.38 ± 6.63 & 65.34 ± 4.02 & 53.83 ± 4.13 \\
    
    ~ & VGCN-gnnsafe & 56.76 ± 13.97 & 26.98 ± 9.96 & 66.73 ± 4.41 & 54.36 ± 4.42 \\
    
    ~ & VGCN-dropout & 62.20 ± 3.44 & 27.54 ± 2.71 & 62.77 ± 2.32 & 51.30 ± 2.20 \\
    
    ~ & VGCN-ensemble & 61.94 ± 8.09 & 28.21 ± 5.10 & 64.46 ± 1.17 & 52.32 ± 1.00 \\
    \cline{2-6} ~ & \multicolumn{4}{c}{\textbf{evidential based}} \\ \cline{2-6}
    ~ & GPN & 68.27 ± 3.25 & 39.59 ± 4.83 & 65.52 ± 4.25 & 56.89 ± 4.42 \\
    
    ~ & SGCN-GKDE & 71.44 ± 3.32 & 45.26 ± 5.66 & 69.48 ± 5.00 & 63.02 ± 5.44 \\
    
    ~ & EGCN & 63.62 ± 2.63 & 31.16 ± 2.86 & 64.97 ± 3.53 & 58.35 ± 3.95 \\
    
    ~ & EGCN-vacuity-prop & 73.59 ± 4.41 & 43.24 ± 5.36 & 71.09 ± 6.16 & 62.29 ± 7.35 \\
    
    ~ & EGCN-evidence-prop & \textcolor{blue}{76.98 ± 3.32} & \textcolor{blue}{51.36 ± 7.36} & \textcolor{blue}{73.90 ± 2.89} & \textcolor{red}{67.22 ± 3.05} \\
    
    ~ & EGCN-vacuity-evidence-prop & \textcolor{red}{79.04 ± 2.75} & \textcolor{red}{53.20 ± 5.24} & \textcolor{red}{75.59 ± 8.46} & \textcolor{blue}{67.02 ± 7.78} \\
    \cline{2-6} ~ & \multicolumn{4}{c}{\textbf{ours}} \\ \cline{2-6}
    ~ & EPN & 63.01 ± 12.63 & 32.61 ± 9.14 & 65.99 ± 0.94 & 52.94 ± 0.74 \\
    \rowcolor{gray!30}
    ~ & EPN-reg & 69.39 ± 3.78 & 36.84 ± 5.64 & 64.78 ± 4.79 & 53.61 ± 4.34 \\
    \midrule
    \multirow{17}{*}{\rotatebox[origin=c]{90}{\textbf{OGBN-arxiv}}\hspace{0.5cm}} 
    ~ & \multicolumn{4}{c}{\textbf{logit based}} \\ \cline{2-6}
    ~ & VGCN-entropy & 66.40 ± 0.63 & 54.46 ± 0.70 & 78.24 ± 0.50 & 69.00 ± 0.72 \\
    
    ~ & VGCN-max-score & 63.91 ± 0.59 & 51.46 ± 0.59 & 76.16 ± 0.50 & 65.62 ± 0.71 \\
    
    ~ & VGCN-energy & 68.01 ± 1.06 & 54.73 ± 1.26 & 81.86 ± 0.40 & 73.21 ± 0.62 \\
    
    ~ & VGCN-gnnsafe & 59.86 ± 1.10 & 48.91 ± 0.90 & 73.76 ± 0.57 & 66.81 ± 0.70 \\
    
    ~ & VGCN-dropout & 66.56 ± 0.72 & 54.64 ± 0.91 & 78.38 ± 0.39 & 69.20 ± 0.57 \\
    
    ~ & VGCN-ensemble & 66.78 ± 0.15 & 54.73 ± 0.19 & 78.76 ± 0.12 & 69.57 ± 0.18 \\
    \cline{2-6} ~ & \multicolumn{4}{c}{\textbf{evidential based}} \\ \cline{2-6}
    ~ & GPN & 75.41 ± 0.72 & 63.39 ± 1.13 & 81.95 ± 0.08 & 72.27 ± 0.22 \\
    
    ~ & SGCN-GKDE & n.a & n.a & n.a & n.a \\
    
    ~ & EGCN & 78.37 ± 1.58 & 68.70 ± 2.49 & 81.75 ± 1.16 & 72.63 ± 1.65 \\
    
    ~ & EGCN-vacuity-prop & \textcolor{blue}{81.93 ± 1.28} & \textcolor{blue}{72.20 ± 2.17} & \textcolor{blue}{83.38 ± 1.43} & \textcolor{blue}{74.46 ± 2.19} \\
    
    ~ & EGCN-evidence-prop & 68.31 ± 1.77 & 54.87 ± 1.75 & 75.69 ± 1.07 & 63.40 ± 1.28 \\
    
    ~ & EGCN-vacuity-evidence-prop & 71.33 ± 1.88 & 59.69 ± 1.59 & 77.97 ± 1.49 & 67.32 ± 1.64 \\
    \cline{2-6} ~ & \multicolumn{4}{c}{\textbf{ours}} \\ \cline{2-6}
    ~ & EPN & 65.82 ± 2.06 & 50.49 ± 2.09 & 66.45 ± 1.37 & 50.08 ± 1.63 \\
    \rowcolor{gray!30}
    ~ & EPN-reg & \textcolor{red}{84.99 ± 5.48} & \textcolor{red}{80.27 ± 6.85} & \textcolor{red}{83.52 ± 0.52} & \textcolor{red}{76.59 ± 0.55} \\
    \bottomrule
    \end{tabular}
    }
    \caption{OOD detection results (\textcolor{red}{best} and \textcolor{blue}{runner-up}) with GCN as backbone on  CoraML, CiteSeer, PubMed, and OGBN-arxiv for OS-4 and OS-5.}
    \label{tab:ood_gcn_3}
\end{table}

\begin{table}[!ht]
    \centering
    \small
    \setlength{\tabcolsep}{20pt}
    \resizebox{0.98\textwidth}{!}{%
    \begin{tabular}{@{}p{2cm}l|cc|cc@{}}
    \toprule
    \multirow{2}{*}{\textbf{Dataset}} & \multirow{2}{*}{\textbf{Model}} & \multicolumn{2}{c|}{\textbf{OS-4 (random)}} & \multicolumn{2}{c}{\textbf{OS-5 (random)}} \\ 
    ~ & ~ & OOD-AUROC$\uparrow$ & OOD-AUPR$\uparrow$ & OOD-AUROC$\uparrow$ & OOD-AUPR$\uparrow$  \\ \midrule
    \multirow{17}{*}{\rotatebox[origin=c]{90}{\textbf{AmazonPhotos}}\hspace{0.5cm}} 
    ~ & \multicolumn{4}{c}{\textbf{logit based}} \\ \cline{2-6}
    ~ & VGCN-entropy & 79.91 ± 7.26 & 84.64 ± 5.50 & 87.65 ± 2.58 & 70.59 ± 4.34 \\
    
    ~ & VGCN-max-score & 79.93 ± 7.63 & 84.65 ± 6.08 & 87.54 ± 2.45 & 70.36 ± 3.59 \\
    
    ~ & VGCN-energy & 77.04 ± 5.86 & 82.06 ± 4.60 & 85.88 ± 3.54 & 68.22 ± 7.05 \\
    
    ~ & VGCN-gnnsafe & 77.95 ± 10.75 & 82.37 ± 6.46 & \textcolor{red}{91.95 ± 0.89} & \textcolor{blue}{74.05 ± 2.07} \\
    
    ~ & VGCN-dropout & 71.53 ± 6.35 & 79.22 ± 3.59 & 85.53 ± 4.33 & 65.99 ± 5.76 \\
    
    ~ & VGCN-ensemble & 76.90 ± 4.76 & 82.43 ± 3.38 & 87.21 ± 3.49 & 69.59 ± 4.25 \\
    \cline{2-6} ~ & \multicolumn{4}{c}{\textbf{evidential based}} \\ \cline{2-6}
    ~ & GPN & \textcolor{red}{92.63 ± 1.68} & \textcolor{red}{93.19 ± 1.36} & 86.10 ± 2.58 & 68.56 ± 5.16 \\
    
    ~ & SGCN-GKDE & 77.81 ± 2.63 & 82.89 ± 2.12 & 84.13 ± 3.18 & 67.50 ± 6.19 \\
    
    ~ & EGCN & 74.28 ± 4.45 & 79.23 ± 3.52 & 70.82 ± 1.96 & 40.41 ± 2.10 \\
    
    ~ & EGCN-vacuity-prop & \textcolor{blue}{83.14 ± 2.64} & 86.07 ± 2.06 & 84.36 ± 2.21 & 60.82 ± 3.60 \\
    
    ~ & EGCN-evidence-prop & 64.24 ± 2.52 & 72.35 ± 1.76 & 78.39 ± 5.80 & 59.30 ± 8.44 \\
    
    ~ & EGCN-vacuity-evidence-prop & 66.26 ± 6.61 & 75.07 ± 3.54 & 82.29 ± 3.68 & 64.97 ± 5.45 \\
    \cline{2-6} ~ & \multicolumn{4}{c}{\textbf{ours}} \\ \cline{2-6}
    ~ & EPN & 78.66 ± 2.93 & 82.08 ± 2.89 & 88.56 ± 1.72 & 69.86 ± 4.23 \\
    \rowcolor{gray!30}
    ~ & EPN-reg & 83.49 ± 7.72 & \textcolor{blue}{87.38 ± 6.14} & \textcolor{blue}{90.82 ± 7.08} & \textcolor{red}{80.53 ± 10.59} \\
    \midrule
    \multirow{17}{*}{\rotatebox[origin=c]{90}{\textbf{AmazonComputers}}\hspace{0.5cm}} 
    ~ & \multicolumn{4}{c}{\textbf{logit based}} \\ \cline{2-6}
    ~ & VGCN-entropy & 62.95 ± 1.47 & 54.31 ± 2.28 & 77.19 ± 2.20 & 85.40 ± 1.63 \\
    
    ~ & VGCN-max-score & 62.31 ± 1.47 & 53.54 ± 2.28 & \textcolor{blue}{79.09 ± 2.58} & \textcolor{blue}{87.37 ± 2.08} \\
    
    ~ & VGCN-energy & 63.58 ± 2.11 & 53.50 ± 2.96 & 70.59 ± 2.87 & 81.42 ± 2.05 \\
    
    ~ & VGCN-gnnsafe & 72.82 ± 2.25 & 56.73 ± 2.18 & 76.72 ± 4.65 & 84.20 ± 3.20 \\
    
    ~ & VGCN-dropout & 60.96 ± 4.33 & 53.22 ± 4.21 & 76.43 ± 3.77 & 84.66 ± 2.99 \\
    
    ~ & VGCN-ensemble & 62.74 ± 4.43 & 54.95 ± 4.58 & 78.00 ± 1.70 & 85.63 ± 1.10 \\
    \cline{2-6} ~ & \multicolumn{4}{c}{\textbf{evidential based}} \\ \cline{2-6}
    ~ & GPN & \textcolor{red}{78.38 ± 4.34} & \textcolor{red}{65.30 ± 5.01} & \textcolor{red}{86.16 ± 1.68} & \textcolor{red}{90.87 ± 1.21} \\
    
    ~ & SGCN-GKDE & 65.44 ± 3.75 & 55.73 ± 4.39 & 67.06 ± 2.62 & 78.53 ± 1.78 \\
    
    ~ & EGCN & 55.42 ± 2.23 & 44.43 ± 1.75 & 64.89 ± 1.70 & 77.30 ± 1.16 \\
    
    ~ & EGCN-vacuity-prop & 62.33 ± 5.39 & 52.04 ± 4.62 & 70.49 ± 2.01 & 80.46 ± 2.01 \\
    
    ~ & EGCN-evidence-prop & 72.13 ± 4.55 & 57.55 ± 3.71 & 55.73 ± 8.76 & 71.41 ± 4.74 \\
    
    ~ & EGCN-vacuity-evidence-prop & \textcolor{blue}{75.53 ± 4.80} & 62.79 ± 4.11 & 61.86 ± 5.23 & 75.00 ± 3.19 \\
    \cline{2-6} ~ & \multicolumn{4}{c}{\textbf{ours}} \\ \cline{2-6}
    ~ & EPN & 74.37 ± 1.30 & \textcolor{blue}{64.45 ± 2.11} & 55.24 ± 6.59 & 74.63 ± 3.34 \\
    \rowcolor{gray!30}
    ~ & EPN-reg & 70.28 ± 7.69 & 62.52 ± 7.80 & 76.43 ± 5.16 & 85.97 ± 4.02 \\
    \midrule
    \multirow{17}{*}{\rotatebox[origin=c]{90}{\textbf{CoauthorCS}}\hspace{0.5cm}} 
    ~ & \multicolumn{4}{c}{\textbf{logit based}} \\ \cline{2-6}
    ~ & VGCN-entropy & 89.98 ± 2.33 & 81.38 ± 4.60 & 91.11 ± 1.92 & 73.87 ± 5.34 \\
    
    ~ & VGCN-max-score & 89.17 ± 2.32 & 79.64 ± 4.75 & 91.15 ± 1.65 & 73.81 ± 4.95 \\
    
    ~ & VGCN-energy & 88.70 ± 2.54 & 80.39 ± 4.82 & 87.51 ± 3.05 & 69.77 ± 5.21 \\
    
    ~ & VGCN-gnnsafe & 94.17 ± 2.07 & 89.55 ± 3.40 & \textcolor{blue}{93.32 ± 2.10} & \textcolor{blue}{81.35 ± 4.70} \\
    
    ~ & VGCN-dropout & 86.94 ± 2.82 & 75.70 ± 5.84 & 89.96 ± 2.09 & 70.77 ± 5.50 \\
    
    ~ & VGCN-ensemble & 89.93 ± 1.81 & 82.09 ± 2.69 & 90.24 ± 1.30 & 73.00 ± 4.59 \\
    \cline{2-6} ~ & \multicolumn{4}{c}{\textbf{evidential based}} \\ \cline{2-6}
    ~ & GPN & 93.78 ± 1.22 & 90.35 ± 1.93 & 87.04 ± 3.16 & 68.67 ± 6.08 \\
    
    ~ & SGCN-GKDE & 55.46 ± 7.00 & 38.83 ± 6.19 & 48.81 ± 5.08 & 25.01 ± 3.18 \\
    
    ~ & EGCN & 88.02 ± 2.73 & 80.44 ± 4.13 & 83.36 ± 3.33 & 63.21 ± 6.46 \\
    
    ~ & EGCN-vacuity-prop & 94.33 ± 1.22 & 89.93 ± 2.94 & \textcolor{red}{93.66 ± 1.48} & \textcolor{red}{81.77 ± 2.65} \\
    
    ~ & EGCN-evidence-prop & 94.63 ± 0.47 & 91.57 ± 0.68 & 83.69 ± 1.83 & 62.72 ± 3.41 \\
    
    ~ & EGCN-vacuity-evidence-prop & \textcolor{blue}{96.67 ± 0.70} & \textcolor{blue}{93.70 ± 1.40} & 89.32 ± 1.76 & 70.42 ± 2.85 \\
    \cline{2-6} ~ & \multicolumn{4}{c}{\textbf{ours}} \\ \cline{2-6}
    ~ & EPN & 88.65 ± 2.54 & 79.70 ± 4.93 & 81.55 ± 8.11 & 66.38 ± 7.45 \\
    \rowcolor{gray!30}
    ~ & EPN-reg & \textcolor{red}{96.96 ± 1.10} & \textcolor{red}{95.39 ± 1.29} & 92.98 ± 3.32 & 80.12 ± 9.41 \\
    \midrule
    \multirow{17}{*}{\rotatebox[origin=c]{90}{\textbf{CoauthorPhysics}}\hspace{0.5cm}} 
    ~ & \multicolumn{4}{c}{\textbf{logit based}} \\ \cline{2-6}
    ~ & VGCN-entropy & 86.60 ± 3.77 & 69.96 ± 5.04 & 81.39 ± 8.89 & 85.14 ± 6.50 \\
    
    ~ & VGCN-max-score & 86.41 ± 4.12 & 68.59 ± 6.30 & 80.82 ± 8.70 & 84.01 ± 6.25 \\
    
    ~ & VGCN-energy & 87.05 ± 2.84 & 73.44 ± 4.32 & 84.03 ± 8.49 & 87.85 ± 6.30 \\
    
    ~ & VGCN-gnnsafe & 92.88 ± 2.18 & 83.85 ± 3.57 & 88.33 ± 8.30 & 90.24 ± 6.54 \\
    
    ~ & VGCN-dropout & 85.39 ± 3.80 & 70.21 ± 8.12 & 85.85 ± 6.75 & 88.86 ± 5.17 \\
    
    ~ & VGCN-ensemble & 84.05 ± 6.72 & 67.60 ± 8.40 & 81.16 ± 6.78 & 84.63 ± 4.42 \\
    \cline{2-6} ~ & \multicolumn{4}{c}{\textbf{evidential based}} \\ \cline{2-6}
    ~ & GPN & 83.22 ± 8.99 & 74.01 ± 11.49 & \textcolor{blue}{96.79 ± 1.06} & \textcolor{red}{97.62 ± 0.77} \\
    
    ~ & SGCN-GKDE & 87.37 ± 5.05 & 78.89 ± 6.97 & 95.55 ± 1.72 & \textcolor{blue}{97.19 ± 1.33} \\
    
    ~ & EGCN & 84.68 ± 5.05 & 69.58 ± 7.20 & 82.98 ± 3.37 & 88.06 ± 2.69 \\
    
    ~ & EGCN-vacuity-prop & \textcolor{blue}{93.44 ± 2.74} & 84.98 ± 6.12 & 87.72 ± 4.46 & 89.82 ± 2.83 \\
    
    ~ & EGCN-evidence-prop & 92.37 ± 1.84 & \textcolor{blue}{85.42 ± 2.88} & 95.35 ± 1.17 & 95.94 ± 0.87 \\
    
    ~ & EGCN-vacuity-evidence-prop & \textcolor{red}{94.87 ± 1.59} & \textcolor{red}{89.35 ± 2.79} & \textcolor{red}{96.79 ± 0.77} & 96.74 ± 0.99 \\
    \cline{2-6} ~ & \multicolumn{4}{c}{\textbf{ours}} \\ \cline{2-6}
    ~ & EPN & 88.37 ± 8.07 & 76.68 ± 9.88 & 84.51 ± 10.09 & 87.17 ± 7.89 \\
    \rowcolor{gray!30}
    ~ & EPN-reg & 92.21 ± 3.34 & 82.51 ± 5.95 & 86.70 ± 3.50 & 89.00 ± 2.94 \\
    \bottomrule
    \end{tabular}
    }
    \caption{OOD detection results (\textcolor{red}{best} and \textcolor{blue}{runner-up}) with GCN as backbone on  AmazonPhotos, AmazonComputers, CoauthorCS, and CoauthorPhysics for OS-4 and OS-5.}
    \label{tab:ood_gcn_4}
\end{table}

\subsection{Discussions}
In the main text, we present the representative results in Table~\ref{tab:abl} and Table~\ref{tab:discussion}. Here, we show more detailed results.  

\textbf{Robustness of GNN backbone.} In this evaluation, we use the GAT architecture as the backbone for all models, except for GPN. Specifically, GAT is employed for the VGNN-based and EGNN models, while SGNN-GKDE uses GAT for both the probability teacher and model backbone. For our proposed EPN, we also leverage GAT as the pretrained model, from which we extract features to feed into the EPN. The average performance rankings are summarized in Table~\ref{tab:ood_gat_rank}, with detailed results across all LOC settings presented in Tables~\ref{tab:ood_gat_1}--~\ref{tab:ood_gat_4}. Our findings with GAT as the backbone are consistent with those observed using other architectures, reinforcing the robustness of our approach across different GNN backbones.

\begin{table*}[htbp]
    \centering
    \small
    \setlength{\tabcolsep}{1.0pt} 
    \resizebox{0.9\linewidth}{!}{%
    \begin{tabular}{l|ccccccc|c}
    \toprule
    \multirow{2}{*}{\textbf{Model}} & \multirow{2}{*}{\textbf{CoraML}} & \multirow{2}{*}{\textbf{CiteSeer}} & \multirow{2}{*}{\textbf{PubMed}} &  \multicolumn{1}{c}{\textbf{Amazon}} & \multicolumn{1}{c}{\textbf{Amazon}} & \multicolumn{1}{c}{\textbf{Coauthor}} & \multicolumn{1}{c|}{\textbf{Coauthor}} & \multirow{2}{*}{\textbf{Average}} \\
     &  &  &   & \textbf{Photos} & \textbf{Computers} & \textbf{CS} & \textbf{Physics} &  \\
    \midrule
    \multicolumn{9}{c}{\textbf{logit based}} \\ 
    \midrule
    VGAT-entropy          & 7.8 & 6.2  & 10.9 & 6.8  & 7.2  & 7.4  & 8.6 & 7.8  \\
    VGAT-max-score        & 10.4 & 9.2  & 10.9 & 8.0  & 9.6  & 9.8  & 11.4 & 9.9 \\
    VGAT-energy           & 6.6 & 6.0  & 10.6 & 7.0  & 5.0  & 7.8  & 7.2 & 7.2  \\
    VGAT-GATsafe          & {\cellcolor[rgb]{0.957,0.8,0.8}}2.4 & {\cellcolor[rgb]{0.957,0.8,0.8}}4.0 & 8.6  & 7.8  & 6.0  & {\cellcolor[rgb]{0.812,0.886,0.953}}2.0 & {\cellcolor[rgb]{0.957,0.8,0.8}}2.6 & 4.8 \\
    VGAT-dropout          & 7.8 & 7.0  & 10.2 & 5.2  & 7.2  & 7.4  & 11.0 & 8.0  \\
    VGAT-ensemble         & 4.4 & 5.0  & 9.8  & 6.4  & {\cellcolor[rgb]{0.812,0.886,0.953}}4.0 & 6.2  & 8.0  & 6.8 \\
    \midrule \multicolumn{9}{c}{\textbf{evidential based}} \\ \midrule
    GPN                   & 11.0 & 12.8 & 5.0  & 3.8  & 6.2  & 7.8  & 8.6 & 7.9 \\
    SGAT-GKDE             & 11.0 & 6.4  & 5.2  & 14.0 & 13.8 & 13.8 & 9.6 & 10.0 \\
    EGAT                  & 8.4 & 11.0 & 5.8  & 7.8  & 7.4  & 3.8  & 7.6 & 7.4 \\
    EGAT-vacuity-prop     & 3.0 & 8.8  & 5.0  & {\cellcolor[rgb]{0.812,0.886,0.953}}3.6 & 4.6  & {\cellcolor[rgb]{0.957,0.8,0.8}}1.8 & {\cellcolor[rgb]{0.812,0.886,0.953}}4.0 & {\cellcolor[rgb]{0.812,0.886,0.953}}4.4 \\
    EGAT-evidence-prop    & 13.8 & 4.8  & {\cellcolor[rgb]{0.812,0.886,0.953}}4.0 & 13.0 & 13.2 & 13.2 & 9.6 & 10.2 \\
    EGAT-vacuity-evidence-prop  & 12.0 & 11.8 & {\cellcolor[rgb]{0.957,0.8,0.8}}2.8 & 12.0 & 11.8 & 11.8 &  4.4 & 9.7 \\
    \midrule \multicolumn{9}{c}{\textbf{ours}} \\ \midrule
    EPN         & 3.6 & 7.8  & 7.6  & 7.8  & 7.2  & 9.2  & 7.2 & 7.2 \\
    EPN-reg     & {\cellcolor[rgb]{0.812,0.886,0.953}}2.8 & {\cellcolor[rgb]{0.812,0.886,0.953}}4.2 & 8.6  & {\cellcolor[rgb]{0.957,0.8,0.8}}1.8 & {\cellcolor[rgb]{0.957,0.8,0.8}}1.8 & 3.0  & 5.2 & {\cellcolor[rgb]{0.957,0.8,0.8}}3.8 \\
    \bottomrule
    \end{tabular}
    }
    \caption{Average OOD detection rank (OOD-AUROC) ($\downarrow$) of each model over different datasets with GAT as the backbone. \colorbox[rgb]{0.957,0.8,0.8}{Best} and \colorbox[rgb]{0.812,0.886,0.953}{Runner-up} results are highlighted in red and blue.}
    \label{tab:ood_gat_rank}
\end{table*}

\begin{table}[!ht]
    \centering
    \small
    \setlength{\tabcolsep}{8pt}
    \resizebox{\textwidth}{!}{%
    \begin{tabular}{@{}p{1.2cm}l|cc|cc|cc@{}}
    \toprule
    \multirow{2}{*}{\textbf{Dataset}} & \multirow{2}{*}{\textbf{Model}} & \multicolumn{2}{c|}{\textbf{OS-1 (last)}} & \multicolumn{2}{c|}{\textbf{OS-2 (first)}} & \multicolumn{2}{c}{\textbf{OS-3 (random)}} \\ 
    ~ & ~ & OOD-AUROC$\uparrow$ & OOD-AUPR$\uparrow$ & OOD-AUROC$\uparrow$ & OOD-AUPR$\uparrow$ & OOD-AUROC$\uparrow$ & OOD-AUPR$\uparrow$ \\ \midrule
    \multirow{17}{*}{\rotatebox[origin=c]{90}{\textbf{CoraML}}\hspace{0.5cm}} 
    ~ & \multicolumn{6}{c}{\textbf{logit based}} \\ \cline{2-8}
    ~ & VGAT-entropy & 87.45 ± 1.11 & 84.38 ± 1.99 & 85.92 ± 1.20 & 77.86 ± 3.25 & 89.62 ± 1.10 & 89.73 ± 1.30 \\
    
    ~ & VGAT-max-score & 86.81 ± 1.19 & 83.07 ± 2.19 & 84.22 ± 1.37 & 75.62 ± 3.85 & 89.01 ± 1.30 & 89.49 ± 1.24 \\
    
    ~ & VGAT-energy & 87.76 ± 1.13 & 85.10 ± 1.88 & 86.95 ± 1.29 & 79.61 ± 3.14 & 89.69 ± 0.88 & 89.85 ± 1.34 \\
    
    ~ & VGAT-gnnsafe & 89.50 ± 1.10 & 87.04 ± 1.57 & \textcolor{blue}{88.90 ± 0.95} & \textcolor{blue}{84.01 ± 3.16} & \textcolor{red}{91.48 ± 0.89} & \textcolor{red}{92.55 ± 0.62} \\
    
    ~ & VGAT-dropout & 88.11 ± 1.56 & 85.95 ± 2.05 & 87.46 ± 2.20 & 77.96 ± 5.31 & 88.50 ± 0.96 & 89.25 ± 1.26 \\
    
    ~ & VGAT-ensemble & 89.48 ± 0.93 & 87.48 ± 1.44 & 88.53 ± 1.74 & 82.38 ± 2.74 & 89.32 ± 0.76 & 89.63 ± 1.52 \\
    \cline{2-8} ~ & \multicolumn{6}{c}{\textbf{evidential based}} \\ \cline{2-8}
    ~ & GPN & 86.83 ± 1.97 & 81.09 ± 2.60 & 79.37 ± 2.47 & 63.43 ± 3.72 & 88.41 ± 2.08 & 88.85 ± 2.29 \\
    
    ~ & SGAT-GKDE & 87.62 ± 6.49 & 84.08 ± 7.83 & 81.97 ± 16.52 & 74.12 ± 18.81 & 86.23 ± 12.54 & 86.86 ± 12.87 \\
    
    ~ & EGAT & 87.05 ± 1.70 & 81.77 ± 2.71 & 87.20 ± 4.86 & 78.51 ± 6.44 & 88.32 ± 1.66 & 87.97 ± 3.05 \\
    
    ~ & EGAT-vacuity-prop & 88.93 ± 1.27 & 85.09 ± 2.26 & \textcolor{red}{91.11 ± 0.86} & \textcolor{red}{84.56 ± 2.27} & 90.96 ± 1.16 & 91.65 ± 1.57 \\
    
    ~ & EGAT-evidence-prop & 78.80 ± 3.43 & 74.96 ± 3.66 & 77.10 ± 4.33 & 68.22 ± 5.79 & 81.69 ± 1.77 & 83.88 ± 1.82 \\
    
    ~ & EGAT-vacuity-evidence-prop & 85.35 ± 2.13 & 77.94 ± 2.37 & 84.60 ± 2.75 & 72.29 ± 3.98 & 87.03 ± 1.24 & 85.62 ± 1.72 \\
    \cline{2-8} ~ & \multicolumn{6}{c}{\textbf{ours}} \\ \cline{2-8}
    ~ & EPN & \textcolor{blue}{91.11 ± 1.10} & \textcolor{blue}{88.08 ± 2.07} & 86.90 ± 3.85 & 79.00 ± 6.60 & 90.44 ± 1.53 & 90.35 ± 2.39 \\
    \rowcolor{gray!30}
    ~ & EPN-reg & \textcolor{red}{91.21 ± 0.76} & \textcolor{red}{88.59 ± 1.83} & 87.41 ± 3.42 & 79.84 ± 6.71 & \textcolor{blue}{91.19 ± 1.20} & \textcolor{blue}{91.83 ± 0.85} \\
    \midrule
    \multirow{17}{*}{\rotatebox[origin=c]{90}{\textbf{CiteSeer}}\hspace{0.5cm}} 
    ~ & \multicolumn{6}{c}{\textbf{logit based}} \\ \cline{2-8}
    ~ & VGAT-entropy & 88.66 ± 2.65 & 69.55 ± 7.44 & 92.03 ± 1.12 & 80.90 ± 2.43 & 87.40 ± 2.01 & 72.42 ± 4.92 \\
    
    ~ & VGAT-max-score & 88.21 ± 2.71 & 68.69 ± 7.13 & 91.63 ± 1.29 & 80.00 ± 2.81 & 86.83 ± 1.82 & 71.10 ± 4.26 \\
    
    ~ & VGAT-energy & 89.04 ± 2.73 & 70.13 ± 7.65 & 92.21 ± 1.04 & 80.82 ± 2.48 & 87.57 ± 2.14 & 72.17 ± 5.18 \\
    
    ~ & VGAT-gnnsafe & 90.13 ± 2.67 & 71.97 ± 7.97 & \textcolor{red}{92.92 ± 0.85} & 82.00 ± 1.77 & 88.18 ± 2.08 & 72.95 ± 4.98 \\
    
    ~ & VGAT-dropout & 89.79 ± 1.72 & 72.95 ± 4.26 & 92.27 ± 0.66 & \textcolor{blue}{82.16 ± 1.65} & 86.32 ± 3.58 & 71.91 ± 4.19 \\
    
    ~ & VGAT-ensemble & 88.47 ± 1.26 & 68.80 ± 2.63 & 92.01 ± 0.73 & 82.09 ± 1.54 & 87.19 ± 2.36 & 74.43 ± 3.44 \\
    \cline{2-8} ~ & \multicolumn{6}{c}{\textbf{evidential based}} \\ \cline{2-8}
    ~ & GPN & 85.96 ± 2.58 & 64.70 ± 5.11 & 88.95 ± 1.88 & 76.96 ± 3.13 & 69.78 ± 11.16 & 54.10 ± 9.67 \\
    
    ~ & SGAT-GKDE & \textcolor{blue}{90.25 ± 2.09} & \textcolor{blue}{73.47 ± 6.79} & 91.37 ± 1.47 & 80.17 ± 3.68 & 83.57 ± 3.15 & 67.38 ± 4.42 \\
    
    ~ & EGAT & 88.18 ± 2.27 & 66.30 ± 6.36 & 89.55 ± 2.01 & 73.59 ± 5.20 & \textcolor{blue}{90.13 ± 0.83} & 77.34 ± 3.09 \\
    
    ~ & EGAT-vacuity-prop & 90.25 ± 2.10 & 70.75 ± 3.96 & 90.00 ± 1.18 & 76.19 ± 3.41 & 88.70 ± 2.46 & 73.62 ± 6.71 \\
    
    ~ & EGAT-evidence-prop & 88.26 ± 1.53 & 72.33 ± 5.10 & 90.49 ± 0.94 & 80.51 ± 2.16 & 89.07 ± 1.71 & \textcolor{blue}{78.36 ± 4.00} \\
    
    ~ & EGAT-vacuity-evidence-prop & 88.44 ± 0.81 & 65.79 ± 2.61 & 91.76 ± 0.74 & 78.85 ± 2.30 & 87.80 ± 1.12 & 71.94 ± 2.91 \\
    \cline{2-8} ~ & \multicolumn{6}{c}{\textbf{ours}} \\ \cline{2-8}
    ~ & EPN & 88.34 ± 2.44 & 67.11 ± 6.41 & 92.15 ± 1.66 & 81.44 ± 4.60 & \textcolor{red}{90.73 ± 1.43} & \textcolor{red}{78.75 ± 3.41} \\
    \rowcolor{gray!30}
    ~ & EPN-reg & \textcolor{red}{90.96 ± 1.99} & \textcolor{red}{75.92 ± 4.67} & \textcolor{blue}{92.56 ± 1.51} & \textcolor{red}{82.60 ± 3.64} & 88.62 ± 3.65 & 75.12 ± 5.16 \\
    \midrule
    \multirow{17}{*}{\rotatebox[origin=c]{90}{\textbf{PubMed}}\hspace{0.5cm}} 
    ~ & \multicolumn{6}{c}{\textbf{logit based}} \\ \cline{2-8}
    ~ & VGAT-entropy & 65.95 ± 2.78 & 52.30 ± 2.41 & 65.59 ± 4.82 & 31.25 ± 4.70 & 50.44 ± 6.23 & 40.07 ± 3.87 \\
    
    ~ & VGAT-max-score & 65.95 ± 2.78 & 52.30 ± 2.41 & 65.59 ± 4.82 & 31.25 ± 4.71 & 50.44 ± 6.23 & 40.07 ± 3.87 \\
    
    ~ & VGAT-energy & 65.96 ± 2.86 & 52.30 ± 2.47 & 65.48 ± 5.01 & 31.12 ± 4.85 & 50.60 ± 6.26 & 40.20 ± 3.93 \\
    
    ~ & VGAT-gnnsafe & 66.09 ± 3.29 & 53.41 ± 2.58 & 67.65 ± 5.13 & 36.38 ± 6.34 & 49.55 ± 6.93 & 39.18 ± 4.40 \\
    
    ~ & VGAT-dropout & 66.46 ± 1.57 & 53.44 ± 2.19 & 63.77 ± 7.61 & 29.63 ± 6.76 & 52.55 ± 6.04 & 41.31 ± 3.88 \\
    
    ~ & VGAT-ensemble & 66.00 ± 3.17 & 52.77 ± 2.70 & 66.35 ± 4.03 & 31.71 ± 3.56 & 50.23 ± 5.74 & 39.61 ± 3.87 \\
    \cline{2-8} ~ & \multicolumn{6}{c}{\textbf{evidential based}} \\ \cline{2-8}
    ~ & GPN & 67.53 ± 4.34 & 58.93 ± 5.55 & 65.11 ± 3.68 & 35.70 ± 6.00 & 54.75 ± 9.11 & 43.29 ± 8.09 \\
    
    ~ & SGAT-GKDE & 66.78 ± 2.36 & 56.01 ± 2.84 & 68.81 ± 6.63 & \textcolor{blue}{43.40 ± 9.99} & 49.30 ± 10.26 & 39.12 ± 7.68 \\
    
    ~ & EGAT & 67.45 ± 3.28 & 54.71 ± 3.41 & 68.10 ± 5.45 & 38.43 ± 8.65 & 46.72 ± 3.21 & 36.75 ± 2.21 \\
    
    ~ & EGAT-vacuity-prop & \textcolor{blue}{68.67 ± 2.44} & 57.14 ± 2.30 & \textcolor{red}{74.36 ± 5.47} & \textcolor{red}{51.76 ± 10.06} & 50.81 ± 6.56 & 39.43 ± 4.91 \\
    
    ~ & EGAT-evidence-prop & 68.55 ± 4.76 & \textcolor{red}{62.21 ± 5.93} & 65.87 ± 5.31 & 37.59 ± 7.62 & 56.82 ± 8.23 & \textcolor{blue}{45.96 ± 8.30} \\
    
    ~ & EGAT-vacuity-evidence-prop & \textcolor{red}{72.69 ± 6.45} & \textcolor{blue}{62.02 ± 6.87} & \textcolor{blue}{72.19 ± 4.35} & 40.14 ± 6.90 & \textcolor{blue}{58.42 ± 7.60} & 44.80 ± 6.72 \\
    \cline{2-8} ~ & \multicolumn{6}{c}{\textbf{ours}} \\ \cline{2-8}
    ~ & EPN & 65.15 ± 1.92 & 52.21 ± 2.13 & 67.41 ± 6.72 & 36.05 ± 8.59 & 55.25 ± 6.55 & 43.55 ± 3.87 \\
    \rowcolor{gray!30}
    ~ & EPN-reg & 68.67 ± 2.56 & 55.83 ± 2.34 & 64.45 ± 5.89 & 31.09 ± 6.14 & \textcolor{red}{66.92 ± 8.04} & \textcolor{red}{52.62 ± 8.42} \\
    \bottomrule
    \end{tabular}
    }
    \caption{OOD detection results (\textcolor{red}{best} and \textcolor{blue}{runner-up}) with GAT as backbone on CoraML, CiteSeer, and PubMed for OS-1, OS-2, and OS-3.}
    \label{tab:ood_gat_1}
\end{table}

\begin{table}[!ht]
    \centering
    \small
    \setlength{\tabcolsep}{8pt}
    \resizebox{\textwidth}{!}{%
    \begin{tabular}{@{}p{1.2cm}l|cc|cc|cc@{}}
    \toprule
    \multirow{2}{*}{\textbf{Dataset}} & \multirow{2}{*}{\textbf{Model}} & \multicolumn{2}{c|}{\textbf{OS-1 (last)}} & \multicolumn{2}{c|}{\textbf{OS-2 (first)}} & \multicolumn{2}{c}{\textbf{OS-3 (random)}} \\ 
    ~ & ~ & OOD-AUROC$\uparrow$ & OOD-AUPR$\uparrow$ & OOD-AUROC$\uparrow$ & OOD-AUPR$\uparrow$ & OOD-AUROC$\uparrow$ & OOD-AUPR$\uparrow$ \\ \midrule
    \multirow{17}{*}{\rotatebox[origin=c]{90}{\textbf{AmazonPhotos}}\hspace{0.5cm}}
    ~ & \multicolumn{6}{c}{\textbf{logit based}} \\ \cline{2-8}
    ~ & VGAT-entropy & 85.87 ± 2.97 & 79.32 ± 3.96 & \textcolor{blue}{91.05 ± 2.12} & 83.66 ± 3.51 & 87.10 ± 2.34 & 70.26 ± 5.98 \\
    
    ~ & VGAT-max-score & 84.21 ± 2.90 & 76.06 ± 3.94 & 90.99 ± 2.57 & 83.47 ± 3.96 & 88.41 ± 2.13 & 72.63 ± 5.82 \\
    
    ~ & VGAT-energy & 87.43 ± 3.74 & 81.39 ± 5.09 & 89.37 ± 3.31 & 83.44 ± 4.84 & 80.60 ± 3.55 & 62.07 ± 5.72 \\
    
    ~ & VGAT-gnnsafe & 89.93 ± 2.33 & 79.53 ± 3.52 & 90.93 ± 3.11 & 83.93 ± 3.38 & 82.90 ± 4.01 & 64.56 ± 5.19 \\
    
    ~ & VGAT-dropout & 86.85 ± 4.72 & 81.04 ± 6.85 & 90.97 ± 3.29 & 83.91 ± 2.90 & \textcolor{red}{89.96 ± 1.39} & 76.42 ± 2.44 \\
    
    ~ & VGAT-ensemble & 86.95 ± 3.14 & 80.57 ± 5.25 & 88.13 ± 3.22 & 81.07 ± 2.84 & 89.07 ± 2.97 & 74.26 ± 4.54 \\
    \cline{2-8} ~ & \multicolumn{6}{c}{\textbf{evidential based}} \\ \cline{2-8}
    ~ & GPN & 90.67 ± 1.79 & 84.70 ± 3.27 & \textcolor{red}{92.44 ± 1.37} & \textcolor{red}{88.88 ± 2.29} & 87.60 ± 1.75 & 76.03 ± 2.66 \\
    
    ~ & SGAT-GKDE & 53.33 ± 13.83 & 41.49 ± 10.82 & 49.35 ± 7.38 & 45.60 ± 8.20 & 58.53 ± 14.52 & 39.25 ± 11.65 \\
    
    ~ & EGAT & 89.38 ± 3.41 & 83.23 ± 5.52 & 76.11 ± 9.92 & 67.70 ± 11.81 & 87.17 ± 3.92 & 72.73 ± 6.45 \\
    
    ~ & EGAT-vacuity-prop & \textcolor{blue}{90.96 ± 3.55} & \textcolor{blue}{86.55 ± 4.88} & 74.76 ± 9.54 & 73.22 ± 7.78 & 87.94 ± 6.68 & \textcolor{blue}{77.60 ± 10.27} \\
    
    ~ & EGAT-evidence-prop & 68.67 ± 2.86 & 59.40 ± 3.50 & 62.41 ± 2.51 & 45.55 ± 2.23 & 60.81 ± 2.60 & 41.83 ± 2.24 \\
    
    ~ & EGAT-vacuity-evidence-prop & 76.67 ± 4.20 & 67.41 ± 3.71 & 80.29 ± 4.22 & 66.00 ± 4.92 & 73.07 ± 4.02 & 51.16 ± 4.64 \\
    \cline{2-8} ~ & \multicolumn{6}{c}{\textbf{ours}} \\ \cline{2-8}
    ~ & EPN & 88.48 ± 3.74 & 83.14 ± 4.77 & 74.96 ± 10.37 & 72.03 ± 10.07 & 78.21 ± 7.57 & 59.90 ± 9.99 \\
    \rowcolor{gray!30}
    ~ & EPN-reg & \textcolor{red}{93.47 ± 1.37} & \textcolor{red}{90.32 ± 2.56} & 87.83 ± 3.77 & \textcolor{blue}{84.81 ± 3.66} & \textcolor{blue}{89.52 ± 3.66} & \textcolor{red}{78.81 ± 8.23} \\
    \midrule
    \multirow{17}{*}{\rotatebox[origin=c]{90}{\textbf{AmazonComputers}}\hspace{0.5cm}} 
    ~ & \multicolumn{6}{c}{\textbf{logit based}} \\ \cline{2-8}
    ~ & VGAT-entropy & 87.45 ± 3.22 & 67.12 ± 7.07 & 85.91 ± 7.82 & 91.59 ± 4.10 & 78.90 ± 9.19 & 84.12 ± 6.46 \\
    
    ~ & VGAT-max-score & 84.76 ± 3.26 & 61.16 ± 6.36 & 76.51 ± 9.06 & 85.42 ± 5.38 & 73.58 ± 8.85 & 80.12 ± 6.70 \\
    
    ~ & VGAT-energy & 88.39 ± 3.04 & 69.29 ± 6.05 & 91.17 ± 1.34 & 94.29 ± 0.96 & 82.05 ± 7.80 & 85.28 ± 6.18 \\
    
    ~ & VGAT-gnnsafe & \textcolor{red}{90.38 ± 2.51} & 69.56 ± 3.98 & \textcolor{red}{92.75 ± 0.83} & 93.97 ± 0.52 & 83.91 ± 7.64 & 85.20 ± 5.75 \\
    
    ~ & VGAT-dropout & 83.07 ± 8.71 & 62.42 ± 13.22 & 86.14 ± 9.25 & 92.26 ± 4.76 & 79.43 ± 6.96 & 82.75 ± 6.15 \\
    
    ~ & VGAT-ensemble & 87.64 ± 3.98 & \textcolor{blue}{70.97 ± 6.58} & 89.79 ± 2.34 & 93.77 ± 1.52 & 82.15 ± 3.93 & 84.86 ± 3.77 \\
    \cline{2-8} ~ & \multicolumn{6}{c}{\textbf{evidential based}} \\ \cline{2-8}
    ~ & GPN & 80.97 ± 3.98 & 57.59 ± 5.87 & \textcolor{blue}{91.54 ± 2.33} & 95.08 ± 1.28 & \textcolor{blue}{85.70 ± 2.31} & 87.17 ± 1.70 \\
    
    ~ & SGAT-GKDE & 62.56 ± 15.35 & 42.62 ± 10.80 & 52.61 ± 17.18 & 70.19 ± 10.21 & 48.00 ± 15.81 & 57.17 ± 10.02 \\
    
    ~ & EGAT & 74.81 ± 6.93 & 52.43 ± 8.70 & 90.36 ± 2.49 & 93.99 ± 1.58 & \textcolor{red}{87.34 ± 5.79} & \textcolor{red}{88.74 ± 5.29} \\
    
    ~ & EGAT-vacuity-prop & 78.56 ± 8.47 & 61.27 ± 11.61 & 91.41 ± 2.36 & \textcolor{red}{95.77 ± 1.25} & 85.32 ± 5.55 & \textcolor{blue}{88.61 ± 5.24} \\
    
    ~ & EGAT-evidence-prop & 57.59 ± 2.24 & 35.68 ± 2.01 & 70.81 ± 1.57 & 83.81 ± 0.94 & 60.29 ± 2.34 & 67.77 ± 1.80 \\
    
    ~ & EGAT-vacuity-evidence-prop & 70.49 ± 4.56 & 45.78 ± 5.30 & 79.02 ± 3.32 & 87.41 ± 1.85 & 64.94 ± 3.97 & 70.32 ± 2.77 \\
    \cline{2-8} ~ & \multicolumn{6}{c}{\textbf{ours}} \\ \cline{2-8}
    ~ & EPN & 81.87 ± 10.02 & 68.33 ± 10.23 & 88.78 ± 2.87 & 93.59 ± 2.48 & 79.41 ± 10.78 & 83.92 ± 8.11 \\
    \rowcolor{gray!30}
    ~ & EPN-reg & \textcolor{blue}{90.31 ± 2.64} & \textcolor{red}{78.11 ± 5.13} & 91.00 ± 3.28 & \textcolor{blue}{95.44 ± 1.87} & 82.42 ± 8.27 & 86.30 ± 6.44 \\
    \midrule
    \multirow{17}{*}{\rotatebox[origin=c]{90}{\textbf{CoauthorCS}}\hspace{0.5cm}} 
    ~ & \multicolumn{6}{c}{\textbf{logit based}} \\ \cline{2-8}
    ~ & VGAT-entropy & 84.49 ± 2.20 & 81.31 ± 2.78 & 89.80 ± 1.01 & 70.26 ± 2.93 & 89.42 ± 1.62 & 65.56 ± 4.14 \\
    
    ~ & VGAT-max-score & 82.70 ± 2.24 & 78.29 ± 3.06 & 88.58 ± 0.78 & 68.47 ± 2.39 & 88.54 ± 1.67 & 64.36 ± 3.54 \\
    
    ~ & VGAT-energy & 85.86 ± 2.32 & 82.44 ± 3.07 & 90.04 ± 1.46 & 70.05 ± 3.95 & 89.70 ± 1.75 & 65.43 ± 4.93 \\
    
    ~ & VGAT-gnnsafe & 88.89 ± 2.39 & \textcolor{blue}{87.84 ± 2.43} & 92.84 ± 1.26 & \textcolor{red}{78.82 ± 2.79} & \textcolor{blue}{92.90 ± 1.50} & \textcolor{blue}{74.71 ± 3.99} \\
    
    ~ & VGAT-dropout & 86.76 ± 1.12 & 83.64 ± 1.58 & 88.87 ± 2.91 & 65.72 ± 6.86 & 90.15 ± 1.19 & 69.57 ± 2.11 \\
    
    ~ & VGAT-ensemble & 85.78 ± 2.36 & 82.75 ± 3.18 & 90.14 ± 0.86 & 69.76 ± 2.30 & 90.99 ± 0.80 & 71.15 ± 3.40 \\
    \cline{2-8} ~ & \multicolumn{6}{c}{\textbf{evidential based}} \\ \cline{2-8}
    ~ & GPN & \textcolor{blue}{89.67 ± 1.78} & 87.20 ± 2.06 & 84.34 ± 2.20 & 59.97 ± 4.38 & 83.29 ± 2.62 & 53.27 ± 4.74 \\
    
    ~ & SGAT-GKDE & 71.70 ± 4.91 & 58.08 ± 8.24 & 59.40 ± 6.01 & 34.96 ± 12.46 & 66.19 ± 6.04 & 36.68 ± 13.48 \\
    
    ~ & EGAT & 87.15 ± 2.48 & 84.65 ± 3.15 & \textcolor{blue}{92.87 ± 1.15} & 78.40 ± 4.27 & 92.02 ± 2.20 & 73.97 ± 7.98 \\
    
    ~ & EGAT-vacuity-prop & 86.80 ± 2.89 & 85.29 ± 2.81 & \textcolor{red}{93.50 ± 1.92} & \textcolor{blue}{78.80 ± 5.31} & \textcolor{red}{94.63 ± 1.23} & \textcolor{red}{80.54 ± 4.53} \\
    
    ~ & EGAT-evidence-prop & 76.64 ± 1.39 & 71.13 ± 2.02 & 70.79 ± 2.81 & 40.87 ± 4.21 & 68.24 ± 2.30 & 32.87 ± 2.95 \\
    
    ~ & EGAT-vacuity-evidence-prop & 85.39 ± 2.55 & 81.14 ± 2.68 & 79.71 ± 4.09 & 50.90 ± 6.45 & 76.22 ± 2.98 & 40.42 ± 4.25 \\
    \cline{2-8} ~ & \multicolumn{6}{c}{\textbf{ours}} \\ \cline{2-8}
    ~ & EPN & 86.74 ± 4.38 & 82.40 ± 5.59 & 90.08 ± 4.17 & 68.63 ± 10.12 & 85.08 ± 5.28 & 52.95 ± 10.99 \\
    \rowcolor{gray!30}
    ~ & EPN-reg & \textcolor{red}{94.21 ± 1.75} & \textcolor{red}{92.61 ± 2.33} & 90.48 ± 3.80 & 72.41 ± 9.77 & 91.53 ± 2.20 & 72.75 ± 6.82 \\
    \midrule
    \multirow{17}{*}{\rotatebox[origin=c]{90}{\textbf{CoauthorPhysics}}\hspace{0.5cm}} 
    ~ & \multicolumn{6}{c}{\textbf{logit based}} \\ \cline{2-8}
    ~ & VGAT-entropy & 93.64 ± 0.54 & 78.65 ± 1.68 & 93.59 ± 1.51 & 84.96 ± 3.71 & 92.52 ± 3.46 & 91.77 ± 4.58 \\
    
    ~ & VGAT-max-score & 93.08 ± 0.43 & 74.65 ± 1.41 & 92.86 ± 1.42 & 82.12 ± 3.55 & 92.30 ± 3.37 & 91.59 ± 4.54 \\
    
    ~ & VGAT-energy & 94.29 ± 0.80 & 81.25 ± 2.24 & 94.98 ± 1.16 & 88.26 ± 2.83 & 92.54 ± 3.28 & 91.68 ± 4.39 \\
    
    ~ & VGAT-gnnsafe & \textcolor{blue}{95.55 ± 0.67} & \textcolor{red}{86.27 ± 1.80} & 96.54 ± 0.79 & 92.34 ± 1.67 & \textcolor{red}{94.65 ± 2.65} & 94.25 ± 3.59 \\
    
    ~ & VGAT-dropout & 93.28 ± 0.52 & 77.16 ± 2.96 & 91.88 ± 2.16 & 83.26 ± 3.92 & 90.44 ± 6.03 & 89.38 ± 6.77 \\
    
    ~ & VGAT-ensemble & 93.62 ± 0.92 & 78.76 ± 2.92 & 93.36 ± 1.50 & 85.98 ± 2.66 & \textcolor{blue}{94.42 ± 0.47} & \textcolor{blue}{94.23 ± 0.37} \\
    \cline{2-8} ~ & \multicolumn{6}{c}{\textbf{evidential based}} \\ \cline{2-8}
    ~ & GPN & 90.60 ± 2.40 & 75.05 ± 5.70 & 84.12 ± 12.22 & 77.11 ± 12.91 & 94.31 ± 1.98 & \textcolor{red}{94.54 ± 2.17} \\
    
    ~ & SGAT-GKDE & 93.46 ± 2.18 & 78.01 ± 5.37 & 93.79 ± 2.42 & 86.06 ± 4.80 & 91.20 ± 5.45 & 90.56 ± 5.39 \\
    
    ~ & EGAT & 94.97 ± 0.92 & 82.56 ± 2.63 & \textcolor{blue}{96.72 ± 1.20} & \textcolor{blue}{92.54 ± 3.04} & 89.42 ± 4.67 & 88.56 ± 5.60 \\
    
    ~ & EGAT-vacuity-prop & \textcolor{red}{95.77 ± 0.83} & \textcolor{blue}{86.25 ± 2.18} & \textcolor{red}{97.84 ± 0.66} & \textcolor{red}{95.27 ± 1.56} & 92.41 ± 3.52 & 90.82 ± 4.26 \\
    
    ~ & EGAT-evidence-prop & 90.05 ± 1.30 & 69.83 ± 3.10 & 91.37 ± 2.73 & 82.65 ± 4.82 & 88.78 ± 2.35 & 89.50 ± 2.12 \\
    
    ~ & EGAT-vacuity-evidence-prop & 93.75 ± 2.00 & 78.79 ± 4.54 & 94.93 ± 2.82 & 88.94 ± 4.85 & 93.09 ± 1.85 & 92.31 ± 2.17 \\
    \cline{2-8} ~ & \multicolumn{6}{c}{\textbf{ours}} \\ \cline{2-8}
    ~ & EPN & 95.02 ± 2.55 & 84.67 ± 5.14 & 93.71 ± 3.30 & 86.99 ± 6.78 & 89.61 ± 7.65 & 88.67 ± 7.12 \\
    \rowcolor{gray!30}
    ~ & EPN-reg & 95.57 ± 0.93 & 85.54 ± 3.61 & 95.08 ± 1.70 & 90.12 ± 2.24 & 90.76 ± 7.33 & 90.00 ± 7.90 \\
    \bottomrule
    \end{tabular}
    }
    \caption{OOD detection results (\textcolor{red}{best} and \textcolor{blue}{runner-up}) with GAT as backbone on   AmazonPhotos, AmazonComputers, CoauthorCS, and CoauthorPhysics for OS-1, OS-2, and OS-3.}
    \label{tab:ood_gat_2}
\end{table}

\begin{table}[!ht]
    \centering
    \small
    \setlength{\tabcolsep}{20pt}
    \resizebox{\textwidth}{!}{%
    \begin{tabular}{@{}p{2cm}l|cc|cc@{}}
    \toprule
    \multirow{2}{*}{\textbf{Dataset}} & \multirow{2}{*}{\textbf{Model}} & \multicolumn{2}{c|}{\textbf{OS-4 (random)}} & \multicolumn{2}{c}{\textbf{OS-5 (random)}} \\ 
    ~ & ~ & OOD-AUROC$\uparrow$ & OOD-AUPR$\uparrow$ & OOD-AUROC$\uparrow$ & OOD-AUPR$\uparrow$ \\ \midrule
    \multirow{17}{*}{\rotatebox[origin=c]{90}{\textbf{CoraML}}\hspace{0.5cm}} 
    ~ & \multicolumn{4}{c}{\textbf{logit based}} \\ \cline{2-6}
    ~ & VGAT-entropy & 84.37 ± 0.94 & 67.15 ± 2.02 & 90.84 ± 1.24 & 89.05 ± 2.11 \\
    
    ~ & VGAT-max-score & 83.77 ± 0.85 & 65.95 ± 1.80 & 89.71 ± 1.53 & 87.17 ± 2.57 \\
    
    ~ & VGAT-energy & 84.64 ± 1.11 & 68.00 ± 2.82 & 90.79 ± 1.17 & 88.40 ± 2.10 \\
    
    ~ & VGAT-gnnsafe & 87.43 ± 0.73 & \textcolor{blue}{74.35 ± 2.88} & \textcolor{blue}{93.33 ± 0.66} & 92.53 ± 0.80 \\
    
    ~ & VGAT-dropout & 85.04 ± 3.29 & 68.29 ± 4.00 & 90.30 ± 0.97 & 88.89 ± 2.02 \\
    
    ~ & VGAT-ensemble & 84.53 ± 1.26 & 69.15 ± 2.19 & 92.25 ± 0.63 & 91.24 ± 1.09 \\
    \cline{2-6} ~ & \multicolumn{4}{c}{\textbf{evidential based}} \\ \cline{2-6}
    ~ & GPN & 79.51 ± 2.24 & 60.55 ± 2.88 & 90.30 ± 1.94 & 90.23 ± 1.65 \\
    
    ~ & SGAT-GKDE & 82.40 ± 12.13 & 67.54 ± 13.55 & 87.54 ± 11.84 & 85.47 ± 11.50 \\
    
    ~ & EGAT & 85.66 ± 3.01 & 69.51 ± 6.36 & 91.15 ± 0.79 & 89.21 ± 1.70 \\
    
    ~ & EGAT-vacuity-prop & \textcolor{red}{89.37 ± 1.37} & 74.10 ± 2.74 & \textcolor{red}{94.07 ± 0.69} & \textcolor{red}{93.74 ± 0.93} \\
    
    ~ & EGAT-evidence-prop & 67.02 ± 4.63 & 48.09 ± 5.54 & 81.39 ± 3.41 & 81.38 ± 4.03 \\
    
    ~ & EGAT-vacuity-evidence-prop & 83.07 ± 4.21 & 66.33 ± 5.16 & 89.34 ± 2.15 & 87.90 ± 2.36 \\
    \cline{2-6} ~ & \multicolumn{4}{c}{\textbf{ours}} \\ \cline{2-6}
    ~ & EPN & \textcolor{blue}{88.45 ± 1.50} & \textcolor{red}{75.01 ± 4.13} & 90.93 ± 2.28 & 90.50 ± 2.54 \\
    \rowcolor{gray!30}
    ~ & EPN-reg & 86.55 ± 3.08 & 70.89 ± 7.19 & 92.96 ± 0.70 & \textcolor{blue}{92.99 ± 0.88} \\
    \midrule
    \multirow{17}{*}{\rotatebox[origin=c]{90}{\textbf{CiteSeer}}\hspace{0.5cm}} 
    ~ & \multicolumn{4}{c}{\textbf{logit based}} \\ \cline{2-6}
    ~ & VGAT-entropy & 87.57 ± 0.71 & 74.35 ± 1.32 & 91.40 ± 2.09 & 79.76 ± 4.38 \\
    
    ~ & VGAT-max-score & 86.89 ± 0.71 & 72.87 ± 2.05 & 90.82 ± 2.03 & 78.23 ± 4.51 \\
    
    ~ & VGAT-energy & 87.80 ± 0.62 & 74.61 ± 1.41 & 91.66 ± 2.31 & \textcolor{blue}{80.06 ± 5.32} \\
    
    ~ & VGAT-gnnsafe & \textcolor{red}{89.09 ± 0.69} & 75.82 ± 1.34 & \textcolor{red}{92.44 ± 1.99} & \textcolor{red}{81.82 ± 4.72} \\
    
    ~ & VGAT-dropout & 86.98 ± 1.71 & 72.92 ± 4.53 & 90.70 ± 1.08 & 75.59 ± 2.16 \\
    
    ~ & VGAT-ensemble & 88.45 ± 0.77 & 75.68 ± 2.34 & 91.30 ± 2.16 & 79.00 ± 3.63 \\
    \cline{2-6} ~ & \multicolumn{4}{c}{\textbf{evidential based}} \\ \cline{2-6}
    ~ & GPN & 77.92 ± 6.04 & 58.15 ± 7.38 & 90.10 ± 1.74 & 76.00 ± 3.64 \\
    
    ~ & SGAT-GKDE & \textcolor{blue}{89.09 ± 1.39} & \textcolor{blue}{76.75 ± 5.72} & 91.35 ± 1.76 & 78.42 ± 3.58 \\
    
    ~ & EGAT & 84.46 ± 1.97 & 63.46 ± 4.54 & 88.11 ± 1.78 & 69.15 ± 2.96 \\
    
    ~ & EGAT-vacuity-prop & 85.27 ± 1.83 & 63.86 ± 3.70 & \textcolor{blue}{91.79 ± 1.53} & 77.52 ± 4.84 \\
    
    ~ & EGAT-evidence-prop & 87.95 ± 1.01 & \textcolor{red}{79.59 ± 2.01} & 89.34 ± 1.56 & 77.09 ± 3.75 \\
    
    ~ & EGAT-vacuity-evidence-prop & 84.07 ± 1.75 & 61.26 ± 3.31 & 90.22 ± 3.25 & 74.67 ± 6.32 \\
    \cline{2-6} ~ & \multicolumn{4}{c}{\textbf{ours}} \\ \cline{2-6}
    ~ & EPN & 88.31 ± 1.87 & 72.33 ± 4.79 & 86.34 ± 4.94 & 70.72 ± 6.80 \\
    \rowcolor{gray!30}
    ~ & EPN-reg & 87.65 ± 2.17 & 72.24 ± 5.25 & 90.77 ± 2.05 & 78.84 ± 3.55 \\
    \midrule
    \multirow{17}{*}{\rotatebox[origin=c]{90}{\textbf{PubMed}}\hspace{0.5cm}} 
    ~ & \multicolumn{4}{c}{\textbf{logit based}} \\ \cline{2-6}
    ~ & VGAT-entropy & 61.69 ± 9.18 & 28.07 ± 6.20 & 67.14 ± 2.73 & 53.18 ± 2.81 \\
    
    ~ & VGAT-max-score & 61.69 ± 9.18 & 28.07 ± 6.19 & 67.14 ± 2.73 & 53.18 ± 2.81 \\
    
    ~ & VGAT-energy & 61.66 ± 9.34 & 28.21 ± 6.59 & 67.08 ± 2.72 & 53.21 ± 2.76 \\
    
    ~ & VGAT-gnnsafe & 63.85 ± 10.13 & 33.46 ± 9.33 & 67.23 ± 2.82 & 53.77 ± 3.22 \\
    
    ~ & VGAT-dropout & 64.44 ± 8.84 & 31.21 ± 6.75 & 64.28 ± 3.76 & 50.69 ± 2.97 \\
    
    ~ & VGAT-ensemble & 69.59 ± 2.10 & 33.88 ± 2.91 & 65.08 ± 1.87 & 51.99 ± 1.90 \\
    \cline{2-6} ~ & \multicolumn{4}{c}{\textbf{evidential based}} \\ \cline{2-6}
    ~ & GPN & 68.27 ± 3.25 & 39.59 ± 4.83 & 65.52 ± 4.25 & 56.89 ± 4.42 \\
    
    ~ & SGAT-GKDE & 67.72 ± 10.92 & \textcolor{blue}{43.86 ± 15.56} & 69.01 ± 5.34 & 57.44 ± 6.04 \\
    
    ~ & EGAT & \textcolor{red}{72.51 ± 2.88} & \textcolor{red}{47.32 ± 5.61} & \textcolor{blue}{69.15 ± 2.58} & 57.57 ± 3.44 \\
    
    ~ & EGAT-vacuity-prop & 66.57 ± 12.19 & 41.31 ± 14.05 & 65.67 ± 2.71 & 55.06 ± 2.45 \\
    
    ~ & EGAT-evidence-prop & 62.24 ± 2.93 & 30.76 ± 5.12 & 67.46 ± 1.65 & \textcolor{blue}{60.78 ± 2.69} \\
    
    ~ & EGAT-vacuity-evidence-prop & \textcolor{blue}{70.17 ± 4.90} & 37.45 ± 7.83 & \textcolor{red}{73.14 ± 2.74} & \textcolor{red}{62.28 ± 2.86} \\
    \cline{2-6} ~ & \multicolumn{4}{c}{\textbf{ours}} \\ \cline{2-6}
    ~ & EPN & 68.28 ± 5.46 & 36.21 ± 7.18 & 67.01 ± 2.81 & 53.86 ± 3.12 \\
    \rowcolor{gray!30}
    ~ & EPN-reg & 62.74 ± 8.80 & 30.40 ± 7.84 & 65.77 ± 5.67 & 53.30 ± 5.44 \\
    \bottomrule
    \end{tabular}
    }
    \caption{OOD detection results (\textcolor{red}{best} and \textcolor{blue}{runner-up}) with GAT as backbone on   CoraML, CiteSeer, and PubMed for OS-4 and OS-5.}
    \label{tab:ood_gat_3}
\end{table}

\begin{table}[!ht]
    \centering
    \small
    \setlength{\tabcolsep}{20pt}
    \resizebox{\textwidth}{!}{%
    \begin{tabular}{@{}p{2cm}l|cc|cc@{}}
    \midrule
    \multirow{2}{*}{\textbf{Dataset}} & \multirow{2}{*}{\textbf{Model}} & \multicolumn{2}{c|}{\textbf{OS-4 (random)}} & \multicolumn{2}{c}{\textbf{OS-5 (random)}} \\ 
    ~ & ~ & OOD-AUROC$\uparrow$ & OOD-AUPR$\uparrow$ & OOD-AUROC$\uparrow$ & OOD-AUPR$\uparrow$ \\ \midrule
    \multirow{17}{*}{\rotatebox[origin=c]{90}{\textbf{AmazonPhotos}}\hspace{0.5cm}} 
    ~ & \multicolumn{4}{c}{\textbf{logit based}} \\ \cline{2-6}
    ~ & VGAT-entropy & \textcolor{blue}{92.08 ± 2.86} & 92.77 ± 2.70 & 88.35 ± 2.65 & 72.30 ± 5.94 \\
    
    ~ & VGAT-max-score & 91.57 ± 3.01 & 92.23 ± 2.99 & 87.41 ± 2.55 & 68.24 ± 5.36 \\
    
    ~ & VGAT-energy & 90.74 ± 4.37 & 92.19 ± 3.19 & 90.09 ± 2.62 & 74.24 ± 5.75 \\
    
    ~ & VGAT-gnnsafe & 92.18 ± 4.00 & 91.38 ± 2.62 & 91.27 ± 2.22 & 72.64 ± 3.79 \\
    
    ~ & VGAT-dropout & 90.38 ± 3.23 & 91.88 ± 2.71 & 89.93 ± 3.11 & 76.67 ± 4.98 \\
    
    ~ & VGAT-ensemble & 91.10 ± 1.89 & 91.89 ± 1.73 & 90.64 ± 1.54 & 76.63 ± 2.92 \\
    \cline{2-6} ~ & \multicolumn{4}{c}{\textbf{evidential based}} \\ \cline{2-6}
    ~ & GPN & \textcolor{red}{92.63 ± 1.68} & \textcolor{red}{93.19 ± 1.36} & 86.10 ± 2.58 & 68.56 ± 5.16 \\
    
    ~ & SGAT-GKDE & 55.51 ± 8.18 & 63.18 ± 6.65 & 53.14 ± 7.39 & 25.68 ± 4.43 \\
    
    ~ & EGAT & 89.29 ± 3.29 & 91.67 ± 2.47 & 87.65 ± 5.11 & 70.10 ± 8.88 \\
    
    ~ & EGAT-vacuity-prop & 90.91 ± 4.08 & \textcolor{blue}{92.85 ± 2.89} & 89.94 ± 3.57 & 77.44 ± 7.79 \\
    
    ~ & EGAT-evidence-prop & 65.08 ± 4.05 & 71.43 ± 2.98 & 63.91 ± 2.20 & 34.91 ± 1.85 \\
    
    ~ & EGAT-vacuity-evidence-prop & 69.66 ± 6.12 & 73.92 ± 3.52 & 77.25 ± 3.67 & 48.94 ± 5.32 \\
    \cline{2-6} ~ & \multicolumn{4}{c}{\textbf{ours}} \\ \cline{2-6}
    ~ & EPN & 89.82 ± 2.49 & 90.27 ± 3.73 & \textcolor{red}{93.24 ± 2.85} & \textcolor{blue}{80.65 ± 8.45} \\
    \rowcolor{gray!30}
    ~ & EPN-reg & 91.65 ± 2.83 & 92.65 ± 2.73 & \textcolor{blue}{91.91 ± 2.38} & \textcolor{red}{81.00 ± 5.06} \\
    \midrule
    \multirow{17}{*}{\rotatebox[origin=c]{90}{\textbf{AmazonComputers}}\hspace{0.5cm}} 
    ~ & \multicolumn{4}{c}{\textbf{logit based}} \\ \cline{2-6}
    ~ & VGAT-entropy & 83.83 ± 1.41 & 78.18 ± 1.69 & 83.88 ± 3.54 & 89.58 ± 2.68 \\
    
    ~ & VGAT-max-score & 82.37 ± 1.59 & 74.78 ± 2.18 & 82.97 ± 3.72 & 88.90 ± 3.00 \\
    
    ~ & VGAT-energy & 84.52 ± 1.18 & 78.60 ± 1.40 & 83.62 ± 3.40 & 89.19 ± 2.50 \\
    
    ~ & VGAT-gnnsafe & 84.14 ± 1.05 & 71.64 ± 0.67 & 85.11 ± 3.24 & 90.15 ± 2.03 \\
    
    ~ & VGAT-dropout & 82.59 ± 2.48 & 76.44 ± 3.69 & 85.75 ± 2.94 & 91.05 ± 2.56 \\
    
    ~ & VGAT-ensemble & \textcolor{red}{86.43 ± 1.47} & \textcolor{blue}{80.14 ± 2.31} & \textcolor{red}{87.42 ± 1.39} & \textcolor{blue}{92.12 ± 0.94} \\
    \cline{2-6} ~ & \multicolumn{4}{c}{\textbf{evidential based}} \\ \cline{2-6}
    ~ & GPN & 78.38 ± 4.34 & 65.30 ± 5.01 & \textcolor{blue}{86.16 ± 1.68} & 90.87 ± 1.21 \\
    
    ~ & SGAT-GKDE & 53.39 ± 10.83 & 44.10 ± 7.00 & 51.78 ± 8.96 & 67.41 ± 6.56 \\
    
    ~ & EGAT & 81.05 ± 3.06 & 72.67 ± 5.50 & 81.10 ± 4.17 & 87.66 ± 3.81 \\
    
    ~ & EGAT-vacuity-prop & 81.97 ± 3.76 & 77.77 ± 4.16 & 82.31 ± 3.12 & 90.25 ± 2.45 \\
    
    ~ & EGAT-evidence-prop & 61.06 ± 3.50 & 48.49 ± 2.96 & 57.20 ± 3.35 & 72.66 ± 2.16 \\
    
    ~ & EGAT-vacuity-evidence-prop & 70.51 ± 3.41 & 55.89 ± 4.24 & 58.53 ± 5.89 & 74.28 ± 3.32 \\
    \cline{2-6} ~ & \multicolumn{4}{c}{\textbf{ours}} \\ \cline{2-6}
    ~ & EPN & 83.56 ± 2.23 & 79.96 ± 3.17 & 78.81 ± 4.03 & 87.53 ± 2.53 \\
    \rowcolor{gray!30}
    ~ & EPN-reg & \textcolor{blue}{84.98 ± 2.59} & \textcolor{red}{81.05 ± 3.84} & 86.20 ± 2.15 & \textcolor{red}{92.45 ± 1.50} \\
    \midrule
    \multirow{17}{*}{\rotatebox[origin=c]{90}{\textbf{CoauthorCS}}\hspace{0.5cm}} 
    ~ & \multicolumn{4}{c}{\textbf{logit based}} \\ \cline{2-6}
    ~ & VGAT-entropy & 91.91 ± 0.87 & 87.20 ± 1.44 & 92.13 ± 0.49 & 80.05 ± 1.45 \\
    
    ~ & VGAT-max-score & 90.33 ± 1.01 & 84.44 ± 1.64 & 91.29 ± 0.52 & 78.40 ± 1.69 \\
    
    ~ & VGAT-energy & 91.83 ± 1.15 & 86.79 ± 1.52 & 91.83 ± 0.68 & 78.65 ± 1.74 \\
    
    ~ & VGAT-gnnsafe & 94.57 ± 1.10 & 92.00 ± 1.38 & \textcolor{blue}{94.42 ± 0.53} & \textcolor{blue}{85.36 ± 1.52} \\
    
    ~ & VGAT-dropout & 91.88 ± 1.48 & 86.48 ± 2.83 & 92.13 ± 0.67 & 80.91 ± 1.30 \\
    
    ~ & VGAT-ensemble & 92.73 ± 2.13 & 88.56 ± 2.59 & 92.51 ± 0.39 & 80.82 ± 1.21 \\
    \cline{2-6} ~ & \multicolumn{4}{c}{\textbf{evidential based}} \\ \cline{2-6}
    ~ & GPN & 93.78 ± 1.22 & 90.35 ± 1.93 & 87.04 ± 3.16 & 68.67 ± 6.08 \\
    
    ~ & SGAT-GKDE & 54.84 ± 6.74 & 42.32 ± 7.68 & 46.43 ± 4.15 & 28.97 ± 8.42 \\
    
    ~ & EGAT & 93.13 ± 1.80 & 89.72 ± 2.56 & 93.83 ± 1.09 & 83.50 ± 2.85 \\
    
    ~ & EGAT-vacuity-prop & \textcolor{blue}{95.45 ± 1.15} & \textcolor{red}{93.15 ± 1.56} & \textcolor{red}{95.14 ± 1.90} & \textcolor{red}{87.07 ± 4.59} \\
    
    ~ & EGAT-evidence-prop & 80.51 ± 2.78 & 70.27 ± 4.09 & 72.08 ± 2.80 & 46.49 ± 2.65 \\
    
    ~ & EGAT-vacuity-evidence-prop & 90.54 ± 2.56 & 82.38 ± 4.18 & 81.94 ± 3.72 & 58.52 ± 4.38 \\
    \cline{2-6} ~ & \multicolumn{4}{c}{\textbf{ours}} \\ \cline{2-6}
    ~ & EPN & 92.60 ± 2.37 & 87.84 ± 3.74 & 85.98 ± 4.83 & 70.04 ± 6.22 \\
    \rowcolor{gray!30}
    ~ & EPN-reg & \textcolor{red}{95.66 ± 1.34} & \textcolor{blue}{93.22 ± 2.64} & 93.65 ± 1.89 & 82.99 ± 5.69 \\
    \midrule
    \multirow{17}{*}{\rotatebox[origin=c]{90}{\textbf{CoauthorPhysics}}\hspace{0.5cm}} 
    ~ & \multicolumn{4}{c}{\textbf{logit based}} \\ \cline{2-6}
    ~ & VGAT-entropy & 92.99 ± 1.22 & 82.04 ± 2.96 & 88.28 ± 4.96 & 88.76 ± 4.58 \\
    
    ~ & VGAT-max-score & 92.41 ± 1.19 & 79.27 ± 2.68 & 87.32 ± 4.83 & 87.02 ± 4.37 \\
    
    ~ & VGAT-energy & 93.78 ± 1.25 & 84.43 ± 3.15 & 91.72 ± 4.38 & 92.14 ± 4.13 \\
    
    ~ & VGAT-gnnsafe & 95.54 ± 1.04 & 90.19 ± 2.15 & 93.92 ± 4.25 & 94.41 ± 4.00 \\
    
    ~ & VGAT-dropout & 93.60 ± 1.22 & 84.14 ± 3.34 & 87.43 ± 5.24 & 88.71 ± 5.13 \\
    
    ~ & VGAT-ensemble & 93.67 ± 1.29 & 84.63 ± 3.00 & 80.17 ± 7.89 & 82.83 ± 6.76 \\
    \cline{2-6} ~ & \multicolumn{4}{c}{\textbf{evidential based}} \\ \cline{2-6}
    ~ & GPN & 83.22 ± 8.99 & 74.01 ± 11.49 & \textcolor{red}{96.79 ± 1.06} & \textcolor{red}{97.62 ± 0.77} \\
    
    ~ & SGAT-GKDE & 91.48 ± 3.28 & 79.23 ± 8.61 & 91.90 ± 4.86 & 92.79 ± 4.57 \\
    
    ~ & EGAT & 95.06 ± 1.38 & 88.45 ± 3.70 & 87.41 ± 6.26 & 87.79 ± 6.06 \\
    
    ~ & EGAT-vacuity-prop & \textcolor{red}{96.66 ± 0.72} & \textcolor{red}{92.79 ± 1.46} & 90.79 ± 8.32 & 92.33 ± 5.87 \\
    
    ~ & EGAT-evidence-prop & 93.53 ± 2.34 & 85.34 ± 4.36 & 91.82 ± 2.25 & 93.67 ± 1.58 \\
    
    ~ & EGAT-vacuity-evidence-prop & \textcolor{blue}{96.36 ± 1.34} & \textcolor{blue}{91.40 ± 2.48} & \textcolor{blue}{94.53 ± 1.60} & \textcolor{blue}{94.87 ± 1.40} \\
    \cline{2-6} ~ & \multicolumn{4}{c}{\textbf{ours}} \\ \cline{2-6}
    ~ & EPN & 95.53 ± 1.53 & 88.85 ± 4.89 & 92.18 ± 3.47 & 93.12 ± 3.20 \\
    \rowcolor{gray!30}
    ~ & EPN-reg & 94.90 ± 2.26 & 88.83 ± 4.54 & 93.40 ± 3.38 & 94.09 ± 3.67 \\
    \bottomrule
    \end{tabular}
    }
    \caption{OOD detection results (\textcolor{red}{best} and \textcolor{blue}{runner-up}) with GAT as backbone on   AmazonPhotos, AmazonComputers, CoauthorCS, and CoauthorPhysics for OS-4 and OS-5.}
    \label{tab:ood_gat_4}
\end{table}

\textbf{GCN and GAT Comparison.} We show the EPN's performance comparison between GCN and GAT as the backbone across all the LOC settings and datasets in Table~\ref{tab:gcn_gat_comp}. Overall, GAT consistently outperforms GCN across most datasets and settings, particularly on larger datasets like AmazonPhotos, AmazonComputers, and CoauthorPhysics. In the OS-1 (last) setting, GAT achieves higher OOD-AUROC and OOD-AUPR scores compared to GCN on datasets such as CoraML (91.21 ± 0.76 vs. 89.97 ± 2.48) and CiteSeer (90.96 ± 1.99 vs. 88.23 ± 2.79). The trend remains similar across OS-2, OS-3, and OS-5, where GAT shows more robust performance, particularly on PubMed and CoauthorPhysics, with differences in OOD-AUROC as high as 8\% (e.g., OS-4 for PubMed: GAT 66.92 ± 8.04 vs. GCN 53.66 ± 3.72). Notably, GAT performs better on the more challenging LOC settings, like random splits in OS-4 and OS-5, where its improvements over GCN are substantial. These results suggest that GAT provides a stronger latent representation for OOD detection across varying conditions and datasets.

\begin{table}[!htbp]
    \centering
    \small
    \setlength{\tabcolsep}{4pt}
    \resizebox{\textwidth}{!}{%
    \begin{tabular}{@{}ll|cc|cc|cc|cc|cc@{}}
    \toprule
        \multirow{2}{*}{\textbf{Dataset}} & \multirow{2}{*}{\textbf{Backbone}} & \multicolumn{2}{c|}{\textbf{OS-1 (last)}} & \multicolumn{2}{c|}{\textbf{OS-2 (first)}} & \multicolumn{2}{c|}{\textbf{OS-3 (random)}} & \multicolumn{2}{c|}{\textbf{OS-4 (random)}} & \multicolumn{2}{c}{\textbf{OS-5 (random)}} \\ 
        ~ & ~ & OOD-AUROC$\uparrow$ & OOD-AUPR$\uparrow$ & OOD-AUROC$\uparrow$ & OOD-AUPR$\uparrow$ & OOD-AUROC$\uparrow$ & OOD-AUPR$\uparrow$ & OOD-AUROC$\uparrow$ & OOD-AUPR$\uparrow$ & OOD-AUROC$\uparrow$ & OOD-AUPR$\uparrow$ \\ \midrule
        \multirow{2}{*}{\textbf{CoraML}} & GCN & 89.97 ± 2.48 & 86.01 ± 4.82 & 85.91 ± 6.18 & 75.53 ± 11.08 & 88.96 ± 1.46 & 89.26 ± 1.74 & 85.46 ± 4.94 & \textbf{73.14 ± 7.73} & 91.27 ± 2.73 & 89.88 ± 3.06  \\ 
        ~ & GAT & \textbf{91.21 ± 0.76} & \textbf{88.59 ± 1.83} & \textbf{87.41 ± 3.42} & \textbf{79.84 ± 6.71} & \textbf{91.19 ± 1.20} & \textbf{91.83 ± 0.85} & \textbf{86.55 ± 3.08} & 70.89 ± 7.19 & \textbf{92.96 ± 0.70} & \textbf{92.99 ± 0.88}  \\ \midrule
        \multirow{2}{*}{\textbf{CiteSeer}} & GCN & 88.23 ± 2.79 & 69.69 ± 4.97 & 90.86 ± 1.51 & 79.18 ± 4.27 & 87.27 ± 2.52 & 73.24 ± 5.35 & \textbf{88.78 ± 1.52} & \textbf{74.07 ± 4.81} & 90.53 ± 3.06 & 78.03 ± 5.92  \\ 
        ~ & GAT & \textbf{90.96 ± 1.99} & \textbf{75.92 ± 4.67} & \textbf{92.56 ± 1.51} & \textbf{82.60 ± 3.64} & \textbf{88.62 ± 3.65} & \textbf{75.12 ± 5.16} & 87.65 ± 2.17 & 72.24 ± 5.25 & \textbf{90.77 ± 2.05} & \textbf{78.84 ± 3.55}  \\ \midrule
        \multirow{2}{*}{\textbf{PubMed}} & GCN & 67.38 ± 3.85 & 53.66 ± 3.72 & \textbf{65.25 ± 7.45} & \textbf{33.01 ± 6.89} & 53.65 ± 6.11 & 41.47 ± 4.14 & \textbf{69.39 ± 3.78} & \textbf{36.84 ± 5.64} & 64.78 ± 4.79 & \textbf{53.61 ± 4.34}  \\ 
        ~ & GAT & \textbf{68.67 ± 2.56} & \textbf{55.83 ± 2.34} & 64.45 ± 5.89 & 31.09 ± 6.14 & \textbf{66.92 ± 8.04} & \textbf{52.62 ± 8.42} & 62.74 ± 8.80 & 30.40 ± 7.84 & \textbf{65.77 ± 5.67} & 53.30 ± 5.44  \\ \midrule
        \multirow{2}{*}{\textbf{AmazonPhotos}} & GCN & 86.49 ± 5.40 & 81.37 ± 6.67 & \textbf{88.79 ± 5.61} & \textbf{85.96 ± 7.93} & \textbf{91.49 ± 3.74} & \textbf{84.33 ± 7.45} & 83.49 ± 7.72 & 87.38 ± 6.14 & 90.82 ± 7.08 & 80.53 ± 10.59  \\ 
        ~ & GAT & \textbf{93.47 ± 1.37} & \textbf{90.32 ± 2.56} & 87.83 ± 3.77 & 84.81 ± 3.66 & 89.52 ± 3.66 & 78.81 ± 8.23 & \textbf{91.65 ± 2.83} & \textbf{92.65 ± 2.73} & \textbf{91.91 ± 2.38} & \textbf{81.00 ± 5.06}  \\ \midrule
        \multirow{2}{*}{\textbf{AmazonComputers}} & GCN & 83.26 ± 6.06 & 68.91 ± 8.41 & 81.99 ± 9.82 & 91.49 ± 4.83 & 70.96 ± 9.16 & 79.18 ± 8.00 & 70.28 ± 7.69 & 62.52 ± 7.80 & 76.43 ± 5.16 & 85.97 ± 4.02  \\ 
        ~ & GAT & \textbf{90.31 ± 2.64} & \textbf{78.11 ± 5.13} & \textbf{91.00 ± 3.28} & \textbf{95.44 ± 1.87} & \textbf{82.42 ± 8.27} & \textbf{86.30 ± 6.44} & \textbf{84.98 ± 2.59} & \textbf{81.05 ± 3.84} & \textbf{86.20 ± 2.15} & \textbf{92.45 ± 1.50}  \\ \midrule
        \multirow{2}{*}{\textbf{CoauthorCS}} & GCN & \textbf{95.09 ± 1.37} & \textbf{94.47 ± 1.29} & \textbf{91.03 ± 4.06} & \textbf{76.13 ± 9.48} & \textbf{93.30 ± 3.77} & \textbf{76.65 ± 12.16} & \textbf{96.96 ± 1.10} & \textbf{95.39 ± 1.29} & 92.98 ± 3.32 & 80.12 ± 9.41  \\ 
        ~ & GAT & 94.21 ± 1.75 & 92.61 ± 2.33 & 90.48 ± 3.80 & 72.41 ± 9.77 & 91.53 ± 2.20 & 72.75 ± 6.82 & 95.66 ± 1.34 & 93.22 ± 2.64 & \textbf{93.65 ± 1.89} & \textbf{82.99 ± 5.69}  \\ \midrule
        \multirow{2}{*}{\textbf{CoauthorPhysics}} & GCN & 93.59 ± 2.41 & 79.31 ± 5.13 & 87.14 ± 7.33 & 78.99 ± 9.06 & 89.05 ± 9.34 & 88.37 ± 9.43 & 92.21 ± 3.34 & 82.51 ± 5.95 & 86.70 ± 3.50 & 89.00 ± 2.94  \\ 
        ~ & GAT & \textbf{95.57 ± 0.93} & \textbf{85.54 ± 3.61} & \textbf{95.08 ± 1.70} & \textbf{90.12 ± 2.24} & \textbf{90.76 ± 7.33} & \textbf{90.00 ± 7.90} & \textbf{94.90 ± 2.26} & \textbf{88.83 ± 4.54} & \textbf{93.40 ± 3.38} & \textbf{94.09 ± 3.67} \\ 
        \bottomrule
    \end{tabular}
    }
    \caption{Backbone: we compare the  OOD detection performance ($\uparrow$) with GCN or GAT as the backbone. The best results are bold.}\label{tab:gcn_gat_comp}
\end{table}

\textbf{Robutness on features.} In the default setting, we use the hidden states from the last layer of the pretrained model (commonly referred to as logits) as the feature input for our EPN. In Table~\ref{tab:layer}, we compare the performance of the EPN when using the output from either the last layer or the second-to-last layer as input features. This comparison highlights the impact of different feature extraction layers on the overall performance, providing insights into which layer offers more informative representations for the OOD detection task in terms of the datasets.

\begin{table}[!htbp]
    \centering
    \small
    \setlength{\tabcolsep}{4pt}
    \resizebox{\textwidth}{!}{%
    \begin{tabular}{@{}ll|cc|cc|cc|cc|cc@{}}
    \hline
        \multirow{2}{*}{\textbf{Dataset}} & \multirow{2}{*}{\textbf{Model}} & \multicolumn{2}{c|}{\textbf{OS-1 (last)}} & \multicolumn{2}{c|}{\textbf{OS-2 (first)}} & \multicolumn{2}{c|}{\textbf{OS-3 (random)}} & \multicolumn{2}{c|}{\textbf{OS-4 (random)}} & \multicolumn{2}{c}{\textbf{OS-5 (random)}} \\ 
        ~ & ~ & OOD-AUROC$\uparrow$ & OOD-AUPR$\uparrow$ & OOD-AUROC$\uparrow$ & OOD-AUPR$\uparrow$ & OOD-AUROC$\uparrow$ & OOD-AUPR$\uparrow$ & OOD-AUROC$\uparrow$ & OOD-AUPR$\uparrow$ & OOD-AUROC$\uparrow$ & OOD-AUPR$\uparrow$ \\ \hline
        \multirow{2}{*}{\textbf{CoraML}} & middle\_layer & 87.27 ± 5.34 & 82.02 ± 8.13 & 76.10 ± 11.61 & 63.09 ± 12.34 & 83.57 ± 9.14 & 83.78 ± 7.90 & 83.86 ± 4.22 & 67.47 ± 7.87 & 89.16 ± 4.76 & 87.36 ± 5.96  \\ 
        ~ & final layer & \textbf{89.54 ± 2.06} & \textbf{84.35 ± 3.59} & \textbf{80.31 ± 4.67} & \textbf{67.09 ± 6.07} & \textbf{86.56 ± 3.13} & \textbf{86.21 ± 3.31} & \textbf{85.77 ± 2.41} & \textbf{70.98 ± 5.38} & \textbf{90.86 ± 2.05} & \textbf{89.84 ± 2.16}  \\ \hline
        \multirow{2}{*}{\textbf{CiteSeer}} & middle\_layer & 85.06 ± 2.79 & 62.71 ± 5.23 & 89.08 ± 2.39 & 74.76 ± 4.61 & 85.61 ± 3.92 & 69.74 ± 5.40 & 82.89 ± 3.27 & 62.32 ± 5.67 & 90.26 ± 0.78 & 76.30 ± 3.39  \\ 
        ~ & final layer & \textbf{88.23 ± 2.79} & \textbf{69.69 ± 4.97} & \textbf{90.86 ± 1.51} & \textbf{79.18 ± 4.27} & \textbf{87.27 ± 2.52} & \textbf{73.24 ± 5.35} & \textbf{88.78 ± 1.52} & \textbf{74.07 ± 4.81} & \textbf{90.53 ± 3.06} & \textbf{78.03 ± 5.92}  \\ \hline
        \multirow{2}{*}{\textbf{PubMed}} & middle\_layer & 65.83 ± 4.62 & 53.61 ± 5.09 & \textbf{66.36 ± 6.17} & 31.83 ± 5.67 & 47.11 ± 12.12 & 37.78 ± 7.82 & 63.23 ± 8.37 & 31.02 ± 8.02 & \textbf{66.56 ± 4.20} & \textbf{54.54 ± 4.25}  \\ 
        ~ & final layer & \textbf{67.38 ± 3.85} & \textbf{53.66 ± 3.72} & 65.25 ± 7.45 & \textbf{33.01 ± 6.89} & \textbf{53.65 ± 6.11} & \textbf{41.47 ± 4.14} & \textbf{69.39 ± 3.78} & \textbf{36.84 ± 5.64} & 64.78 ± 4.79 & 53.61 ± 4.34  \\ \hline
        \multirow{2}{*}{\textbf{AmazonPhotos}} & middle\_layer & \textbf{84.54 ± 5.55} & 77.12 ± 8.94 & \textbf{90.06 ± 6.40} & \textbf{87.79 ± 9.44} & \textbf{89.98 ± 2.21} & \textbf{78.36 ± 6.20} & \textbf{84.13 ± 5.41} & \textbf{87.77 ± 4.41} & 87.39 ± 9.12 & 74.25 ± 17.78  \\ 
        ~ & final layer & 84.48 ± 6.64 & \textbf{78.42 ± 7.36} & 88.61 ± 4.27 & 86.07 ± 5.40 & 88.34 ± 5.87 & 78.18 ± 10.86 & 83.76 ± 5.14 & 86.89 ± 4.47 & \textbf{92.47 ± 3.14} & \textbf{83.89 ± 5.27}  \\ \hline
        \multirow{2}{*}{\textbf{AmazonComputers}} & middle\_layer & \textbf{84.62 ± 2.60} & \textbf{69.79 ± 5.44} & 83.70 ± 5.08 & 92.81 ± 2.64 & \textbf{71.65 ± 7.10} & \textbf{80.37 ± 6.39} & \textbf{76.39 ± 4.32} & \textbf{68.84 ± 5.53} & \textbf{78.48 ± 5.30} & \textbf{87.06 ± 3.98}  \\ 
        ~ & final layer & 80.45 ± 4.57 & 66.92 ± 6.29 & \textbf{85.09 ± 4.42} & \textbf{93.09 ± 2.49} & 70.39 ± 11.95 & 78.10 ± 9.22 & 72.34 ± 7.10 & 64.04 ± 7.82 & 75.34 ± 5.21 & 85.33 ± 3.94  \\ \hline
        \multirow{2}{*}{\textbf{CoauthorCS}} & middle\_layer & 90.34 ± 5.41 & 88.82 ± 8.40 & 90.13 ± 4.15 & 69.40 ± 12.09 & 88.56 ± 6.16 & 60.68 ± 15.86 & 91.99 ± 4.78 & 86.16 ± 8.28 & 89.96 ± 4.39 & 70.87 ± 12.10  \\ 
        ~ & final layer & \textbf{92.85 ± 1.86} & \textbf{91.00 ± 2.38} & \textbf{93.31 ± 2.74} & \textbf{77.89 ± 7.92} & \textbf{91.30 ± 7.46} & \textbf{67.95 ± 15.37} & \textbf{94.04 ± 2.59} & \textbf{89.71 ± 5.65} & \textbf{93.60 ± 3.71} & \textbf{80.74 ± 10.03}  \\ \hline
        \multirow{2}{*}{\textbf{CoauthorPhysics}} & middle\_layer & \textbf{94.97 ± 0.92} & \textbf{82.97 ± 2.43} & 86.98 ± 5.32 & 77.17 ± 6.24 & \textbf{92.49 ± 3.07} & \textbf{91.86 ± 3.26} & \textbf{94.68 ± 1.46} & \textbf{86.25 ± 3.71} & \textbf{89.83 ± 5.39} & \textbf{91.78 ± 4.31}  \\ 
        ~ & final layer & 93.59 ± 2.41 & 79.31 ± 5.13 & \textbf{87.14 ± 7.33} & \textbf{78.99 ± 9.06} & 89.05 ± 9.34 & 88.37 ± 9.43 & 92.21 ± 3.34 & 82.51 ± 5.95 & 86.70 ± 3.50 & 89.00 ± 2.94 \\ \hline
    \end{tabular}
    }
    \caption{Feature: we compare the  OOD detection performance ($\uparrow$) with last or second-last layer output from the pretrained model as the input of EPN. The best results are bold. }
    \label{tab:layer}
    
\end{table}

\textbf{Robutness on activation function.} We use the Exponential or SoftPlus as the last layer's activation function and report the result in Table~\ref{tab:act}. Overall, the Exponential activation function shows better performance in terms of OOD-AUROC and OOD-AUPR across several datasets compared to SoftPlus. Specifically, in datasets such as CoauthorCS and CoauthorPhysics, the Exponential activation consistently achieves higher OOD-AUROC, particularly in scenarios like OS-3 and OS-5, with notable gains of around 4-5 points. On the other hand, SoftPlus activation has comparable or slightly better performance in certain cases, such as in PubMed for OS-2 and OS-4. However, the Exponential activation tends to provide more consistent robustness across random splits and diverse datasets.
\begin{table}[!htbp]
    \centering
    \small
    \setlength{\tabcolsep}{4pt}
    \resizebox{\textwidth}{!}{%
    \begin{tabular}{@{}ll|cc|cc|cc|cc|cc@{}}
    \hline
        \multirow{2}{*}{\textbf{Dataset}} & \multirow{2}{*}{\textbf{Model}} & \multicolumn{2}{c|}{\textbf{OS-1 (last)}} & \multicolumn{2}{c|}{\textbf{OS-2 (first)}} & \multicolumn{2}{c|}{\textbf{OS-3 (random)}} & \multicolumn{2}{c|}{\textbf{OS-4 (random)}} & \multicolumn{2}{c}{\textbf{OS-5 (random)}} \\ 
        ~ & ~ & OOD-AUROC$\uparrow$ & OOD-AUPR$\uparrow$ & OOD-AUROC$\uparrow$ & OOD-AUPR$\uparrow$ & OOD-AUROC$\uparrow$ & OOD-AUPR$\uparrow$ & OOD-AUROC$\uparrow$ & OOD-AUPR$\uparrow$ & OOD-AUROC$\uparrow$ & OOD-AUPR$\uparrow$ \\ \hline
        \multirow{2}{*}{\textbf{CoraML}} & softplus & \textbf{90.57 ± 1.38} & \textbf{87.02 ± 1.87} & \textbf{85.53 ± 3.05} & \textbf{73.24 ± 5.70} & \textbf{88.77 ± 1.60} & \textbf{87.73 ± 2.42} & \textbf{87.63 ± 0.94} & \textbf{73.01 ± 1.61} & 90.55 ± 1.71 & 89.52 ± 1.51  \\
        ~ & exp & 89.54 ± 2.06 & 84.35 ± 3.59 & 80.31 ± 4.67 & 67.09 ± 6.07 & 86.56 ± 3.13 & 86.21 ± 3.31 & 85.77 ± 2.41 & 70.98 ± 5.38 & \textbf{90.86 ± 2.05} & \textbf{89.84 ± 2.16}  \\ \hline
        \multirow{2}{*}{\textbf{CiteSeer}} & softplus & \textbf{88.24 ± 2.24} & \textbf{70.03 ± 4.95} & 90.55 ± 1.67 & 78.78 ± 3.39 & 86.64 ± 3.00 & 72.60 ± 5.82 & 85.81 ± 2.39 & 68.41 ± 5.40 & 90.36 ± 1.40 & 76.56 ± 3.26  \\
        ~ & exp & 88.23 ± 2.79 & 69.69 ± 4.97 & \textbf{90.86 ± 1.51} & \textbf{79.18 ± 4.27} & \textbf{87.27 ± 2.52} & \textbf{73.24 ± 5.35} & \textbf{88.78 ± 1.52} & \textbf{74.07 ± 4.81} & \textbf{90.53 ± 3.06} & \textbf{78.03 ± 5.92}  \\ \hline
        \multirow{2}{*}{\textbf{PubMed}} & softplus & \textbf{67.77 ± 3.91} & \textbf{54.23 ± 4.31} & 61.91 ± 12.33 & 31.03 ± 9.02 & 50.59 ± 5.92 & 39.37 ± 4.41 & \textbf{70.03 ± 2.79} & \textbf{39.31 ± 4.15} & \textbf{66.15 ± 2.27} & 52.41 ± 1.97  \\
        ~ & exp & 67.38 ± 3.85 & 53.66 ± 3.72 & \textbf{65.25 ± 7.45} & \textbf{33.01 ± 6.89} & \textbf{53.65 ± 6.11} & \textbf{41.47 ± 4.14} & 69.39 ± 3.78 & 36.84 ± 5.64 & 64.78 ± 4.79 & \textbf{53.61 ± 4.34}  \\ \hline
        \multirow{2}{*}{\textbf{AmazonPhotos}} & softplus & \textbf{88.57 ± 4.94} & \textbf{82.54 ± 6.72} & 83.20 ± 10.47 & 80.35 ± 10.43 & 76.77 ± 5.81 & 58.31 ± 6.12 & \textbf{86.28 ± 2.65} & \textbf{88.53 ± 1.94} & 91.31 ± 2.28 & 80.27 ± 5.39  \\
        ~ & exp & 84.48 ± 6.64 & 78.42 ± 7.36 & \textbf{88.61 ± 4.27} & \textbf{86.07 ± 5.40} & \textbf{88.34 ± 5.87} & \textbf{78.18 ± 10.86} & 83.76 ± 5.14 & 86.89 ± 4.47 & \textbf{92.47 ± 3.14} & \textbf{83.89 ± 5.27}  \\ \hline
        \multirow{2}{*}{\textbf{AmazonComputers}} & softplus & 77.21 ± 7.10 & 55.53 ± 11.28 & \textbf{90.64 ± 1.30} & \textbf{95.82 ± 0.73} & \textbf{78.96 ± 4.57} & \textbf{82.05 ± 3.97} & \textbf{75.22 ± 4.29} & \textbf{64.51 ± 6.68} & 72.60 ± 3.54 & 83.36 ± 2.19  \\
        ~ & exp & \textbf{80.45 ± 4.57} & \textbf{66.92 ± 6.29} & 85.09 ± 4.42 & 93.09 ± 2.49 & 70.39 ± 11.95 & 78.10 ± 9.22 & 72.34 ± 7.10 & 64.04 ± 7.82 & \textbf{75.34 ± 5.21} & \textbf{85.33 ± 3.94}  \\ \hline
        \multirow{2}{*}{\textbf{CoauthorCS}} & softplus & 90.14 ± 3.31 & 85.90 ± 5.35 & 81.64 ± 8.33 & 56.55 ± 14.99 & 82.38 ± 10.26 & 55.23 ± 16.75 & 91.18 ± 3.75 & 85.90 ± 6.39 & 77.49 ± 7.76 & 53.94 ± 11.67  \\
        ~ & exp & \textbf{92.85 ± 1.86} & \textbf{91.00 ± 2.38} & \textbf{93.31 ± 2.74} & \textbf{77.89 ± 7.92} & \textbf{91.30 ± 7.46} & \textbf{67.95 ± 15.37} & \textbf{94.04 ± 2.59} & \textbf{89.71 ± 5.65} & \textbf{93.60 ± 3.71} & \textbf{80.74 ± 10.03}  \\ \hline
        \multirow{2}{*}{\textbf{CoauthorPhysics}} & softplus & \textbf{94.98 ± 2.28} & \textbf{83.10 ± 5.83} & \textbf{91.04 ± 4.71} & \textbf{82.99 ± 6.29} & 86.41 ± 5.64 & 84.39 ± 5.96 & 92.14 ± 4.65 & 81.21 ± 8.54 & \textbf{88.18 ± 7.44} & \textbf{89.31 ± 7.38}  \\
        ~ & exp & 93.59 ± 2.41 & 79.31 ± 5.13 & 87.14 ± 7.33 & 78.99 ± 9.06 & \textbf{89.05 ± 9.34} & \textbf{88.37 ± 9.43} & \textbf{92.21 ± 3.34} & \textbf{82.51 ± 5.95} & 86.70 ± 3.50 & 89.00 ± 2.94 \\ \hline
    \end{tabular}
    }
    \caption{Activation function: we compare the  OOD detection performance ($\uparrow$) with Exponential or SoftPlus as the last layer's activation function. The best results are bold. }
    \label{tab:act}
\end{table}

\subsection{OOD Detection in Image Classification Tasks}
We conduct preliminary experiments on image classification tasks to evaluate the generalization ability of our proposed framework. Our original EPN design incorporates a label propagation layer within a graph structure. To adapt the framework for image classification, we remove the label propagation layer and replace the graph neural network (GNN) components with a standard image classification network. Specifically, we use LeNet in the experiments. 

For our evaluation, we follow the setup in \cite{malinin2018predictive}, using the MNIST dataset for training and Omniglot as the out-of-distribution (OOD) dataset. We compare our model against two baseline approaches: batch-ensemble \citep{chen2023batch} and packed-ensemble \citep{laurent2023packed}, leveraging publicly available implementations\footnote{\url{https://github.com/ENSTA-U2IS-AI/torch-uncertainty}}. To ensure a fair comparison, we tune the hyperparameters of all models, including the baselines, based on the OOD detection performance on the pseudo-OOD validation set, which is FasionMNIST in our experiments.

\begin{table*}[h!]
\centering
\caption{OOD Detection on Image Classification Task}
\vspace{1mm}
\begin{tabular}{c|c|c|c}
\toprule
         Models & ID-ACC & OOD-AUROC & OOD-AUPR \\
\midrule
        Energy & 99.17 $\pm$ 0.26 & 98.05 $\pm$ 0.74 & 98.07 $\pm$ 0.63 \\
        Batch-ensemble & \textbf{99.21} $\pm$ 0.09 & 95.63 $\pm$ 3.40 & 95.15 $\pm$ 3.62 \\
        Packed-ensemble & 99.16 $\pm$ 0.08 & 97.58 $\pm$ 0.95 & 96.82 $\pm$ 1.15 \\
        ENN & 99.06 $\pm$ 0.06 & \textbf{99.81} $\pm$ 0.24 & \textbf{99.83} $\pm$ 0.23 \\
        EPN-reg & 99.03 $\pm$ 0.00 & 98.84 $\pm$ 0.72 & 98.48 $\pm$ 0.40\\
\bottomrule
\end{tabular}
\label{tab:ood_image}
\end{table*}

The results of these experiments, presented in Table~\ref{tab:ood_image}, indicate that our proposed model (EPN-reg) delivers the best epistemic uncertainty quantification (OOD detection AUPR and AUROC), outperforming all baselines by 2.25\% (AUPR) compared to the worst model, and is comparable to EGNN. These preliminary results demonstrate the effectiveness of our proposed model on epistemic uncertainty quantification.

\end{document}